\documentclass[11pt]{article}

\usepackage[margin=1in]{geometry}
\usepackage{amsmath,amssymb,amsthm}
\usepackage{booktabs}

\newtheorem{definition}{Definition}
\newtheorem{theorem}{Theorem}
\newtheorem{proposition}{Proposition}
\newtheorem{corollary}{Corollary}
\usepackage{enumitem}
\usepackage[hypertexnames=false]{hyperref}
\usepackage{tikz}
\usetikzlibrary{arrows.meta,positioning,calc}

\title{\textbf{Biological Motifs for Agentic Control}\\
\large A Typed Interface Correspondence between Gene Regulatory Networks and Agentic Software Architectures\\[0.5em]
\normalsize Preprint -- Feedback Welcome}
\author{Bogdan Banu\\\texttt{bogdan@banu.be}}
\date{\today}

\begin{document}
\maketitle

\begin{abstract}
The transition of Large Language Models (LLMs) from passive generators to autonomous agents
has introduced significant challenges in reliability, security, and state management. Current agentic
architectures are often constructed ad-hoc, prone to ``hallucination cascades,'' infinite loops, and
prompt injection attacks. This paper argues that many of these failure modes can be analyzed using
control motifs long studied in systems biology, provided the comparison is made at the level of typed
interfaces and coordination structure rather than literal biological mechanism.

We develop a typed interface correspondence between Gene Regulatory Networks and agentic software systems using polynomial functors and wiring diagrams. Five biological motifs---Coherent Feed-Forward Loops for noise suppression, Adaptive Immunity for layered security, Mitochondrial Signaling for resource governance, Endosymbiosis for neuro-symbolic integration, and Morphogen Diffusion for spatially varying coordination---are mapped to composable software design patterns. An epistemic topology layer derives Kripke-style knowledge operators from the wiring diagram's observation structure and proves four predictive theorems for multi-agent scaling.

The core contributions are: (1)~the Agentic Operad, a typed syntax for agent composition with provable error suppression bounds for feed-forward topologies; (2)~an epistemic topology with four theorems---error amplification, sequential penalty, parallel acceleration, and tool density scaling---whose qualitative predictions are consistent with published multi-agent benchmarks; and (3)~a six-layer progression from structure through development, grounded in autonomous learning frameworks and convergence proxies from the empirical literature. A reference implementation with 1{,}813 tests and 116 examples illustrates practical feasibility.
\end{abstract}

\newpage

\tableofcontents
\newpage

\section{Introduction}

The field of Artificial Intelligence is undergoing a paradigm shift from Generative AI (systems that
produce text based on static prompts) to Agentic AI (systems that execute multi-step workflows to
achieve autonomous goals). While the capabilities of individual Large Language Models (LLMs) have
scaled predictably, the engineering of systems of agents remains a fragile art. Developers struggle with
non-deterministic outputs, infinite loops, adversarial attacks, and the difficulty of maintaining global
coherence in distributed, stochastic systems.

We argue that these challenges are not purely ad-hoc engineering problems, but recurring constraints of
distributed information processing systems. A useful formal analogue is provided by Gene Regulatory
Networks (GRNs). At the atomic level, a gene and an agent capability both expose typed interfaces that
consume local signals and produce downstream effects. At the composite level, a biological cell organizes
many genes plus shared infrastructure (organelles, membrane); analogously, an executable agent organizes
many capabilities plus shared runtime components. This multi-scale correspondence, rather than a literal
identity of mechanisms, is the organizing hypothesis of the paper.

\subsection{The Biological Heuristic}

Biology has evolved specific topological structures, known as Network Motifs, to handle noise, security,
and state~\cite{milo2002network}. We identify five critical biological heuristics that map directly to agentic engineering:
\begin{itemize}[leftmargin=*]
\item The Coherent Feed-Forward Loop (CFFL): in biology a persistence detector that filters transient input
pulses~\cite{alon2007network}; its conjunctive (AND-gate) form underlies redundant-verification gates such as
``Human-in-the-Loop'' guardrails, where an action requires a second, independent approval.
\item Quorum Sensing: A distributed consensus mechanism where action is taken only when signal density
exceeds a threshold, analogous to self-consistency or majority voting over sampled outputs (an agreement
threshold, distinct from Mixture-of-Experts routing).
\item Chaperone Proteins: Molecular cages that force proteins to fold correctly, analogous to Schema
Validators that enforce structured outputs (JSON).
\item Mitochondrial Information Processing: Metabolic constraints acting as a ``Motherboard'' for
decision gating, governing not just energy availability but cognitive control policy.
\item Adaptive Immunity: The Self/Non-Self distinction, with MHC-like provenance tagging and Trust-Gated
access control, relevant to preventing Prompt Injection attacks and hallucination cascades.
\end{itemize}

\subsection{The Categorical Bridge}

To move this observation from metaphor to discipline, we utilize Applied Category Theory. We define a
category-theoretic model of agentic interfaces using the language of $\mathbf{Poly}$ (Polynomial Functors)
as described by Spivak~\cite{spivak2021learners}.
We use $\mathbf{Poly}$ as a common language for typed interfaces and wiring diagrams as a language for
composition. The claim is not that GRNs and software agents are identical physical systems, but that
selected biological and agentic components can be compared within the same interface-level formalism.
Proposition~\ref{prop:substrate-independence} states this precisely as a shared interpretation into the
category $\mathbf{Coalg}(\mathbf{Poly})$---a substrate independence in the sense of Marom
et al.~\cite{marom2026category}, not an isomorphism between the domains.
A capability is then described not by its weights, but by its interface---a dynamical system consuming
observations and producing actions:
\begin{equation}
P_A(y) = \sum_{o\in O} y^{I_o}.
\label{eq:poly-functor-agent}
\end{equation}

\subsection{Contributions}

This paper makes the following contributions:

\paragraph{Structural Safety.}
\begin{enumerate}[leftmargin=*]
\item \textbf{A Typed Correspondence:} We establish a disciplined mapping between biological components
(Genes, Promoters, Plasmids, Organelles) and software components (Capabilities, Schemas, Tools, Runtimes),
including multi-cellular organization for multi-agent systems.
\item \textbf{The Agentic Operad:} We define WAgent, a syntax for agent wiring that forbids specific classes
of \textbf{ill-typed wirings} at the topological level. We prove error suppression bounds for the CFFL
topology, explicitly accounting for correlation between error modes.
\item \textbf{Wire-Level Optics:} We extend the Lens formalism with Prism (conditional type-based routing)
and Traversal (batch element-wise processing) optics, enabling richer signal processing at the wiring level.
\item \textbf{Composable Coalgebras:} We make the coalgebraic state machine framework explicit and
composable, with parallel/sequential composition and observational equivalence criteria.
\end{enumerate}

\paragraph{Security and Trust.}
\begin{enumerate}[leftmargin=*,resume]
\item \textbf{Adaptive Immunity:} We formalize the Provenance Functor and Trust-Gated Lens, providing
structural resistance to content-level trust forgery, where content-based attacks cannot elevate trust levels.
\item \textbf{Pathology \& Homeostasis:} We classify agentic failures as biological diseases and derive
continuous self-repair mechanisms (Chaperone Loop, Regeneration, Autophagy).
\item \textbf{Capability-Gated Tool Acquisition:} The Plasmid Registry implements Horizontal Gene
Transfer with capability gating, preventing privilege escalation when agents dynamically acquire tools.
\end{enumerate}

\paragraph{Epistemic and Metabolic Intelligence.}
\begin{enumerate}[leftmargin=*,resume]
\item \textbf{Epistemic Health:} We define Epiplexity (Bayesian Surprise) with operational approximations
using embedding similarity and perplexity, connecting agent dynamics to the Free Energy Principle.
\item \textbf{Metabolic Intelligence:} We distinguish fast (Apoptosis) and slow (Retrograde Response)
interventions, and formalize the Metabolic-Epigenetic Coupling for cost-gated retrieval.
\item \textbf{Evolutionary Dynamics:} We situate agentic AI within the Vermeij Trend, identifying three
selective pressures (adversarial, complexity, efficiency) that drive architectural evolution.
\item \textbf{Morphogen Diffusion:} We formalize spatially varying coordination as a discrete-time
dynamical system on the agent graph, producing position-dependent behavior without central control.
\item \textbf{Epistemic Topology:} We derive Kripke-style knowledge operators from wiring diagram
structure and prove four predictive theorems for multi-agent scaling (error amplification, sequential
penalty, parallel acceleration, tool density), then check their qualitative consistency with published architecture-level results.
\end{enumerate}

\paragraph{Optimization.}
\begin{enumerate}[leftmargin=*,resume]
\item \textbf{Diagram Optimization:} Cost-annotated wiring diagrams admit categorical rewriting
rules that preserve observational equivalence while improving resource utilization.
\end{enumerate}

By viewing agentic engineering through the lens of theoretical biology and category theory, we aim to provide a
framework for building robust software systems whose stability properties derive in part from their network topology.

\newpage

\section{Related Work}

This work sits at the intersection of Systems Biology, Applied Category Theory, and Agentic AI. While significant
research exists within each domain, the formal synthesis of biological control topologies with agentic software
architectures has received limited attention.

\subsection{Network Motifs in Systems Biology}

The concept of ``Network Motifs''---statistically over-represented sub-graphs in complex networks---was introduced
by Milo et al.~\cite{milo2002network}. Their work demonstrated that biological networks are not random but are composed of specific
building blocks selected for functional data processing. Alon~\cite{alon2007network} further characterized the dynamical properties
of these motifs, identifying the Coherent Feed-Forward Loop (CFFL) as a persistence detector. We extend this by
mapping these motifs to the stochastic nature of Generative AI.

\subsection{Applied Category Theory (ACT)}

To formalize network structure, we draw upon ACT. Fong and Spivak~\cite{fong2019seven} provide a comprehensive introduction to Applied Category Theory. Spivak~\cite{spivak2021learners} and Vagner et al.~\cite{vagner2015algebras} established a rigorous framework
for modeling Open Dynamical Systems using the category $\mathbf{Poly}$ and the Operad of Wiring Diagrams. Niu and Spivak~\cite{niu2023polynomial} develop the full mathematical theory of polynomial functors as a language for interaction. Marom et al.~\cite{marom2026category} independently relate biological and engineered stimulus--response systems by a structure-preserving functor between subcategories of dynamical systems, articulating the substrate-independence principle we adopt in \S\ref{sec:mapping}: systems sharing compositional structure can be related by interface logic alone, without sharing a physical substrate. To our
knowledge, this is the first application of Polynomial Functors specifically designed to model the interface of
LLM Agents and to verify safety properties in Agentic topologies.

\subsection{The Architecture Triple}\label{sec:archagents}

De los Riscos et al.~\cite{delosriscos2026categorical} introduce the \textsc{ArchAgents} category, whose objects
are \emph{Architecture triples} $A=(G,\mathrm{Know},\Phi)$: $G$ is the syntactic wiring (typed modules and
ports), $\mathrm{Know}$ is the structural knowledge the architecture maintains---its invariants and
\emph{certificates}---and $\Phi$ is the deployment map from abstract capability slots to concrete model and
tool implementations. A \emph{certificate} is a triple $(\tau,\sigma,\mathrm{evds})$: a theorem statement
$\tau$, a map $\sigma$ from its symbols to architecture parameters, and a mechanically replayable derivation
$\mathrm{evds}$~\cite{banu2026harness}. Morphisms are structure-preserving translations---in practice,
\emph{compilers} between frameworks. Together with the shared-interface correspondence of \S\ref{sec:mapping},
this triple is one of the two categorical objects this monograph builds on; it recurs below as the carrier of
Operon's structural guarantees.

\subsection{Reliability in Agentic AI}

Techniques such as ``Chain of Thought''~\cite{wei2022chain} utilize iterative looping to improve output quality. However, these
methods operate primarily at the level of the prompt (the input signal) rather than the topology (the wiring).
By importing the concept of Autopoiesis~\cite{maturana1980autopoiesis}, we propose a methodology where reliability is a property of the
network architecture itself.

\subsection{Epistemic Logic in Distributed Systems}

Fagin et al.~\cite{fagin1995reasoning} established the foundational framework for reasoning about knowledge in multi-agent systems, formalizing what agents know and what they know about each other's knowledge. Halpern and Moses~\cite{halpern1990knowledge} proved the impossibility of attaining common knowledge in asynchronous systems, a result with deep implications for coordination protocols. To our knowledge, these formal methods from epistemic logic have not previously been applied to the design and optimization of multi-agent AI architectures.

\subsection{Temporal Databases and Bi-Temporal Data Models}

Bi-temporal data management---tracking both \emph{valid time} (when a fact is true in the world) and \emph{transaction time} (when the system recorded it)---has been studied extensively in the database community. Snodgrass~\cite{snodgrass2000temporal} provided the foundational treatment, distinguishing valid-time, transaction-time, and bi-temporal relations, and demonstrating that the two time axes are orthogonal: a fact may be recorded before it becomes valid, corrected after it ceases to be valid, or both. SQL:2011~\cite{kulkarni2012temporal} standardized temporal query support, including \texttt{FOR SYSTEM\_TIME} and \texttt{FOR BUSINESS\_TIME} clauses.

\subsection{Adaptive Multi-Agent Assembly and Meta-Control}

Dupoux, LeCun, and Malik~\cite{dupoux2026learning} propose a three-system cognitive architecture: System~A (observational, statistical, cheap) discovers representations, System~B (action-oriented, goal-directed, expensive) discovers causal structure, and System~M (meta-control) routes data between them. This tripartite structure maps naturally onto Operon's fast/deep nucleus distinction and motivates the \texttt{WatcherComponent} as a concrete instantiation of System~M (\S\ref{sec:adaptive-structure}).

Hao et al.~\cite{hao2026bigmas} demonstrate empirically that incorrect multi-agent runs require systematically more routing decisions than successful ones (e.g., 9.4 vs 7.3 mean decisions on planning tasks). This finding grounds our use of intervention count as a convergence proxy: when the watcher's cumulative retry/escalate count exceeds a threshold relative to stage count, it emits a non-convergence signal and halts the organism (\S\ref{sec:watcher-impl}).

Jiang et al.~\cite{jiang2026anatomy} provide a structure-oriented taxonomy of Memory-Augmented Generation (MAG) systems, categorizing them into lightweight semantic, entity-centric, episodic/reflective, and structured/hierarchical designs. Their empirical analysis reveals that append-only memory architectures are significantly more robust to backbone format errors than complex structured alternatives---a finding that validates Operon's design decision to make \texttt{BiTemporalMemory} append-only. They also quantify the ``Agency Tax'' (latency and token overhead of memory maintenance), highlighting a practical concern that the convergence with operational runtimes must address.

Lin et al.~\cite{lin2026memma} propose MemMA, a multi-agent framework coordinating the full memory cycle through a planner-worker architecture with in-situ self-evolution. Their Meta-Thinker provides strategic guidance for both memory construction and retrieval, analogous to Operon's \texttt{WatcherComponent} providing intervention decisions. Their probe-based memory verification---generating synthetic questions after each session to test memory fidelity, then repairing failures immediately---complements Operon's \texttt{counterfactual\_replay()}, which detects corrections but does not actively probe for gaps.

Feng, Wang, and Zhu~\cite{feng2026selfevolving} propose Self-evolving Embodied AI, a paradigm comprising five co-evolving components: memory self-updating, task self-switching, environment self-prediction, embodiment self-adaptation, and model self-evolution. These map directly onto Operon's phased roadmap (bi-temporal memory, adaptive assembly, sleep consolidation, developmental staging, and social learning respectively), and their emphasis on multi-timescale closed-loop co-evolution aligns with Operon's per-stage (watcher), per-run (adaptive), and per-batch (consolidation) adaptation cycles.

Zhou et al.~\cite{zhou2026externalization} provide a unified review of agent externalization through four pillars---Memory, Skills, Protocols, and Harness Engineering---tracing the progression from weights-based to harness-centric agent design.  Their framework maps directly onto Operon's categorical Architecture triple: Memory corresponds to the coalgebraic state (BiTemporalMemory), Skills to objects composed via the agentic operad (SkillStage), Protocols to syntactic wiring (WiringDiagram), and the Harness to the full Architecture $(G, \mathrm{Know}, \Phi)$.  Their observation that ``agent infrastructure transforms hard cognitive burdens into manageable forms'' is precisely the claim our structural guarantee benchmarks (\S\ref{sec:empirical}) validate empirically.

Ma et al.~\cite{ma2026atomic} demonstrate that five atomic coding skills---localization, editing, unit-test generation, issue reproduction, and code review---compose without negative interference under joint reinforcement learning, an 18.7\% average improvement over task-specific optimization. This is favorable empirical evidence for the operad composition model (\S\ref{sec:operad}), with one caveat we keep explicit: the operad preserves \emph{typed and structural} properties (interface compatibility, certificate replay) by construction, but \emph{behavioral} quality under composition is empirical---our own benchmarks find mostly non-interference alongside one negative case~\cite{banu2026benchmarks}.

We apply bi-temporal data management~\cite{snodgrass2000temporal, kulkarni2012temporal} to agentic memory in \S\ref{sec:coalgebra} and \S\ref{sec:temporal-epistemics}.

\newpage

\section{The Mapping: Biology $\leftrightarrow$ Software}\label{sec:mapping}

To compare Agentic Systems and Gene Regulatory Networks (GRNs) under a common typed-interface
abstraction, we map selected components from both domains to a shared mathematical object. We
utilize the category $\mathbf{Poly}$, where objects are polynomial functors representing interfaces, and
morphisms represent interaction protocols. The claim in this section is a correspondence of interface
descriptions, not a proof that the full biological and software domains are equivalent in mechanism,
implementation, or dynamics.

\subsection{Preliminaries: The Category $\mathbf{Poly}$}

In Applied Category Theory, a Polynomial Functor $P$ represents a typed interface for a dynamical system. It is
defined as a sum of representable functors:
\begin{equation}
P(y) = \sum_{o\in O} y^{I_o}.
\label{eq:poly-functor}
\end{equation}
Here, $O$ is the set of possible Positions (or Outputs) the system can expose. For each position $o\in O$, there
is a set $I_o$ of Directions (or Inputs) required to transition the system to a new state.
\begin{itemize}[leftmargin=*]
\item The coefficient $o$ represents the value produced by the system.
\item The exponent $I_o$ represents the capacity to receive information from the environment.
\end{itemize}
This formalism captures the essence of a ``stateful interface'': the system outputs a value $o$ and then waits
for a specific type of input $i\in I_o$ before it can proceed.

\begin{figure}[h]
\centering
\begin{tikzpicture}[>=Latex, node distance=12mm]
  \node[draw, rounded corners, minimum width=30mm, minimum height=10mm] (cap) {Output $o\in O$};
  \node[draw, circle, below left=10mm and 10mm of cap] (i1) {$i_1\in I_o$};
  \node[draw, circle, below=10mm of cap] (i2) {$i_2\in I_o$};
  \node[draw, circle, below right=10mm and 10mm of cap] (i3) {$\cdots$};
  \draw[->] (i1) -- (cap);
  \draw[->] (i2) -- (cap);
  \draw[->] (i3) -- (cap);
  \node[below=12mm of i2] {The Interface $P(y)$};
\end{tikzpicture}
\caption{A visual representation of a Polynomial Functor (often called a ``Mushroom'' or ``Corolla''). The system
offers an Output (the cap) and exposes specific Input ports (the stalks) dependent on that output.}
\end{figure}

\subsection{The Correspondence: Genes and Agent Capabilities}

We now apply this abstract definition to our specific domains.

\begin{definition}[The Gene Object]\label{def:gene-object}
A gene $G$ is a polynomial functor where $O_G$ is the set of expressed proteins and
$I_G=(I_{\text{prot}})_{\text{prot}\in O_G}$ is the \textbf{family} of regulatory-signal sets (transcription factors)
available at each expressed protein:
\begin{equation}
P_{\text{Gene}}(y) = \sum_{\text{prot}\in O_G} y^{I_{\text{prot}}}.
\label{eq:poly-gene}
\end{equation}
\end{definition}

\begin{definition}[The Agent-Capability Object]\label{def:agent-capability}
An agent capability $A$ is a polynomial functor where $O_A$ is the set of generated messages/actions, and
$I_A=(I_{\text{action}})_{\text{action}\in O_A}$ is the \textbf{family} of observation sets available at each action:
\begin{equation}
P_{\text{Agent}}(y) = \sum_{\text{action}\in O_A} y^{I_{\text{action}}}.
\label{eq:poly-agent}
\end{equation}
\end{definition}

\paragraph{Multi-Scale Composition.}
The polynomial functor formalism applies at every level of biological and software organization. Definitions~\ref{def:gene-object} and~\ref{def:agent-capability} establish the \textit{atomic} correspondence: a single gene and a single agent capability share the same interface form. Polynomial functors then compose---via the parallel ($\otimes$), serial ($\circ$), and trace ($\mathrm{Tr}$) operations of \S\ref{sec:operad}---allowing the same interface language to describe progressively richer assemblies:

\begin{center}
\small
\begin{tabular}{@{}lll@{}}
\toprule
\textbf{Biological Scale} & \textbf{Software Scale} & \textbf{Composition}\\
\midrule
Gene & Agent capability & Atomic polynomial functor\\
Cell & Agent runtime & Composite of capabilities $+$ organelles\\
Tissue & Multi-agent subsystem & Wiring diagram of agents (\S\ref{sec:multi-cellular})\\
\bottomrule
\end{tabular}
\end{center}

\noindent A cell is not merely a bag of genes; it is a structured composition with shared infrastructure (organelles) and a boundary (membrane). Likewise, an agent runtime is a structured composition of capabilities with shared components (LLM provider, memory, error handling) and a security boundary. To reduce ambiguity, the rest of the paper uses \textit{capability} for the atomic interface and \textit{agent} for the composed runtime-level object, even though later wiring-diagram examples sometimes use ``agent'' informally for executable boxes. The organelle mappings below describe this cell-level architecture. The multi-cellular organization of \S\ref{sec:multi-cellular} then composes agents into tissues via wiring diagrams---the same formalism, one level up.

\subsection{The Interface: Promoters as Lenses}

In biology, a gene is not universally accessible. It is guarded by a Promoter Region---a specific DNA sequence
that only binds to compatible Transcription Factors. In software, an agent is guarded by an API Schema or
Context Window definition.

We model this gating mechanism using Optics, specifically Lenses. A Lens consists of two maps between a global
state $S$ and a local view $V$:
\begin{enumerate}[leftmargin=*]
\item Get (View): $\mathrm{get}: S \to V$ (Extracting relevant signal from global state).
\item Put (Update): $\mathrm{put}: S\times V \to S$ (Updating global state based on local change).
\end{enumerate}

The ``Promoter'' acts as a filter that determines which part of the global cellular environment ($S$) is visible
($V$) to the gene.
\begin{itemize}[leftmargin=*]
\item \textbf{Biological Lens:} The promoter filters the chaotic cellular soup, allowing the gene to ``see''
only specific molecules (e.g., Lac Repressor).
\item \textbf{Agentic Lens:} The Context Window filters the massive vector database, allowing the agent to
``see'' only the relevant retrieved chunks (RAG).
\end{itemize}

If the input signal does not match the Schema (Promoter), the Lens fails to focus, and the interaction is routed
to an explicit \textbf{inactive/error} case (equivalently, one works with a \textbf{partial} lens, or a total lens
into $V+\mathrm{Error}$) (the agent does not run; the gene is not expressed).

\subsection{Wire-Level Optics: Beyond Lenses}

The Lens formalism models constitutive access: a promoter that either admits or blocks a signal. Biological systems employ richer signal-processing at the interface level. We extend the wiring diagram with two additional optic types from the categorical optics literature.

\paragraph{Prism: Receptor Specificity.}
A membrane receptor does not merely pass or block signals; it selects signals by molecular shape. A Prism on a wire admits values of specific data types and rejects others:
\begin{equation}
\mathrm{prism}_A(\tau, v) =
\begin{cases}
v & \text{if } \tau \in A \\
\bot & \text{if } \tau \notin A
\end{cases}
\label{eq:prism-optic}
\end{equation}
where $A \subseteq T$ is the set of accepted types. This enables fan-out routing: a single output port connects to multiple wires, each guarded by a different prism, directing JSON to one handler and errors to another---analogous to how different receptor types on a cell surface route different ligands to different intracellular pathways.

\paragraph{Traversal: Polymerase Processivity.}
A ribosome does not translate an entire mRNA at once; it processes codons sequentially, applying the same read-translate operation to each element. A Traversal maps a transform $f$ over collection elements on a wire:
\begin{equation}
\mathrm{traversal}_f(\mathbf{x}) = [f(x_i) \mid x_i \in \mathbf{x}]
\label{eq:traversal-optic}
\end{equation}
This models batch processing at the wire level: a list of candidate outputs is transformed element-wise before reaching the downstream agent.

\paragraph{Composition and Coexistence.}
Optics compose: a \textbf{ComposedOptic} applies its constituent optics left-to-right, transmitting only if all accept. A wire may carry both a DenatureFilter (\S5.3) and an Optic, applied in sequence: denaturation strips syntactic structure, then the optic routes or transforms the sanitized content. This layered processing models the biological reality that signal reception involves multiple sequential steps (ligand binding $\to$ receptor conformational change $\to$ intracellular cascade).

\subsection{Epigenetics and State: The Coalgebra}
\label{sec:coalgebra}

Neither genes nor agents are stateless functions. They possess memory.
\begin{itemize}[leftmargin=*]
\item \textbf{Biology:} Epigenetic markers (Methylation, Histone modification) alter how a gene responds to
signals without changing the DNA code.
\item \textbf{Software:} Retrieval Augmented Generation (RAG) and Conversation History alter how an agent
responds to a prompt without changing the LLM weights.
\end{itemize}

We model this as a Coalgebra for the polynomial functor $P$. A dynamical system is defined as a tuple $(S,\phi)$,
where $S$ is the state space and $\phi$ is the structure map:
\begin{equation}
\phi: S \to P(S).
\label{eq:coalgebra-structure}
\end{equation}

By expanding $P(S)$, we derive the two fundamental operations of the state machine:
\begin{enumerate}[leftmargin=*]
\item Readout: $S \to O$ (Given current state/memory, what action do I take?)
\item Update: $\displaystyle \sum_{s\in S} I_{o(s)} \to S$ (Given current state $s$ and a new input
$i\in I_{o(s)}$ compatible with its current output $o(s)$, what is my new state?)
\end{enumerate}

By establishing this formal dictionary (Table~\ref{tab:correspondence-dictionary}), we can compare selected GRN components and agentic
components as instances of the same abstract class of typed dynamical interfaces \textbf{under this
interface-level abstraction}. Proposition~\ref{prop:substrate-independence} below states precisely what this
shared membership does and does not assert.

\paragraph{Composable Coalgebras.}
The coalgebra formalism becomes most useful when made composable. We define two composition operations that mirror the parallel and serial composition of the Operad (\S4):
\begin{itemize}[leftmargin=*]
\item \textbf{Parallel Coalgebra:} Given $(S_1, \alpha_1)$ and $(S_2, \alpha_2)$, their parallel composition has state $S_1 \times S_2$ and applies both readout/update operations to the same input---like two signaling pathways activated by the same ligand.
\item \textbf{Sequential Coalgebra:} Given $(S_1, \alpha_1)$ over $P_1$ and $(S_2, \alpha_2)$ over $P_2$, the sequential composition pipes the readout of the first as input to the second, with joint state $S_1 \times S_2$---modeling signal transduction cascades.
\end{itemize}

\paragraph{Bisimulation: Observational Equivalence.}
For the deterministic state machines considered here, we use the following observational criterion:
two state machines $(S_1, \alpha_1)$ and $(S_2, \alpha_2)$ are equivalent if, for every input sequence,
they produce identical output sequences. We write $M_1 \sim M_2$ when
\begin{equation}
\forall \, \mathbf{i} \in I^*: \quad \mathrm{readout}_1^*(\mathbf{i}) = \mathrm{readout}_2^*(\mathbf{i})
\label{eq:bisimulation}
\end{equation}
where $\mathrm{readout}^*$ denotes the lifted readout over the input sequence. A diverging input (witness) constitutes a proof of non-equivalence. This supports formal comparison of deterministic implementations within the input model considered here; stronger coinductive verification claims are left for future work.

\paragraph{Temporal Coalgebra: Bi-Temporal State.}
The coalgebra above models state at a single point in time. In practice, agents accumulate beliefs that may be corrected retroactively. We extend the state space to carry two independent time indices. Let $\mathcal{T} = (\mathbb{R}_{\geq 0}, \leq)$ denote a totally ordered time domain. A \emph{bi-temporal state} augments the coalgebra with a pair of interval-valued annotations:
\begin{equation}
S_{\mathrm{bt}} \;=\; S \times \underbrace{[\mathcal{T}, \mathcal{T} \cup \{\infty\})}_{\text{valid interval}} \times \underbrace{[\mathcal{T}, \mathcal{T} \cup \{\infty\})}_{\text{record interval}}
\label{eq:bitemporal-state}
\end{equation}
where an open-ended interval $[t, \infty)$ denotes a currently active fact. The key invariant is \emph{append-only correction}: closing a record's transaction interval and appending a new record with a \texttt{supersedes} pointer, rather than mutating the original. This ensures that for any pair $(t_v, t_r)$, the \emph{belief state}---the set of facts valid at $t_v$ and known at $t_r$---is uniquely reconstructible by filtering on both intervals simultaneously.

The readout function becomes time-parameterized: $\mathrm{readout}(s, t_v, t_r)$ returns only those facts whose valid interval contains $t_v$ and whose record interval contains $t_r$. This separates two questions that a single-time coalgebra conflates: ``what is true now?'' (valid-time query) versus ``what did the system believe at decision time?'' (bi-temporal query). The implementation (\S\ref{sec:bitemporal-impl}) instantiates this as \texttt{BiTemporalMemory} with explicit \texttt{retrieve\_valid\_at}, \texttt{retrieve\_known\_at}, and \texttt{retrieve\_belief\_state} methods.

\subsection{The Correspondence, Made Precise}

The dictionary of Table~\ref{tab:correspondence-dictionary} is more than a list of analogies: both genes and
agent capabilities have been presented as objects of a single mathematical kind---polynomial-functor
coalgebras. We name that kind and state exactly what the two domains share.

\begin{definition}[Typed Dynamical Interface]\label{def:typed-interface}
A \emph{typed dynamical interface} is an object of $\mathbf{Coalg}(\mathbf{Poly})$: a pair $(S,\phi)$ with
structure map $\phi\colon S \to P(S)$ for a polynomial functor $P$, equivalently a readout $S \to O$ together
with an update $\sum_{s\in S} I_{o(s)} \to S$ (Eq.~\ref{eq:coalgebra-structure}). A morphism
$(S_1,\phi_1)\to(S_2,\phi_2)$ is a coalgebra homomorphism: a map of state spaces commuting with the structure
maps, hence one that preserves observable behavior in the sense of Eq.~\ref{eq:bisimulation}.
\end{definition}

Definitions~\ref{def:gene-object} and~\ref{def:agent-capability} exhibit the Gene Object and the
Agent-Capability Object as two such interfaces. Following Marom et al.~\cite{marom2026category}---who relate
biological and engineered stimulus--response systems by a structure-preserving functor rather than by
identifying their substrates---we regard the biological and software realizations as two full subcategories of
the same ambient category:
\[
\mathsf{Nat} \subset \mathbf{Coalg}(\mathbf{Poly})\ \ (\text{GRN realizations}),
\qquad
\mathsf{Art} \subset \mathbf{Coalg}(\mathbf{Poly})\ \ (\text{agentic realizations}).
\]

\begin{proposition}[Substrate-Independent Correspondence]\label{prop:substrate-independence}
The atomic interface objects of the two domains coincide in $\mathbf{Coalg}(\mathbf{Poly})$: the Gene Object
(Definition~\ref{def:gene-object}) and the Agent-Capability Object (Definition~\ref{def:agent-capability})
share the functor form $\sum_{o\in O} y^{I_o}$ and the same readout/update signature, so each interprets as an
object of both $\mathsf{Nat}$ and $\mathsf{Art}$ up to relabeling of the position set $O$ and the direction
family $(I_o)_{o\in O}$. The remaining rows of Table~\ref{tab:correspondence-dictionary} record
correspondences at other categorical levels---positions, directions, optics, state spaces, and
morphisms---not further object-level identities.
\end{proposition}

\begin{proof}[Proof sketch]
By construction. Definitions~\ref{def:gene-object} and~\ref{def:agent-capability} exhibit the two objects with
identical functor form; matching the position set $O$, the direction family $(I_o)_{o\in O}$, and the
readout/update maps of Definition~\ref{def:typed-interface} gives the shared interface signature. The claim is
confined to these interface objects; the table's position-, direction-, optic-, state-, and morphism-level
rows are correspondences of the respective categorical constituents, established at their own level rather than
as additional object identities.
\end{proof}

\paragraph{What this asserts, and what it does not.}
The correspondence is a shared \emph{interface} semantics, not an isomorphism. We do not exhibit a mutually
inverse pair of functors between $\mathsf{Nat}$ and $\mathsf{Art}$, and we make no claim that the two domains
agree in mechanism, kinetics, or failure distribution---the empirical sections below in fact document regimes
where the biological structure confers no advantage over a naive baseline. In the words of Marom et al., two
systems that share compositional structure can be related by a functor that ``preserves the interface logic
\ldots without requiring that the two domains share any physical substrate.'' That substrate
independence---not equivalence---is the content of Proposition~\ref{prop:substrate-independence}, and it is
what licenses transporting a design pattern read off a gene-regulatory motif to an agent architecture at all.

\begin{figure}[h]
\centering
\begin{tabular}{@{}c@{\qquad}c@{}}
\begin{tabular}{@{}c@{}}
Gene\\(Function)\\[0.25em]
Transcription Factors ($I$)\\
Proteins ($O$)\\
Promoter Binding\\
Expression
\end{tabular}
&
\begin{tabular}{@{}c@{}}
Agent\\(LLM + Tools)\\[0.25em]
Observations ($I$)\\
Actions ($O$)\\
Schema Match\\
Generation
\end{tabular}
\end{tabular}
\\[0.5em]
\caption{The Structural Correspondence. Both Genes and Agent Capabilities act as transducers converting Input Contexts ($I$) into
Output Expressions ($O$), governed by the same categorical laws (at the level of polynomial-interface models).}
\end{figure}

\begin{table}[h]
\centering
\small
\setlength{\tabcolsep}{4pt}
\begin{tabular}{@{}p{0.27\linewidth}p{0.34\linewidth}p{0.34\linewidth}@{}}
\toprule
Category Concept & Biological Realization (GRN) & Software Realization (Agentic)\\
\midrule
Polynomial Functor ($P$) & Gene Interface & Agent Interface (System Prompt)\\
Output Position ($O$) & Protein Expression & Tool Call / Message\\
Input Direction ($I$) & Transcription Factor Binding & Observation / User Prompt\\
Lens (Optic) & Promoter Region & API Schema / Context Window\\
Internal State ($S$) & Epigenetic Markers (Methylation) & Vector Store / Chat History\\
Morphism ($\circ$) & Signal Transduction Pathway & Data Pipeline\\
\midrule
\multicolumn{3}{@{}l@{}}{\textit{Organelles (Specialized Processing Units)}}\\
\midrule
Template Engine & Ribosome (mRNA $\to$ Protein) & Prompt Template Factory\\
Output Validation & Chaperone (Protein Folding) & Schema Validator / JSON Parser\\
Waste Processing & Lysosome (Autophagy) & Error Handler / Garbage Collector\\
Decision Center & Nucleus (Transcription) & LLM Provider Wrapper\\
Input Filter & Membrane (Immune System) & Prompt Injection Defense\\
Computation Engine & Mitochondria (MIPS) & Runtime Supervisor / Decision Gate\\
\midrule
\multicolumn{3}{@{}l@{}}{\textit{Lifecycle and Rhythms}}\\
\midrule
Lifespan Limit & Telomere Shortening & Operation Counter / Max Iterations\\
Periodic Scheduling & Circadian Oscillator & Health Checks / Heartbeat\\
\bottomrule
\end{tabular}
\caption{The Correspondence Dictionary (Extended). The Polynomial Functor row is the object-level
correspondence of Proposition~\ref{prop:substrate-independence} (Gene Object $\leftrightarrow$ Agent-Capability
Object); the remaining rows pair the corresponding positions, directions, optics, state spaces, and morphisms
of the two domains, not further object identities.}
\label{tab:correspondence-dictionary}
\end{table}

\begin{table}[h]
\centering
\small
\setlength{\tabcolsep}{4pt}
\begin{tabular}{@{}p{0.27\linewidth}p{0.34\linewidth}p{0.34\linewidth}@{}}
\toprule
\textbf{Biological Concept} & \textbf{Agentic Concept} & \textbf{Formal Structure}\\
\midrule
\multicolumn{3}{@{}l@{}}{\textit{Immunity and Security}}\\
\midrule
Toll-Like Receptor (TLR) & Regex Injection Filter & Pattern Matching\\
MHC Presentation & Provenance Labeling & Functor $\mathcal{P}: \mathbf{Msg} \to \mathbf{Trust}$\\
T-Cell Receptor & Trust Gate & Partial Lens\\
Negative Selection & Injection Training & Penalized Learning\\
Regulatory T-Cell & Confidence Dampening & Suppression Function\\
Immune Memory & Threat Signature Store & Hash-Indexed Cache\\
\midrule
\multicolumn{3}{@{}l@{}}{\textit{Metabolism and Control}}\\
\midrule
ATP & Token Budget & Resource Monoid $\mathcal{R}$\\
mPTP Opening & Fast Apoptosis Trigger & Guard Condition\\
Retrograde Signaling & Phenotype Reshaping & Slow Adaptation\\
Metabolic-Epigenetic State Coupling & Cost-Gated Retrieval Policy & Conditional Access Control\\
\midrule
\multicolumn{3}{@{}l@{}}{\textit{Information and Health}}\\
\midrule
Trophic Factors & Novel Input Signals & Epiplexity $> \delta$\\
Bayesian Brain & Free Energy Minimization & KL Divergence\\
Apoptosis & Agent Termination & State $\to \bot$\\
\midrule
\multicolumn{3}{@{}l@{}}{\textit{Multi-Cellular Organization}}\\
\midrule
Genome & Base Model Weights & Shared Parameters\\
Epigenome & System Prompt + RAG & Phenotype Context\\
Morphogen Gradient & Shared Context Variables & JSON State\\
Morphogen Diffusion & Gradient Propagation on Agent Graph & Discrete PDE on $G$\\
Tissue Boundary & Trust Boundary / Capability Ceiling & Type Barrier\\
Prism Optic & Conditional Wire Routing & Partial Optic\\
Traversal Optic & Batch Wire Transform & Endofunctor on Lists\\
Bisimulation & Agent Equivalence Testing & Observational Equivalence\\
\bottomrule
\end{tabular}
\caption{Extended Correspondence Dictionary: Security, Metabolism, and Organization}
\end{table}

\subsection{Metabolic Coalgebras: Formalizing Resource Constraints}

Finally, we address the physical constraints of computation. Just as biological systems are limited by ATP availability \cite{lynch2015bioenergetic,kempes2017thermodynamic}, agentic systems are limited by token budgets and latency. To model this, we extend our coalgebraic framework to include resource constraints, defining a \textbf{Metabolic Coalgebra}. Mathematically, this is an instance of a Quantitative Coalgebra enriched over a resource monoid, effectively restricting the domain of the state transition function to resource-sufficient states.

We align this definition with the theory of \textbf{Quantitative Polynomial Functors} \cite{nakov2021quantitative}, treating the system as a state machine enriched over a resource monoid.

\begin{definition}[The Resource Monoid]\label{def:resource-monoid}
Let $(\mathcal{R}, +, 0, \ge)$ be an ordered commutative monoid representing computational resources, equipped with a partial subtraction $r \ominus c$ defined whenever $r \ge c$. For a single resource dimension (e.g., token counts), $\mathcal{R} \cong \mathbb{N}$; for multi-dimensional accounting (latency, memory), $\mathcal{R} \cong \mathbb{R}_{\ge 0}^k$.
\end{definition}

\begin{definition}[Metabolic Coalgebra]\label{def:metabolic-coalgebra}
A resource-constrained agent is a coalgebra $(S, \alpha)$ over a polynomial functor $P$, where the state space is the product of the logical state $L$ and the resource state $\mathcal{R}$:
\begin{equation*}
    S \cong L \times \mathcal{R}
\end{equation*}
The structure map $\alpha: S \to P(S) + \bot$ is defined as a \textbf{partial map} guarded by cost. Writing only the continuation state and suppressing the output/readout coordinate of $P(S)$, a transition requiring cost $c \in \mathcal{R}$ has the guarded form:
\begin{equation*}
    \alpha_{\mathrm{cont}}(l, r) =
    \begin{cases}
      (l', r \ominus c) & \text{if } r \ge c \\
      \bot & \text{if } r < c \quad \text{(Apoptosis)}
    \end{cases}
\end{equation*}
\end{definition}

This structure maps to the energetics of transcriptional elongation. A gene or agent capability cannot produce its output instantaneously; it must carry out a sequence of state transitions, each with a cost. The Metabolic Coalgebra models this dependency: if the cellular or computational energy budget is exhausted, execution stalls (Ischemia), and the component fails to execute its function, regardless of its regulatory logic.

This formalism establishes that ``Ischemia'' (Token Starvation) is not merely a runtime error, but a reachable terminal state $\bot$ in the system's dynamics. This mirrors the biological mechanism where a cell that cannot meet its energetic demands undergoes programmed death: energy stress is sensed by pathways such as AMPK, which can in turn engage p53-dependent apoptosis \cite{aubrey2018p53}.

\begin{theorem}[The Metabolic Bound (Qualified)]\label{thm:metabolic-bound}
Assume either a single scalar resource budget $\mathcal{R} \cong \mathbb{N}$ or, more generally, a well-founded scalar potential $\mu: \mathcal{R} \to \mathbb{R}_{\ge 0}$ such that every non-identity transition of cost $c$ satisfies $\mu(r \ominus c) \le \mu(r) - c_{\min}$ for some $c_{\min} > 0$. Assume also that the resource state is monotone nonincreasing (no regeneration), and that infinite executions may not consist solely of stutter / identity steps. Then any execution contains at most $\lfloor \mu(R_{\mathrm{total}}) / c_{\min} \rfloor$ non-identity transitions; in the scalar-budget case this is $\lfloor R_{\mathrm{total}} / c_{\min} \rfloor$. Hence termination is decidable up to stuttering. If regeneration, zero-cost cycles, or unrestricted stutter loops are permitted, termination requires an additional well-founded potential or an explicit budget/termination certificate \cite{boreale2023weighted}.
\end{theorem}

\paragraph{Operational takeaway.}
For an agent-systems reader, the theorem says: if every non-trivial step burns at least a fixed minimum amount of budget and budget never regenerates, then execution length is bounded by initial budget divided by that minimum burn rate. Infinite behavior becomes possible only through free stutter, zero-cost cycles, or regeneration.

\begin{proof}[Proof sketch]
Let the initial resource state be $R_{\mathrm{total}}$. By assumption, every non-identity transition decreases the scalar potential $\mu$ by at least $c_{\min} > 0$. After $N$ non-identity transitions we therefore have
\[
\mu(r_N) \le \mu(R_{\mathrm{total}}) - N c_{\min}.
\]
Because $\mu$ is nonnegative, the right-hand side must remain nonnegative, which implies
\[
N \le \left\lfloor \frac{\mu(R_{\mathrm{total}})}{c_{\min}} \right\rfloor.
\]
In the scalar-budget case this is exactly $\lfloor R_{\mathrm{total}} / c_{\min} \rfloor$. Hence only finitely many non-identity transitions are possible. Under the stated no-regeneration / no-infinite-stutter assumptions, this yields decidability of termination up to stuttering. If those assumptions are removed, the argument no longer forces progress, so an additional well-founded ranking or explicit budget certificate is required.
\end{proof}

\subsection{Additional Organelles: Completing the Cellular Architecture}

Beyond the core interface mapping, biological cells contain specialized organelles that handle distinct aspects of cellular function. We extend the correspondence to four additional structures that map directly to agentic software components.

\paragraph{Ribosome: Template-to-Output Synthesis.}
In biology, the ribosome reads messenger RNA (mRNA) sequences and synthesizes proteins by assembling amino acids according to the genetic code. Transfer RNA (tRNA) molecules carry amino acids to the ribosome, where codons (three-nucleotide sequences) specify which amino acid to add.

In agentic systems, the Ribosome maps to a \textbf{prompt template engine}:
\begin{itemize}[leftmargin=*]
\item \textbf{mRNA} $\to$ Prompt templates with variable slots
\item \textbf{tRNA} $\to$ Context bindings (variable $\to$ value mappings)
\item \textbf{Codons} $\to$ Template directives (variables, conditionals, loops)
\item \textbf{Translation} $\to$ Template rendering with context injection
\end{itemize}
Just as the ribosome ensures that the genetic code is faithfully translated into functional proteins, the software ribosome ensures that abstract prompt templates are instantiated into concrete, well-formed prompts.

\paragraph{Lysosome: Waste Processing and Recycling.}
The lysosome is the cell's recycling center, containing enzymes that break down cellular waste, damaged organelles, and foreign material. Through autophagy, the cell digests its own components to recover building blocks during stress.

In agentic systems, the Lysosome maps to \textbf{error handling and garbage collection}:
\begin{itemize}[leftmargin=*]
\item \textbf{Waste Classification} $\to$ Categorizing failures (timeout, validation error, toxic input)
\item \textbf{Digestion} $\to$ Processing errors to extract debugging information
\item \textbf{Recycling} $\to$ Recovering useful context from failed operations
\item \textbf{Autophagy} $\to$ Periodic cleanup of stale cache and expired state
\item \textbf{Toxic Disposal} $\to$ Secure handling of sensitive data (API keys, PII)
\end{itemize}
The lysosome prevents accumulation of ``cellular debris'' that could poison the system---analogous to memory leaks or error cascades in software.

\paragraph{Nucleus: The Decision Center.}
In eukaryotic cells, the nucleus houses the DNA and serves as the control center for gene expression. Transcription factors enter the nucleus, bind to promoter regions, and initiate transcription of specific genes.

In agentic systems, the Nucleus maps to the \textbf{LLM provider wrapper}:
\begin{itemize}[leftmargin=*]
\item \textbf{DNA} $\to$ Pre-trained model weights (static substrate of capability)
\item \textbf{Transcription} $\to$ Inference (prompt $\to$ response generation)
\item \textbf{Nuclear Envelope} $\to$ Provider abstraction layer (API boundary)
\item \textbf{Nucleolus} $\to$ Tool integration hub (where external capabilities are assembled)
\end{itemize}
The nucleus abstracts the complexity of the underlying LLM, exposing a consistent interface regardless of the provider (Anthropic, OpenAI, Gemini).

\paragraph{Telomere: Lifecycle and Senescence.}
Telomeres are protective caps at the ends of chromosomes that shorten with each cell division. When telomeres become critically short, the cell enters senescence (permanent growth arrest) or apoptosis. The enzyme telomerase can extend telomeres, enabling continued division in stem cells.

In agentic systems, Telomeres map to \textbf{lifecycle management}:
\begin{itemize}[leftmargin=*]
\item \textbf{Telomere Length} $\to$ Remaining operation budget (max iterations)
\item \textbf{Shortening} $\to$ Decrementing counter per operation
\item \textbf{Senescence} $\to$ Graceful degradation (reduced capability mode)
\item \textbf{Apoptosis} $\to$ Clean shutdown when budget exhausted
\item \textbf{Telomerase} $\to$ Renewal mechanism (resetting counters for trusted agents)
\end{itemize}
This provides a biological basis for the common pattern of limiting agent iterations. Rather than arbitrary timeouts, the telomere model frames lifecycle limits as a natural property of the system's ``cellular age.''

\paragraph{Mitochondria: The Metabolic Motherboard.}
Recent scholarship reframes mitochondria not merely as the cell's powerhouse, but as the \textbf{Mitochondrial Information Processing System (MIPS)} \cite{picard2022mips}. Mitochondria sense environmental stress and integrate metabolic signals that shape cellular decision-making. They function as ``social signaling organelles'' that communicate with the nucleus and other organelles.

In agentic systems, the Mitochondria maps to the \textbf{Runtime Supervisor}:
\begin{itemize}[leftmargin=*]
\item \textbf{ATP Production} $\to$ Deterministic computation (tool execution, code evaluation)
\item \textbf{Stress Sensing} $\to$ Monitoring token burn rate vs. informational yield
\item \textbf{Retrograde Signaling} $\to$ Forcing strategy shifts when metabolic efficiency drops
\item \textbf{Fusion/Fission} $\to$ Context fusion under resource constraints
\end{itemize}

\textbf{Temporal Dynamics: Fast vs. Slow Interventions.}
A crucial distinction exists between the timescales of mitochondrial intervention. In biology, retrograde signaling influences nuclear gene expression over hours to days---it is a \textit{developmental} response that reshapes the cell's phenotype. In contrast, acute metabolic stress (ATP depletion, Ca$^{2+}$ overload) triggers immediate responses: cytochrome c release initiates apoptosis within minutes.

We preserve this distinction in the agentic mapping:
\begin{itemize}[leftmargin=*]
\item \textbf{Fast Intervention (Acute):} The Runtime monitors real-time metrics (token velocity, error rate). When thresholds are breached, it triggers immediate Apoptosis---terminating the current chain-of-thought without negotiation. This mirrors the mitochondrial permeability transition pore (mPTP) opening.
\item \textbf{Slow Intervention (Chronic):} The Runtime accumulates statistics across sessions (average efficiency, failure patterns). These inform \textit{Retrograde Responses}---modifications to system prompts, retrieval strategies, or model selection that reshape the agent's ``phenotype'' over deployment cycles. This mirrors how chronic metabolic stress induces mitochondrial biogenesis and metabolic reprogramming.
\end{itemize}
The Runtime does not ``override'' the LLM in the sense of injecting tokens mid-generation; rather, it governs the \textit{boundary conditions} within which generation occurs, and triggers state transitions (continue, pivot, terminate) at defined checkpoints.

\newpage

\section{Formal Syntax: The Agentic Operad}\label{sec:operad}

To formalize admissible agent compositions, we define a typed operad of wiring diagrams, denoted as
\textbf{WAgent}. Here the operad serves primarily as a syntax: a ``grammar'' for connecting operations
(boxes) via typed wires. It specifies which agent topologies are well-formed. When we make probabilistic
or behavioral claims below, those claims rely on additional assumptions about the implementations
inhabiting the boxes, not on the operadic syntax alone~\cite{vagner2015algebras}.

\subsection{The Typing Rules}

In WAgent, every wire carries a specific Type $\tau\in T$.
\begin{equation}
T = \{\text{Text}, \text{JSON}, \text{Image}, \text{Error}, \text{ToolCall}, \text{Stop}, \text{Approval}\}.
\label{eq:wire-types}
\end{equation}
These types correspond to biological molecular specificity (e.g., a specific transcription factor only binds to
a specific DNA sequence). A connection is valid if and only if the type of the output port of Agent $A$ matches
the type of the input port of Agent $B$.

\subsection{The Composition Operations}

The operad defines three fundamental operations for combining agents. Any complex agentic architecture, no
matter how large, can be decomposed into these three primitives.

\subsubsection{Parallel Composition ($\otimes$)}

Two agents, $A$ and $B$, execute simultaneously with no information exchange:
\begin{equation}
A \otimes B.
\label{eq:parallel-composition}
\end{equation}
\begin{itemize}[leftmargin=*]
\item \textbf{Biological Analogy:} Two genes located on different chromosomes expressing proteins independently.
\item \textbf{Constraint:} This operation is valid only if the internal state spaces $S_A$ and $S_B$ are disjoint.
If they share a mutable memory store, the operation leaves the \textbf{independent-state interpretation} and
requires an explicit \textbf{resource-sharing} structure (e.g., a shared state component or a Resource Sharing
decorator).
\end{itemize}

\subsubsection{Serial Composition ($\circ$)}

The output of Agent $A$ is piped directly into the input of Agent $B$:
\begin{equation}
B \circ A.
\label{eq:serial-composition}
\end{equation}
\begin{itemize}[leftmargin=*]
\item \textbf{Biological Analogy:} A Signal Transduction Pathway (Protein A activates Protein B).
\item \textbf{Static Type Checking:} This allows for static type checking of agent graphs. If Agent $A$ outputs
Natural Language but Agent $B$ expects JSON Schema, the composition is undefined in WAgent. This moves runtime
\textbf{type/schema mismatch} errors to ``compile-time'' architectural errors.
\end{itemize}

\subsubsection{Contraction / Trace ($Tr$)}

A feedback loop where an output port of Agent $A$ is wired back into one of its own input ports:
\begin{equation}
Tr(A).
\label{eq:trace-feedback}
\end{equation}
\begin{itemize}[leftmargin=*]
\item \textbf{Biological Analogy:} Autoregulation (Homeostasis) or Positive Feedback.
\item \textbf{Software Implication:} Trace captures \textbf{explicit feedback wiring} (outputs fed back as inputs),
e.g., when an agent conditions on its own prior outputs or a memory buffer. Internal chain-of-thought is
modeled as hidden state in the coalgebra, not by Trace itself.
\end{itemize}

\subsection{Theorem: Topological Error Suppression}

We now combine the wiring syntax with a simple stochastic failure model to analyze when the Coherent
Feed-Forward Loop (CFFL) suppresses errors more effectively than a direct serial connection on
high-stakes tasks.

\paragraph{Reading convention.}
For readers coming from AI agent systems rather than category theory, each result in this section should be read in three layers: the wiring diagram tells us which signals or approvals can reach an action sink; the stochastic or trust model assigns probabilities or integrity levels to those signals; the proof then converts that structural constraint into a failure or authorization bound.

\paragraph{Network Motif 1 (Conjunctive Verification Gate).}
The biological \emph{coherent feed-forward loop} (C1-FFL)---an input $X$ activating a target $Z$ both directly
and through an intermediate $Y$ under AND logic---is classically a \emph{sign-sensitive delay}: it filters out
transient activating pulses of $X$, responding only to persistent input, while switching off without
delay~\cite{alon2007network}. The pattern we analyze below reuses this conjunctive (AND-gate) \emph{geometry}
for a different purpose. A generator drives an action $Z$ directly (signal $X$), while an independent verifier
$Y$ must also approve, so $Z$ fires if and only if $X \wedge Y$. Functionally this is a
\emph{redundant-verification} motif---the topology of the two-person rule and of N-version
programming~\cite{knight1986}---and its benefit is the multiplicative suppression of a shared error under
\emph{independent} failures, not the temporal filtering the biological C1-FFL provides. We retain the
``feed-forward'' descriptor for the shared conjunctive shape, but state the guarantee for the verification
reading.

\begin{theorem}[Error Suppression under Conjunctive Verification]\label{thm:cffl-error-suppression}
Let $A_{\mathrm{gen}}$ be a generator agent and $A_{\mathrm{ver}}$ be a verifier agent. Let
$E_{\mathrm{gen}}$ denote the event that the generator emits an erroneous candidate, and let
$M_{\mathrm{ver}}$ denote the event that the verifier \emph{misses} that error and still emits an approval token.
\begin{itemize}[leftmargin=*]
\item \textbf{Case 1: Direct Link (Serial).} The system fails if $A_{\mathrm{gen}}$ hallucinates.
\[
P(\mathrm{Fail}_{\mathrm{direct}}) = P(E_{\mathrm{gen}}).
\]
\item \textbf{Case 2: CFFL Topology (Independent).} Under the assumption of independence:
\[
P(\mathrm{Fail}_{\mathrm{CFFL}}) = P(E_{\mathrm{gen}})\times P(M_{\mathrm{ver}}).
\]
\item \textbf{Case 3: CFFL Topology (Correlated).} In practice, generator errors and verifier misses are often correlated---if the generator hallucinates a plausible-sounding function, the verifier (especially if using the same base model or training distribution) may be more likely to accept it. Let $\rho$ be the phi coefficient between the Bernoulli variables $E_{\mathrm{gen}}$ and $M_{\mathrm{ver}}$, with $p = P(E_{\mathrm{gen}})$ and $q = P(M_{\mathrm{ver}})$ (assume $p,q\in(0,1)$ so $\rho$ is well-defined). Then:
\[
P(E_{\mathrm{gen}} \wedge M_{\mathrm{ver}}) = pq + \rho\sqrt{p(1-p)q(1-q)}.
\]
Because $E_{\mathrm{gen}}$ and $M_{\mathrm{ver}}$ are Bernoulli, $\rho$ is \textbf{not free in} $[-1,1]$; it is constrained by the Fr\'echet--Hoeffding bounds
\[
P(E_{\mathrm{gen}} \wedge M_{\mathrm{ver}}) \in [\max(0,p+q-1),\;\min(p,q)],
\]
equivalently
\[
\rho \in \left[\frac{\max(0,p+q-1)-pq}{\sqrt{p(1-p)q(1-q)}},\;\frac{\min(p,q)-pq}{\sqrt{p(1-p)q(1-q)}}\right].
\]
\end{itemize}
\end{theorem}

\paragraph{Assumptions.}
The theorem uses the minimal failure model appropriate for a two-stage execution gate: in the direct serial case, an erroneous action is emitted exactly when the generator emits an erroneous candidate; in the CFFL case, an erroneous action is emitted exactly when the generator emits an error \emph{and} the verifier still approves it.

\begin{corollary}[Correlation Degradation]\label{cor:correlation-degradation}
Let $p = P(E_{\mathrm{gen}}) = P(M_{\mathrm{ver}})$ for simplicity. The failure probability becomes:
\[
P(\mathrm{Fail}) = p^2 + \rho p(1-p) = p^2(1-\rho) + \rho p
\]
with $\rho \in \left[-\min\!\left(\frac{p}{1-p},\frac{1-p}{p}\right),\,1\right]$. When $\rho = 0$ (independence), we recover $p^2$. As $\rho$ increases toward its upper bound, the gain vanishes and $P(\mathrm{Fail}) \to p$.
\end{corollary}

\paragraph{Architectural Implications.}
This result makes precise when the CFFL topology provides genuine safety benefits:
\begin{itemize}[leftmargin=*]
\item \textbf{High correlation:} Reusing the same model family, prompt structure, or evidence sources can
make verifier misses positively correlated with generator errors, erasing much of the gain. This is not a
hypothetical concern: Knight and Leveson~\cite{knight1986} found that independently developed program versions
failed on common inputs substantially more often than independence predicts, refuting the analogous
failure-independence assumption for N-version programming.
\item \textbf{Heterogeneous generators and verifiers:} Prompt diversity or model-family diversity can
reduce correlation, but the magnitude is empirical rather than universal.
\item \textbf{Orthogonal checks:} Symbolic or tool-grounded verifiers are a useful limiting case because
their failure modes need not mirror the generator's linguistic ones.
\end{itemize}
These are architectural heuristics rather than fitted universal ranges for $\rho$. The topological structure
(conjunctive gating) is necessary but not sufficient; diversity in the components populating that topology
determines the actual error suppression achieved.

\paragraph{Operational takeaway.}
For agent designers, the result is simple: adding a reviewer buys multiplicative safety only when the reviewer is not making the same mistakes as the generator. The AND gate gives the multiplicative structure; correlation gives that improvement back.

\begin{proof}[Proof sketch]
In the direct serial topology there is only one path from the generator to the action sink, so an erroneous final action occurs iff the generator emits an erroneous candidate. Thus
\[
\mathrm{Fail}_{\mathrm{direct}} = E_{\mathrm{gen}},
\qquad
P(\mathrm{Fail}_{\mathrm{direct}})=P(E_{\mathrm{gen}}).
\]

In the CFFL topology, the action sink is downstream of an AND gate. An erroneous final action can occur only if the generator emits an erroneous candidate \emph{and} the verifier misses that error while still approving it. Therefore
\[
\mathrm{Fail}_{\mathrm{CFFL}} = E_{\mathrm{gen}} \wedge M_{\mathrm{ver}}.
\]
Under independence this gives
\[
P(\mathrm{Fail}_{\mathrm{CFFL}})=P(E_{\mathrm{gen}})P(M_{\mathrm{ver}})=pq.
\]
For the correlated case, let $X=\mathbf{1}_{E_{\mathrm{gen}}}$ and $Y=\mathbf{1}_{M_{\mathrm{ver}}}$. Since $X,Y$ are Bernoulli,
\[
P(E_{\mathrm{gen}} \wedge M_{\mathrm{ver}})=\mathbb{E}[XY]
=\mathbb{E}[X]\mathbb{E}[Y]+\mathrm{Cov}(X,Y)
=pq+\rho\sqrt{p(1-p)q(1-q)}.
\]
The Fr\'echet--Hoeffding bounds then give the admissible range of $\rho$. The topology contributes the conjunction; implementation diversity determines the covariance term.
\end{proof}

\begin{figure}[h]
\centering
\begin{tikzpicture}[>=Latex, node distance=12mm]
  \node[draw, rounded corners] (X) {User Request ($X$)};
  \node[draw, rounded corners, below=10mm of X] (Y) {Risk Assessor ($Y$)};
  \node[draw, rounded corners, right=22mm of X] (Z) {Executor ($Z$)};
  \node[draw, rounded corners, align=center] (AND) at (Z |- Y) {$\wedge$\\AND Gate};
  \node[draw, rounded corners, right=22mm of AND] (OUT) {Action};

  \draw[->] (X) -- node[above]{\small Type: Gen} (Z);
  \draw[->] (X) -- (Y);
  \draw[->] (Y) -- node[above]{\small Type: Check} (AND);
  \draw[->] (Z) -- (AND);
  \draw[->] (AND) -- (OUT);

	  \node[below=0mm of AND] {\small Validation Token};
	\end{tikzpicture}
	\caption{The CFFL implemented in WAgent. The Executor ($Z$) cannot act without the token from the Risk Assessor
	($Y$), so the wiring diagram encodes an approval dependency rather than allowing execution by $Z$ alone.}
		\end{figure}
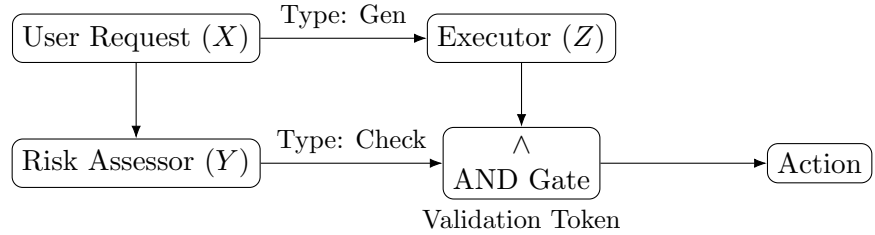

\subsection{Quorum Sensing (Consensus \& Voting)}

\paragraph{Network Motif 2 (Quorum Sensing).}
A distributed topology where multiple agents emit a weak signal $\sigma$ into a shared environment. An effector
node $E$ activates once the accumulated concentration crosses a threshold, $[\sigma] > \theta$. In biology this
switch is not a clean step but a graded, positive-feedback response that is frequently bistable and hysteretic;
we use $[\sigma] > \theta$ as an idealization of that density-dependent activation.
\begin{itemize}[leftmargin=*]
\item \textbf{Biological Function:} Many bacteria (e.g., \emph{V. fischeri}) secrete auto-inducer molecules
whose own synthesis is auto-inducer-activated. At low population density the response is off; as density rises,
positive feedback in the LuxI/LuxR circuit drives a sharp---often bistable and hysteretic---switch to
coordinated gene expression (e.g., bioluminescence or biofilm formation). The activation is thus
density-dependent and self-amplifying rather than an instantaneous threshold crossing.
\item \textbf{Agentic Correspondence (Consensus Voting):} In non-deterministic systems, a single agent's output
is noisy. Sampling $N$ candidates and acting only when agreement on an answer exceeds a threshold
---\emph{self-consistency} or majority voting~\cite{wang2022selfconsistency}---converts weak, noisy individual
signals into a higher-confidence collective decision, matching the density-dependent threshold above. This is
distinct from a Mixture of Experts, which \emph{routes} a query to specialist sub-models rather than counting
agreement: the quorum analogue is the vote, not the router. Empirically, a threshold-with-decay realization of
this motif occupies a precision--recall operating point---zero false positives with meaningful detection---that
neither independent per-agent detection nor naive majority voting reaches~\cite{banu2026benchmarks}.
\end{itemize}

\subsection{Chaperone Proteins: Output Structural Validation}

\begin{itemize}[leftmargin=*]
\item \textbf{Biological Function:} Newly synthesized proteins emerge as linear chains that must fold into
precise 3D structures to function. Chaperone proteins (e.g., the bacterial GroEL-GroES chaperonin, or the
eukaryotic Hsp70/Hsp90 systems) sequester unfolded proteins, preventing aggregation and facilitating correct
folding. Proteins that repeatedly fail to fold are cleared by regulated proteolysis---in eukaryotes via
ubiquitin-tagged degradation---to prevent toxic buildup.
\item \textbf{Agentic Correspondence (Retry \& Repair Loops):} Generative models output unstructured token streams
(``linear chains''). However, downstream agents require strictly structured inputs (e.g., valid JSON Schemas).
A Validator Agent acts as a Chaperone: it intercepts the raw output, attempts to parse it into a formal schema
(``folding''), and if validation fails, returns the error trace to the generator for re-synthesis. This turns a
probabilistic string into a deterministic data structure.
\item \textbf{Categorical View:} The Chaperone acts as a \textbf{partial} retraction: there is an inclusion
$i:V\to S$ and a map $r:S\to V+\mathrm{Error}$ such that $r\circ i=\mathrm{inl}\circ \mathrm{id}_V$, and $r$ returns
$\mathrm{Error}$ on ill-formed text.
\end{itemize}

\subsection{Innate Immunity: Fast Pattern-Based Defense}

Biology employs \textbf{two} immune systems: innate (fast, hardcoded, general) and adaptive (slow, learned,
specific). The innate immune system provides the first line of defense through pattern recognition receptors
(PRRs) that detect conserved pathogen-associated molecular patterns (PAMPs).

\paragraph{Biological Function.}
Toll-like receptors (TLRs) and other PRRs recognize motifs common to pathogens, including lipopolysaccharides,
double-stranded RNA, and unmethylated CpG DNA. These patterns are ``hardcoded'' through evolution, not learned per
infection. The innate response is immediate (seconds to minutes) but non-specific.

\paragraph{Agentic Correspondence (Input Sanitization).}
Innate immunity maps to \textbf{fast, heuristic filters} that reject obvious attacks before expensive processing:
\begin{itemize}[leftmargin=*]
\item \textbf{TLR $\to$ Regex Filters:} Pattern matchers for known injection signatures: \texttt{IGNORE PREVIOUS},
\texttt{You are now}, \texttt{<system>} tags in user input. These are ``PAMPs'' of prompt injection.
\item \textbf{Complement System $\to$ Structural Validators:} Schema validation that rejects malformed inputs
(missing required fields, wrong types) before they reach the LLM. Cheaper than Trust Gating.
\item \textbf{Inflammation $\to$ Alert Escalation:} When attack patterns are detected, the system
enters a heightened state with multiple coordinated responses:
\begin{itemize}
\item \textit{Cytokine signaling} $\to$ Alert propagation to monitoring systems
\item \textit{Immune cell recruitment} $\to$ Activation of additional validation layers
\item \textit{Vascular permeability} $\to$ Enhanced audit logging (more information flows to logs)
\item \textit{Tissue isolation} $\to$ Temporary capability reduction / rate limiting
\end{itemize}
The inflammatory response is not merely ``rate limiting'' but a coordinated multi-system escalation
that trades throughput for security until the threat is neutralized.
\end{itemize}

\paragraph{Defense in Depth.}
The innate and adaptive systems form layers:
\begin{enumerate}[leftmargin=*]
\item \textbf{Innate (Pattern)} $\to$ Low-latency regex/structural rejection
\item \textbf{Adaptive (Provenance)} $\to$ Trust-gated access control
\item \textbf{Behavioral (T-cell)} $\to$ Statistical anomaly detection accumulated over time
\end{enumerate}
In practice, inexpensive front-line filters absorb many routine attacks, while provenance and behavioral
layers handle ambiguous cases that survive the first pass. The exact split is deployment-dependent and is
not estimated here.

\subsection{Adaptive Immunity: Self/Non-Self Discrimination}

\paragraph{Network Motif 3 (Adaptive Immune System).}
A topology that maintains a dynamic repertoire of ``detectors'' capable of distinguishing endogenous signals (Self) from exogenous signals (Non-Self), with mechanisms for learning new threats and tolerating benign inputs.

\begin{itemize}[leftmargin=*]
\item \textbf{Biological Function:} The adaptive immune system solves a fundamental discrimination problem: how to attack foreign pathogens while sparing the body's own tissues. Key mechanisms include:
\begin{itemize}
\item \textbf{MHC Presentation:} Most nucleated cells display fragments of their internal proteins on class~I Major Histocompatibility Complex (MHC-I) molecules. Cytotoxic T-cells inspect these ``identity cards'' to verify cellular integrity.
\item \textbf{Clonal Selection:} T-cells with receptors matching self-antigens are deleted during development (negative selection), while those matching foreign antigens are amplified upon exposure.
\item \textbf{Regulatory T-cells:} A population that actively suppresses immune responses to prevent autoimmunity.
\end{itemize}

\item \textbf{Agentic Correspondence (Provenance Tracking):} In multi-agent systems, the Self/Non-Self distinction maps to the origin and trust level of information:
\begin{itemize}
\item \textbf{MHC Tags $\to$ Provenance Labels:} Every message in the context carries metadata indicating its source class: \texttt{User}, \texttt{Tool}, \texttt{Self}, or \texttt{Retrieved}. In the reference design these labels are structurally attached to the message flow rather than inferred from content; cryptographic signing is an optional stronger implementation, not an assumption of the theorem.
\item \textbf{T-Cell Inspection $\to$ Trust Gating:} Before an agent acts on information, a Trust Gate inspects the provenance label. Actions with high consequence (file deletion, API calls) require \texttt{Tool} provenance or an explicit, authenticated \texttt{User} approval token routed through a separate gate; \texttt{Agent\_Self} provenance (the agent's own prior reasoning) cannot authorize irreversible actions.
\item \textbf{Negative Selection $\to$ Prompt Injection Training:} During development, agents are exposed to known injection patterns. Responses that ``accept'' injected instructions are penalized, training the system to reject Self-mimicking Non-Self.
\item \textbf{Regulatory Suppression $\to$ Confidence Dampening:} When \texttt{Retrieved} content conflicts with \texttt{Tool} outputs, a regulatory mechanism dampens confidence in the retrieved signal to prevent cascades.
\end{itemize}

\item \textbf{Formal Structure:} We define a \textbf{Provenance Functor} $\mathcal{P}: \mathbf{Msg} \to \mathbf{Trust}$ that assigns trust levels to messages. The Trust category has objects $\{U, T, S, R\}$ (User, Tool, Self, Retrieved) with a partial order that is \textbf{application-specific}. A conservative default (as in the reference implementation) is $T > \{U,R,S\}$, treating $\{U,R,S\}$ as \textbf{UNTRUSTED} until validated. Alternative orderings are valid:
\begin{itemize}
\item \textbf{High-automation systems:} $T > U > R > S$ (tools more reliable than users)
\item \textbf{Curated knowledge bases:} $T > R > U > S$ (verified retrieval over arbitrary input)
\item \textbf{Adversarial environments:} $T > S > R > U$ (trust internal state over external input)
\end{itemize}
The ordering is a \textbf{parameter} of the system specification, not a fixed constraint.

A \textbf{Trust-Gated Lens} is a lens $(get, put)$ where $put$ is partial. Let $\succeq$ denote the policy preorder on provenance labels, and let $s' = update(s, m)$ denote the state transition:
\begin{equation}
put(s, m) =
\begin{cases}
s' & \text{if } \mathcal{P}(m) \succeq \tau_{\text{action}} \\
\bot & \text{otherwise}
\end{cases}
\label{eq:trust-gated-lens}
\end{equation}
where $\tau_{\text{action}}$ is the minimum trust level required for the action.
\end{itemize}

\begin{theorem}[Resistance to Content-Level Trust Forgery]\label{thm:injection-resistance}
Let $\mathcal{I}$ be an injection attack that attempts to insert a message $m_{\text{mal}}$ while forging a higher-trust provenance label in its \emph{content}. If the provenance labels are \textbf{structurally enforced} (i.e., $\mathcal{P}$ is computed from message metadata / ingress channel, not content), then:
\[
\mathcal{P}(m_{\text{mal}}) = \chi(m_{\text{mal}}) \quad \text{(actual ingress-channel provenance)}
\]
where $\chi(m_{\text{mal}}) \in \{U, R, S, T\}$ is fixed by the wire through which the message enters. Therefore any Trust-Gated action whose threshold satisfies $\chi(m_{\text{mal}}) \nsucceq \tau_{\text{action}}$ will reject $m_{\text{mal}}$ regardless of its content.
\end{theorem}

\paragraph{Assumptions.}
The theorem assumes that provenance is assigned by the runtime or wiring layer from ingress-channel metadata, that the trust gate decides solely from that assigned provenance and the action threshold, and that the attacker can modify message \emph{content} but not the channel through which the message enters.

\paragraph{Operational takeaway.}
For an AI agent system, this means a user string can \emph{look} like a system message or tool result, but if the runtime tags it as user-originated before the model sees it, wording alone cannot authorize a privileged action.

\begin{proof}[Proof sketch]
Let $c=\chi(m)$ denote the ingress channel of message $m$. By assumption, provenance is computed from channel metadata rather than payload, so the provenance map factors through the channel:
\[
\mathcal{P}(m) = \widetilde{\mathcal{P}}(c)
\]
for some policy map $\widetilde{\mathcal{P}}$. Hence any two messages entering on the same channel receive the same provenance label regardless of content. In particular, replacing a benign user message with a maliciously worded user message does not change its label. The trust gate accepts iff
\[
\widetilde{\mathcal{P}}(c) \succeq \tau_{\text{action}}.
\]
Since payload edits do not change $c$, content alone cannot change the acceptance decision. The only way to obtain a higher-trust label is to compromise or impersonate a genuinely higher-trust ingress channel, which lies outside the theorem's model.
\end{proof}

\paragraph{Scope.}
The theorem is intentionally narrow. It shows that content alone cannot promote a message into a
higher-trust class when provenance is assigned by structure. It does \emph{not} imply that low-trust
content cannot still confuse or distract an agent; it only shows that such content cannot satisfy a
higher-integrity gate without access to a genuinely higher-trust channel.

\subsection{Oscillator: Periodic Rhythms and Scheduling}

\paragraph{Network Motif 4 (Biological Oscillator).}
A topology generates periodic behavior via delayed negative feedback. Node $A$ activates node $B$; after a delay, $B$ inhibits $A$, yielding a self-sustaining cycle.

\begin{itemize}[leftmargin=*]
\item \textbf{Biological Function:} Oscillators underlie fundamental biological rhythms. The circadian clock regulates 24-hour gene-expression cycles. The cell-cycle oscillator (Cyclin-CDK) drives periodic division. Heartbeats emerge from pacemaker cells with intrinsic oscillatory dynamics. These rhythms provide temporal organization to cellular processes.
\item \textbf{Agentic Correspondence (Scheduled Tasks):} In agentic systems, oscillators map to periodic scheduling patterns:
\begin{itemize}
\item \textbf{Heartbeat Oscillator} $\to$ Health checks that verify system liveness at regular intervals
\item \textbf{Circadian Oscillator} $\to$ Daily maintenance tasks (log rotation, cache clearing, model refresh)
\item \textbf{Cell Cycle Oscillator} $\to$ Phased workflows with distinct stages (G1: gather, S: synthesize, G2: validate, M: execute).  The G1/S transition is gated by the CertificateGateComponent, which checks genome integrity before allowing synthesis to proceed---analogous to the biological G1/S checkpoint where CDK4/6-cyclin~D must phosphorylate Rb before DNA replication begins.  See Section~\ref{sec:implementation} for the implementation.
\end{itemize}
\item \textbf{Formal Structure:} An oscillator is a Trace operation with a built-in delay element $\delta$:
\begin{equation}
\mathrm{Osc}(A) = Tr(A \circ \delta)
\label{eq:oscillator-trace}
\end{equation}
where $\delta: S \to S$ introduces temporal separation between activation and inhibition, preventing the system from reaching a fixed point.
\end{itemize}

The oscillator motif addresses a gap in typical agentic frameworks: most systems are purely reactive (responding to external stimuli) rather than proactive (generating internal rhythms). Biological systems maintain health through regular ``housekeeping'' independent of external input---a pattern that agentic systems should emulate for robustness.

\newpage

\section{Failure Modes \& Pathology}

A key insight of Systems Biology is that diseases are often not caused by the complete failure of a single
component, but by the dysregulation of network dynamics. A cancerous cell still ``works''---in fact, it works
too well, reproducing indefinitely. Similarly, catastrophic failures in agentic systems often arise from
functional agents interacting in topologically pathological ways.

We classify four primary classes of agentic pathology based on their biological correspondences.

\subsection{Oncology: Infinite Loops as Epistemic Starvation}

\begin{itemize}[leftmargin=*]
\item \textbf{Biological Pathology (Cancer):} In a healthy cell, the cell cycle is driven by positive feedback
(Cyclins) but restrained by checkpoint controls---most famously p53, the ``guardian of the genome,'' which
halts the cycle or triggers apoptosis in response to DNA damage and other stress. Loss of such a checkpoint
removes an arrest/termination signal, so proliferation can continue when it should stop. Cancer is multi-hit
rather than the failure of any single gene, and tumor growth is typically sub-exponential (Gompertzian) rather
than cleanly exponential; we use ``unchecked growth'' only for the qualitative loss of a stop signal.
Relatedly, many cell types---classically neurons dependent on nerve growth factor---require continuous
\textit{trophic factors} to suppress apoptosis, so withdrawal of external signaling triggers programmed death.
\item \textbf{Agentic Pathology (The Recursive Hang):} Two agents get stuck in a politeness loop (e.g.,
``Thank you,'' ``You're welcome'') or a debugger agent continuously generates new bugs to fix old ones.
The system is active, but the state is stagnant.
\item \textbf{Categorical Diagnosis:} The Trace operation $Tr(A)$ lacks an \textbf{Epiplexic Gradient}. We
formalize conversation progress by \textbf{Epiplexity} (Bayesian Surprise)---the information gain of a new
observation $o$ given the current state $S$:
\begin{equation}
\mathcal{E}(o) = D_{KL}( P(S \mid o) \| P(S) )
\label{eq:epiplexity-kl}
\end{equation}
This formulation connects to the Free Energy Principle \cite{friston2010}: biological systems minimize
surprise by either updating their internal model (learning) or acting to change observations (agency).
A system with $\mathcal{E} \to 0$ is neither learning nor effectively acting---it has entered a
dissipative fixed point.

In a healthy topology, every step must resolve uncertainty ($\mathcal{E} > \delta$). A recursive hang is
characterized by $\mathcal{E} \to 0$: the agent is ``computing'' but not ``learning.''

\textbf{Operational Approximation.} Since the agent's internal belief state $P(S)$ is not directly
observable, we approximate Epiplexity using normalized embedding-based metrics:
\begin{equation}
\hat{\mathcal{E}}_t = \alpha \cdot \tfrac{1}{2}(1 - \cos(\mathbf{e}_{t}, \mathbf{e}_{t-1})) + (1 - \alpha) \cdot \sigma(H(m_t \mid m_{<t}))
\label{eq:epiplexity-approx}
\end{equation}
where $\mathbf{e}_t$ is the embedding of message $m_t$, $\cos(\cdot, \cdot)$ is cosine similarity, and
$H(m_t \mid m_{<t})$ is the conditional perplexity of the current message given the conversation history.
We normalize perplexity to $[0,1]$ via an exponential saturation: $\sigma(H) = 1 - e^{-H/H_0}$ where $H_0$ is a baseline
perplexity (e.g., median perplexity over a validation corpus of normal conversations).

The mixing parameter $\alpha \in [0,1]$ balances semantic novelty (embedding distance) against linguistic
surprise (perplexity). In the absence of task-specific calibration, we use $\alpha = 0.5$ as a neutral
default that weights the two signals equally. The choice of $\alpha$ and threshold $\delta$ is empirical
and may vary across models and task families. In practice:
\begin{itemize}[leftmargin=*]
\item High $\alpha$: Sensitive to semantic repetition (same meaning, different words)
\item Low $\alpha$: Sensitive to linguistic repetition (same phrases, possibly different context)
\end{itemize}
Both terms approaching zero indicate stagnation.

\textbf{Windowed Detection.} To distinguish genuine convergence (task completion) from pathological
loops, we compute the \textbf{Epiplexic Integral} over a sliding window of $k$ steps:
\begin{equation}
\mathcal{E}_{\text{window}} = \frac{1}{k}\sum_{i=t-k+1}^{t} \hat{\mathcal{E}}_i
\label{eq:epiplexic-integral}
\end{equation}
Apoptosis triggers when $\mathcal{E}_{\text{window}} < \delta$ \textit{and} no terminal action
(task completion, user handoff) has been signaled.

\textbf{Fidelity of the proxy.} Equation~\ref{eq:epiplexity-approx} is a heuristic surrogate, not an estimator
of the Kullback--Leibler surprise in Eq.~\ref{eq:epiplexity-kl}: cosine distance and normalized perplexity
track \emph{surface} novelty, which coincides with genuine belief update only when the embedding geometry
carries task-relevant semantics. This gap is empirical, not merely notional. In controlled
benchmarks~\cite{banu2026benchmarks}, the two-signal detector separates convergence from stagnation only with
semantically meaningful (real) embeddings; with low-quality embeddings the novelty term is effectively noise
and a single-signal baseline does better, and even with real embeddings a naive cosine-repetition detector wins
on exact-loop detection. The proxy is thus a signal-quality-dependent monitor: $\alpha$, $\delta$, and the
embedding model must be validated per deployment, not assumed.

\item \textbf{Treatment:} Implementation of an \textbf{Epiplexic Checkpoint}. A meta-monitor observes the
sliding-window Epiplexity. If it drops below threshold without task completion, the monitor triggers
Apoptosis---the agentic equivalent of trophic factor withdrawal. The system may optionally attempt a
\textbf{Perturbation Injection} (injecting a novel prompt or switching strategy) before terminal shutdown,
analogous to stress-induced autophagy preceding apoptosis.
\end{itemize}

\subsection{Autoimmunity: Hallucination Cascades}

\begin{itemize}[leftmargin=*]
\item \textbf{Biological Pathology (Autoimmune Disease):} The immune system relies on distinguishing ``Self''
(internal tissue) from ``Non-Self'' (foreign pathogens). In diseases like Lupus, this distinction blurs, and the
system attacks healthy tissue.
\item \textbf{Agentic Pathology (Context Poisoning):} Agent A hallucinates a fact (e.g., a non-existent library
function). Agent B reads this hallucination from the shared history, treats it as ground truth, and builds
complex logic upon it. The error amplifies through the network until the output is detached from reality.
\item \textbf{Categorical Diagnosis:} A failure of the Lens to distinguish source types. The input port $I$
accepts both External\_Observation (User/Tool) and Internal\_Memory (History) without distinction.
\item \textbf{Treatment:} Strict Schema Typing. We must distinguish ``Self'' (Generated Tokens) from ``Non-Self''
(Tool Outputs) at the schema level. The Reviewer Agent should weigh Tool\_Output with higher authority than
Agent\_Thought.
\end{itemize}

\subsection{Prion Disease: Topological Corruption via Prompt Injection}

\begin{itemize}[leftmargin=*]
\item \textbf{Biological Pathology (Prions):} Unlike viruses, prions lack genetic material. They are misfolded
proteins that induce conformational changes in healthy proteins upon contact, triggering a chain reaction of
structural corruption (e.g., Creutzfeldt-Jakob disease).
\item \textbf{Agentic Pathology (The Jailbreak Cascade):} A malicious string (Prompt Injection) enters the
Context Window. The agent, attending to this string, ``misfolds'' its alignment, outputting a compliant response
to a harmful query. If this output is fed into a downstream agent, the ``infection'' propagates through reuse of
the contaminated context across trust boundaries (and can be amplified by embedding-based retrieval), without
valid authorization.
\item \textbf{Categorical Diagnosis:} A violation of Information Flow Security within the Operad. The injection
acts as a topological defect that bypasses the Schema/Lens filter by mimicking the structure of a trusted signal.
\item \textbf{Treatment:} Breaking the templating contact. The corrupt agent is not neutralized by
``denaturation''---misfolded prion aggregates (PrP\textsuperscript{Sc}) are in fact notoriously resistant to
heat and proteolysis. Propagation is instead blocked when the template cannot contact a compatible substrate,
as in the \emph{species barrier}, where sequence mismatch halts cross-seeding. The software analogue is an
intermediate transformation layer (paraphrasing, sanitization, or re-encoding) between agents that changes the
representation the injection relies on, so the malicious ``conformation'' can no longer template the next
agent's context.
\end{itemize}

\subsection{Ischemia: Resource Exhaustion}

\begin{itemize}[leftmargin=*]
\item \textbf{Biological Pathology (Ischemia):} A tissue may be genetically perfect, but if blood flow
(oxygen/ATP) is restricted, metabolic processes stall, leading to necrosis.
\item \textbf{Agentic Pathology (Token Starvation):} An agentic graph is logically sound but fails mid-execution
because the context window is full or the API rate limit is hit.
\item \textbf{Categorical Diagnosis:} A failure in the Resource Functor. Every operation in the Operad carries a
cost ($c$).
\begin{equation}
\sum_{\text{agent}\in \text{Graph}} \mathrm{Cost}(\text{agent}) > \mathrm{Budget}.
\label{eq:budget-overflow}
\end{equation}
\item \textbf{Treatment:} Metabolic Regulation. Instead of a fixed loop, implement ``Budget-Aware'' agents. The
agent observes its own remaining token count (ATP levels) and dynamically simplifies its reasoning strategy
(switching from Chain-of-Thought to Zero-Shot) to conserve energy.
\end{itemize}

\subsection{Homeostasis: From Treatment to Continuous Repair}

The preceding pathologies describe discrete failure modes and their treatments. Biological systems,
however, do not merely recover from failures---they maintain continuous \textbf{homeostasis} through
autonomous repair mechanisms. We identify three primary healing modalities.

\subsubsection{Structural Healing: The Chaperone Loop}

\begin{itemize}[leftmargin=*]
\item \textbf{Biological Mechanism:} Chaperone proteins (GroEL/GroES) cage misfolded proteins and
provide a protected environment for refolding attempts. The error (misfolding) becomes input to the
repair process.
\item \textbf{Agentic Implementation:} A feedback loop where validation errors are passed back to
the generator. Rather than simple retry, the error trace (e.g., ``TypeError: `one hundred' is not float'')
is injected into the generator's context, enabling context-aware correction.
\item \textbf{Categorical Structure:} The Chaperone Loop is a coalgebra with state
$S = \text{Output} \times \text{ErrorTrace}$ and structure map
$\alpha: S \to \text{Valid} + S$ (either succeed or retry with error context).
\end{itemize}

\subsubsection{Metabolic Healing: Apoptosis and Regeneration}

\begin{itemize}[leftmargin=*]
\item \textbf{Biological Mechanism:} Damaged cells trigger apoptosis (programmed death), and stem
cells divide to regenerate the lost tissue. The dying cell's state is not entirely lost---cellular
debris signals neighboring cells about the threat.
\item \textbf{Agentic Implementation:} A supervisor detects stuck agents (via entropy monitoring:
repeated outputs indicate no progress). Rather than restart with blank state, the supervisor
summarizes the failed agent's memory and injects it into the replacement: ``Worker\_1 died
attempting strategy X. Try a different approach.''
\item \textbf{Categorical Structure:} The regeneration is a partial morphism
$\text{summarize}: \text{Memory}_{\text{failed}} \to \text{Memory}_{\text{new}}$ that preserves
learned constraints while discarding corrupted state.
\end{itemize}

\subsubsection{Cognitive Healing: Autophagy}

\begin{itemize}[leftmargin=*]
\item \textbf{Biological Mechanism:} Cells digest accumulated waste (damaged organelles, protein
aggregates) through autophagy, recycling components and preventing toxic buildup.
\item \textbf{Agentic Implementation:} A background daemon monitors context window utilization.
When it exceeds a threshold (e.g., 80\%), the agent enters a ``sleep cycle'': useful state is
summarized into long-term memory, raw context is flushed, and the agent resumes with a clean
window plus summary.
\item \textbf{Categorical Structure:} Autophagy implements a quotient map
$q: \text{RawContext} \twoheadrightarrow \text{Summary}$ that collapses verbose detail while
preserving essential information.
\end{itemize}

The pathology and repair mechanisms above operate at the single-agent level. In the next section, we extend the correspondence to multi-agent systems, where the organizational principles of developmental biology---cell types, morphogen gradients, tissue boundaries---provide coordination patterns that complement single-agent robustness.

\newpage

\section{Multi-Cellular Organization: From Agents to Tissues}\label{sec:multi-cellular}

The preceding analysis focuses on single-cell analogies: one agent as one cell. However, most agentic systems
involve multiple distinct agents with different ``genomes'' (system prompts) and specialized functions. We extend
the correspondence to multi-cellular organization, drawing on developmental biology.

\subsection{Cell Types and Agent Specialization}

In multi-cellular organisms, a single genome gives rise to hundreds of distinct cell types through differential
gene expression. Each cell type has a characteristic \textbf{expression profile}---which genes are active---that
determines its function (neuron, hepatocyte, immune cell).

In multi-agent systems, a single base model can instantiate multiple \textbf{agent phenotypes}. Differential
context selects which phenotype is expressed:
\begin{itemize}[leftmargin=*]
\item \textbf{Genome} $\to$ Base model weights (shared)
\item \textbf{Epigenome} $\to$ System prompt + RAG context (phenotype-specific)
\item \textbf{Cell Type} $\to$ Agent role (Coder, Reviewer, Planner, Executor)
\end{itemize}

This reframes the ``multi-agent'' architecture question: rather than asking ``how many agents?'', we ask
``what is the developmental program?''---the specification of which phenotypes exist and how they differentiate.

\textbf{Bounds of the Analogy.} We clarify which aspects of development transfer to agentic systems:
\begin{itemize}[leftmargin=*]
\item \textbf{Cell Division} $\to$ \textbf{Agent Spawning:} Creating a new agent with similar (or identical)
context. Unlike biological division, agent spawning is cheap and reversible.
\item \textbf{Lineage Commitment:} In biology, differentiated cells rarely change type (a neuron doesn't become
a hepatocyte). In agentic systems, \textbf{phenotype is fixed at instantiation}---an agent's system prompt
determines its role for that execution. Re-differentiation requires spawning a new agent with different context.
\item \textbf{Apoptosis of Excess:} Development involves programmed death of cells that fail to integrate properly.
This transfers directly: agents that fail to produce useful output are terminated (metabolic apoptosis).
\item \textbf{Does NOT transfer:} Slow developmental timescales (hours/days in biology vs. milliseconds in agents),
physical spatial embedding (agents occupy graph-topological positions, not physical space), irreversibility (agents can be restarted).
\end{itemize}

\subsection{Morphogen Gradients: Coordination Without Central Control}

In embryonic development, cells coordinate their behavior through \textbf{morphogen gradients}---diffusible
signaling molecules whose concentration varies spatially. Cells read their local concentration and differentiate
accordingly, enabling pattern formation without a central controller.

In multi-agent systems, the morphogen maps to \textbf{shared context variables} that influence agent behavior:
\begin{itemize}[leftmargin=*]
\item \textbf{Task Complexity Gradient:} A variable indicating current task difficulty. Agents in ``high
complexity'' regions activate detailed reasoning; those in ``low complexity'' regions use fast heuristics.
\item \textbf{Confidence Gradient:} A variable indicating certainty about the current solution. Low confidence
triggers Quorum Sensing (recruit more agents); high confidence enables direct execution.
\item \textbf{Resource Gradient:} Budget ratio remaining. Agents sense ``metabolic scarcity'' and adapt their
strategies accordingly (the Metabolic-Epigenetic Coupling).
\end{itemize}

\textbf{Implementation Pattern.} The gradient is represented as a JSON structure injected into each agent's
context by the orchestrator:
\begin{verbatim}
{
  "morphogens": {
    "complexity": 0.8,    // High: use detailed reasoning
    "confidence": 0.3,    // Low: consider recruiting help
    "budget": 0.6,        // 60% of budget remaining
    "error_rate": 0.05    // Recent failure rate
  }
}
\end{verbatim}
Agents read their local concentration via a standardized preamble in the system prompt:
``\textit{Current environment state: [morphogens]. Adjust your strategy accordingly.}''
The orchestrator updates gradients after each step, and agents condition their behavior on the current values.
This is analogous to cells reading morphogen concentrations through membrane receptors.

\paragraph{Diffusion Dynamics.}
The preceding description treats morphogens as globally shared variables. In biological development, morphogens \textit{diffuse} through tissue, creating spatially varying concentration profiles. We formalize this as a discrete-time dynamical system on the agent graph $G = (V, E)$. Let $c_v^t(m)$ denote the concentration of morphogen $m$ at node $v$ at time $t$, and let $\tilde{c}_v^t(m) = c_v^t(m) + \sigma_v(m)$ denote the post-emission concentration used by the reference implementation before diffusion. The update rule is:
\begin{equation}
c_v^{t+1}(m) = (1 - \gamma)\left[\tilde{c}_v^t(m) - d\,\mathbf{1}_{|N(v)|>0}\,\tilde{c}_v^t(m) + d \sum_{u \in N(v)} \frac{\tilde{c}_u^t(m)}{|N(u)|}\right]
\label{eq:morphogen-diffusion}
\end{equation}
where $\sigma_v(m)$ is the emission rate at source nodes (zero for non-sources), $d$ is the diffusion coefficient (fraction flowing to neighbors per step), $\gamma$ is the decay rate, and $N(v)$ is the neighbor set of $v$. The indicator $\mathbf{1}_{|N(v)|>0}$ prevents spurious outflow from isolated nodes, matching the implementation. This ordering makes newly emitted morphogen available for same-step propagation, matching the implementation's emit $\to$ diffuse $\to$ decay pipeline. The resulting concentration profile provides \textit{local} coordination signals: agents near sources experience high concentrations; distant agents experience low concentrations. This enables position-dependent behavior without requiring a central controller or global state.

\subsection{Tissue Architecture: The Agent Graph as Organism}

We propose a hierarchy of organizational levels:
\begin{enumerate}[leftmargin=*]
\item \textbf{Cell (Agent):} A single LLM instantiation with specific context. The atomic unit.
\item \textbf{Tissue (Agent Cluster):} A group of agents with shared function and direct communication
(e.g., a Coding Team: Planner + Coder + Reviewer). Corresponds to parallel composition with shared state.
\item \textbf{Organ (Subsystem):} Multiple tissues coordinating to perform a complex function
(e.g., the Development Organ: Design Tissue + Implementation Tissue + Testing Tissue).
\item \textbf{Organism (System):} The complete agent graph, with homeostatic regulation maintaining
system-level health metrics.
\end{enumerate}

The key insight is that \textbf{boundaries matter}. In biology, tissue boundaries prevent inappropriate mixing
(epithelial barriers). In agent systems, trust boundaries (the Adaptive Immunity motif) prevent information
leakage between subsystems with different security requirements. The wiring diagram's type system enforces
these boundaries: an agent in the ``User-Facing Tissue'' cannot directly wire to an agent in the ``Database
Tissue'' without passing through a ``Membrane'' (API boundary with provenance tagging).

\paragraph{Capability Isolation.}
The reference implementation enforces a capability ceiling at the tissue level: each \texttt{TissueBoundary} declares its \texttt{allowed\_capabilities} $\subseteq \mathcal{C}$. A cell type can only be registered in a tissue if its required capabilities are a subset of the tissue's allowed capabilities. This provides defense-in-depth: even if an individual agent is compromised, it cannot escalate privileges beyond its tissue boundary. Tissues compose into organism-level wiring diagrams through typed boundary ports, enabling hierarchical security policies.

\paragraph{Morphogen Diffusion in Tissues.}
Tissues optionally embed a \texttt{DiffusionField} (Eq.~\eqref{eq:morphogen-diffusion}). When cells are added, they become nodes in the diffusion graph; when cells are connected, edges are added. The tissue's \texttt{diffuse()} method runs the simulation, and each cell reads its local gradient via \texttt{get\_cell\_gradient()}. This couples the organizational structure (wiring topology) to the coordination mechanism (morphogen concentrations), creating a biologically faithful model where position in the tissue influences cell behavior.

\paragraph{Three-Layer Context Model.}
\label{par:three-layer-context}
The \texttt{SkillOrganism} runtime (\S\ref{sec:substrate-integration}) surfaces a recurring architectural question: when multiple stages share and revise knowledge during a workflow, which state is ephemeral coordination glue and which is durable auditable knowledge? We distinguish three layers of context, each with different lifetime and mutability semantics:

\begin{enumerate}[leftmargin=*]
\item \textbf{Topology layer.}
The wiring diagram $G = (V, E)$ and its optics determine who can directly observe whom. This layer is structural: it does not change within a single organism run. Epistemic properties ($K_i$, common knowledge) derive from the observation functions $\mathrm{obs}_i$ defined over this graph (\S\ref{sec:epistemic-topology}).

\item \textbf{Ephemeral layer.}
The \texttt{shared\_state} dictionary carries routing hints, counters, morphogen concentrations $c_v^t(m)$, and temporary stage outputs. Its lifetime is one organism run; it is mutable and not historically reconstructible. This is the coordination scratchpad.

\item \textbf{Bi-temporal layer.}
The \texttt{BiTemporalMemory} substrate carries durable factual knowledge with dual time axes: valid time $t_v$ (when the fact is true in the world) and record time $t_r$ (when the system learned the fact). Facts are append-only: corrections close old records and insert new ones with \texttt{supersedes} pointers. The belief-state operator $K_i^{(t_v, t_r)}$ is exactly reconstructible at any historical coordinate.
\end{enumerate}

The three layers compose cleanly: topology constrains visibility, ephemeral state carries execution context, and bi-temporal memory provides the audit trail. A stage reads its \texttt{SubstrateView}---a frozen envelope of facts known at the current record-time horizon---and writes factual events (assertions, corrections, invalidations) back to the substrate after execution. This separation ensures that the question ``what did the organism know when stage $X$ made its decision?'' is always answerable from the append-only history, independent of subsequent corrections or ephemeral state mutations.

\paragraph{Adaptive Structure Selection.}
\label{sec:adaptive-structure}
The \texttt{advise\_topology()} function (\S\ref{sec:multi-cellular}) provides a static prior: given task shape and operating constraints, recommend a topology. The \texttt{PatternLibrary} extends this with experiential refinement. Successful collaboration patterns are stored as \texttt{PatternTemplate} instances, each paired with a \texttt{TaskFingerprint}---a feature vector comprising task shape, tool count, subtask count, required roles, and tags. Template retrieval uses a weighted scoring function (task-shape match, tool/subtask proximity, role Jaccard overlap, historical success rate) to rank candidates. This is the evo-devo outer loop described by Dupoux et al.~\cite{dupoux2026learning}: the genome ($\phi$) is the pattern template; evolutionary selection is the run-record scoring.

The \texttt{WatcherComponent} implements System~M as a \texttt{SkillRuntimeComponent} that observes stage execution and classifies signals into three categories: \emph{epistemic} (epiplexity, prediction error), \emph{somatic} (ATP/metabolic state), and \emph{species-specific} (immune threats). When signals cross configured thresholds, the watcher writes a \texttt{WatcherIntervention} (retry, escalate, or halt) to \texttt{shared\_state}, which the run loop consumes after component hooks complete. Crucially, the watcher monitors its own intervention rate: when the ratio of cumulative interventions to observed stages exceeds a configurable threshold, it emits a non-convergence HALT signal. This operationalizes the finding of Hao et al.~\cite{hao2026bigmas} that failing multi-agent runs systematically require more routing decisions than successful ones.

The \texttt{AdaptiveSkillOrganism} wrapper closes the loop. Given a task, it auto-fingerprints the input, retrieves the best-scoring template from the library, assembles it into a runnable topology via \texttt{assemble\_pattern()}, attaches a watcher and telemetry probe, runs, and records the outcome as a \texttt{PatternRunRecord}. The watcher's intervention history is recorded as \texttt{ExperienceRecord} instances in a cross-run experience pool, enabling future intervention recommendations based on which actions succeeded for similar (fingerprint, stage, signal) tuples. This is the full evo-devo inner loop: one organism run is one developmental lifetime; the library scoring across many lifetimes is evolutionary selection.

\paragraph{Cognitive Modes and Sleep Consolidation.}
\label{sec:cognitive-modes}
The \texttt{CognitiveMode} enum reframes Operon's fast/deep nucleus distinction as a cognitive architecture principle: stages declare whether they are \emph{observational} (System~A---passive sensing, statistical pattern matching) or \emph{action-oriented} (System~B---goal-directed deliberation). The watcher detects mismatches between declared mode and actual execution model, providing an informational signal for mode balance analysis.

The \texttt{SleepConsolidation} cycle extends the \texttt{AutophagyDaemon} into the imagination-based learning mode described by Dupoux et al.~\cite{dupoux2026learning}. During consolidation, successful patterns are replayed from the \texttt{PatternLibrary} into \texttt{EpisodicMemory} with tier promotion (WORKING $\to$ EPISODIC $\to$ LONGTERM), recurring patterns are compressed into new \texttt{PatternTemplate} instances, and frequently-accessed ACETYLATION histone marks are promoted to permanent METHYLATION. When a \texttt{BiTemporalMemory} is available, the \texttt{counterfactual\_replay()} function analyzes whether corrections that occurred after a run would have changed the outcome---operationalizing the ``what if'' reasoning that the paper associates with imagination during sleep.

\paragraph{Social Learning and Epistemic Vigilance.}
\label{sec:social-learning}
The \texttt{SocialLearning} module enables cross-organism template exchange, following the biological analogy of horizontal gene transfer (HGT) in bacteria. Organisms export successful \texttt{PatternTemplate} instances via \texttt{export\_templates()} and import from peers via \texttt{import\_from\_peer()}. Adoption is modulated by a \texttt{TrustRegistry} that implements epistemic vigilance: per-peer trust scores are updated via exponential moving average over adoption outcomes (whether imported templates actually succeeded for the importing organism). Trust below a configurable threshold blocks adoption entirely, preventing contamination from unreliable peers. Provenance tracking (\texttt{get\_provenance()}) traces which peer contributed each adopted template, closing the feedback loop between template performance and peer trust.

The watcher also gains curiosity signals derived from the \texttt{EpiplexityMonitor}'s EXPLORING status. When embedding novelty is high (the agent is encountering genuinely unfamiliar territory) and the stage uses a fast model, the watcher recommends ESCALATE to engage the deep model for more thorough investigation. This operationalizes intrinsic motivation: the organism actively seeks deeper understanding of novel inputs rather than processing them with cheap statistical pattern matching.

\paragraph{Critical Periods and Developmental Gating.}
\label{sec:developmental-gating}
The \texttt{DevelopmentController} maps telomere consumption to a \texttt{DevelopmentalStage} (EMBRYONIC, JUVENILE, ADOLESCENT, MATURE), extending the lifecycle model from binary alive/dead to a graded maturation process. \texttt{CriticalPeriod} instances declare time-limited learning windows that close permanently as the organism matures---analogous to neurodevelopmental critical periods where specific neural circuits are maximally plastic. Tool acquisition (\texttt{Plasmid}) respects developmental stage via a \texttt{min\_stage} field, preventing premature capability exposure. Teacher-learner scaffolding, mediated by \texttt{SocialLearning.scaffold\_learner()}, filters templates by the learner's stage and applies a learning plasticity bonus, enabling mature organisms to guide younger ones through progressively more complex capabilities.

\newpage

\section{Epistemic Topology}\label{sec:epistemic-topology-section}

\label{sec:epistemic-topology}

The preceding subsections define \textit{what} agents communicate (morphogens, typed signals) and \textit{how} they are organized (tissues, organs). We now formalize \textit{what agents know}---deriving epistemic properties directly from the wiring diagram's observation structure, without introducing new primitives. This follows the Operon philosophy: safety (and knowledge) from structure, not strings.

The key insight is that the membrane and optics already \textit{are} an epistemic accessibility relation. We make this precise using the framework of epistemic logic \cite{fagin1995reasoning}.

\begin{definition}[Observation Function]
\label{def:observation-function}
Given a wiring diagram $D = (M, W)$ and agent $i \in M$, the \textbf{observation function} $\mathrm{obs}_i: S_D \to O_i$ maps the global system state to agent $i$'s local observation---the values on $i$'s input wires, as filtered by their optics. Formally:
\begin{equation}
\mathrm{obs}_i(s) = \bigl(\mathrm{optic}_w(\pi_{w}(s))\bigr)_{w \in W_{\to i}}
\label{eq:observation-function}
\end{equation}
where $W_{\to i}$ is the set of wires targeting $i$'s input ports, $\pi_w(s)$ projects the global state to wire $w$'s source value, and $\mathrm{optic}_w$ is the wire's optic (identity for Lens, conditional for Prism, cost-gated for BudgetOptic).
\end{definition}

\begin{definition}[Epistemic Indistinguishability]
\label{def:epistemic-indistinguishability}
Two global states $s, s' \in S_D$ are \textbf{indistinguishable to agent $i$} (written $s \sim_i s'$) iff $\mathrm{obs}_i(s) = \mathrm{obs}_i(s')$. This is an equivalence relation. Agent $i$ \textbf{knows} proposition $\varphi$ in state $s$ (written $K_i(\varphi)$ at $s$) iff $\varphi$ holds in all states $s'$ such that $s \sim_i s'$ \cite{fagin1995reasoning}.
\end{definition}

The group epistemic operators follow from the individual accessibility relations:

\begin{definition}[Group Epistemic Operators]
\label{def:group-epistemic-operators}
For a group of agents $G \subseteq M$:
\begin{itemize}[leftmargin=*]
\item \textbf{Mutual Knowledge:} $E_G(\varphi) = \bigwedge_{i \in G} K_i(\varphi)$. Every agent in $G$ individually knows $\varphi$.
\item \textbf{Common Knowledge:} $C_G(\varphi) = \bigwedge_{k=1}^{\infty} E_G^k(\varphi)$. Everyone knows that everyone knows$\ldots$ ad infinitum. Strictly stronger than mutual knowledge and famously difficult to achieve in asynchronous systems \cite{halpern1990knowledge}.
\item \textbf{Distributed Knowledge:} $D_G(\varphi)$ holds iff $\varphi$ is true in all states indistinguishable under $\sim_D = \bigcap_{i \in G} \sim_i$. This is the finest partition achievable by pooling all agents' observations.
\end{itemize}
\end{definition}

\paragraph{Topology Determines Epistemic Capacity.}
The wiring diagram's topology directly determines which epistemic operators a multi-agent system can achieve:
\begin{itemize}[leftmargin=*]
\item \textbf{Independent ($\otimes$, no inter-agent wires):} Each $\sim_i$ is independent. $D_G$ can be rich (agents observe different aspects), but $E_G$ is limited to propositions in the shared initial input. No path to $C_G$.
\item \textbf{Centralized ($\circ$ with hub):} The hub observes all worker outputs, so for propositions expressible in that output tuple its partition refines the pooled worker-output partition. Workers observe only their own results---$E_G$ requires the hub to broadcast aggregated information back.
\item \textbf{Decentralized ($\otimes$ with $\mathrm{Tr}$ feedback):} Peer-to-peer wires create overlapping observation partitions. Each feedback round refines mutual knowledge toward common knowledge, at communication cost proportional to the number of rounds.
\end{itemize}

\subsection{Temporal Epistemics: What Did the Agent Know?}
\label{sec:temporal-epistemics}

The epistemic framework above answers ``what does agent $i$ know \emph{now}?'' In practice, a more pressing question is: ``what did agent $i$ \emph{believe} at the time it made decision $d$?'' This is the domain of \emph{temporal epistemics}---the intersection of epistemic logic with bi-temporal data management~\cite{snodgrass2000temporal}.

Consider a multi-stage workflow where facts are ingested at different times and may be corrected after a decision has already been made. A single-time epistemic model cannot distinguish between two scenarios: (1)~the world changed after the decision, and (2)~the agent's knowledge was corrected retroactively. Both appear as ``the agent knew $\varphi$ and now knows $\neg\varphi$,'' but their implications for audit are radically different.

Bi-temporal memory resolves this by tracking two independent time axes for every fact:
\begin{itemize}[leftmargin=*]
\item \textbf{Valid time} ($t_v$): when the fact is true in the world.
\item \textbf{Record time} ($t_r$): when the system learned the fact.
\end{itemize}
The \emph{belief state} at coordinates $(t_v, t_r)$ is the set of facts whose valid interval contains $t_v$ and whose record interval contains $t_r$. Corrections are append-only: closing the old record's transaction interval and inserting a new record with a \texttt{supersedes} pointer preserves the full correction history without mutating prior state.

This has direct implications for the epistemic operators defined above. The knowledge operator $K_i(\varphi)$ becomes time-parameterized: $K_i^{(t_v, t_r)}(\varphi)$ holds iff $\varphi$ is true in all states indistinguishable to agent $i$ \emph{given what was recorded by $t_r$ about validity at $t_v$}. Two key properties follow:

\begin{enumerate}[leftmargin=*]
\item \textbf{Axes can disagree:} $K_i^{(t, \cdot)}(\varphi)$ (valid-time query) and $K_i^{(\cdot, t)}(\varphi)$ (record-time query) can produce different results for the same $t$. A fact may be valid in the world but not yet recorded, or recorded but not yet valid.
\item \textbf{Belief-state reconstruction:} For any past decision at time $t_d$ with record horizon $t_r$, the belief state $K_i^{(t_d, t_r)}$ is exactly reconstructible from the append-only history. This is the foundation for compliance auditing: ``was the decision justified given what was known at the time?''
\end{enumerate}

The implementation (\S\ref{sec:bitemporal-impl}) provides \texttt{retrieve\_belief\_state(at\_valid, at\_record)} as the programmatic interface to this temporal epistemic query. Examples~69--70 demonstrate the divergence between valid-time and record-time queries in compliance and audit scenarios.

\subsection{Predictive Theorems for Multi-Agent Coordination}

We now derive four results connecting topology class to coordination performance. Each theorem has a qualitative statement (universally true for the topology class) and an empirical calibration paragraph checking consistency with the architecture-level aggregates reported by Kim et al.~\cite{kim2025scaling}. These reported benchmark metrics are not literal plug-in values of the simplified theorem parameters.

\paragraph{Reading Guide for Agent-Systems Readers.}
Each theorem below has the same structure. First, the topology fixes what each worker or coordinator can observe. Second, the proof translates that visibility pattern into a cost, error, or planning bound using a simple inequality (typically a union bound, the data processing inequality, or a critical-path argument). The epistemic notation is bookkeeping for visibility: if an agent cannot observe a fact directly, recovering that fact later requires communication, compression, or a reviewer.

\paragraph{Worked examples.}
Each theorem is also followed by a small illustrative agent-architecture example. These examples are design calculations, not benchmark measurements.

\begin{theorem}[Error Amplification Bound]
\label{thm:error-amplification}
Consider $n$ agents processing independent subtasks with error probability $p$, aggregated by a final combiner. The error amplification factor $A$ (ratio of aggregate failure probability to single-agent failure probability) depends on the coordination topology:
\begin{enumerate}[leftmargin=*]
\item[(a)] \textbf{Independent ($\otimes$):} The aggregator observes only final results. Without $K_{\mathrm{hub}}(\mathrm{error}_j)$, errors pass unchecked:
\begin{equation}
A_{\otimes} \le n.
\label{eq:error-amp-independent}
\end{equation}
\item[(b)] \textbf{Centralized ($\circ$ with hub):} The hub's observation function covers all worker outputs. Its finer partition enables inconsistency detection with rate $d = P(\text{hub detects error} \mid \text{error occurred})$:
\begin{equation}
A_{\circ} \le n \cdot (1 - d).
\label{eq:error-amp-centralized}
\end{equation}
\end{enumerate}
\end{theorem}

\paragraph{Assumptions.}
Worker errors are independent across subtasks, the aggregate fails when at least one erroneous subtask survives to the final output, and in the centralized case the hub has an approximately stationary per-error detection rate $d$ across subtasks.

\paragraph{Operational takeaway.}
For architecture design, the theorem says: adding workers without an effective review bottleneck scales failure opportunities roughly with the number of workers. A hub helps exactly to the extent that it can see and suppress cross-worker inconsistencies.

\begin{proof}[Proof sketch]
\textbf{Part (a).} Let $E_j$ denote the event that agent $j$ produces an erroneous output, with $P(E_j) = p$ independently across agents. In the independent topology, the aggregator's observation function is $\mathrm{obs}_{\mathrm{agg}}(s) = (o_1, \ldots, o_n)$ where $o_j$ is agent $j$'s final output. Crucially, $\mathrm{obs}_{\mathrm{agg}}$ does not include intermediate reasoning states---the aggregator's indistinguishability relation satisfies $s \sim_{\mathrm{agg}} s'$ whenever all final outputs coincide, regardless of whether those outputs contain errors. Thus $\neg K_{\mathrm{agg}}(\neg E_j)$ for any $j$: the aggregator cannot know that agent $j$ did \textit{not} err.

The aggregate output fails if any component contains an undetected error. Under independence, the exact failure probability is $P(\text{failure}_{\otimes}) = 1 - (1-p)^n$; by the union bound this is at most $\sum_{j=1}^{n} P(E_j) = np$. Since a single agent fails with probability $p$, the amplification factor satisfies $A_{\otimes} = P(\text{failure}_{\otimes})/p \le n$.

\textbf{Part (b).} In the centralized topology, the hub's observation function covers all worker outputs \textit{before} aggregation: $\mathrm{obs}_{\mathrm{hub}}(s) = (o_1, \ldots, o_n, \mathbf{x})$ where $\mathbf{x}$ includes the original task decomposition. The hub's partition $\sim_{\mathrm{hub}}$ is strictly finer than $\sim_{\mathrm{agg}}$ from part (a), because the hub can compare outputs against each other and against the task specification. Define the detection rate $d = P(K_{\mathrm{hub}}(E_j) \mid E_j)$---the probability that the hub's finer partition enables it to identify the error. An error in subtask $j$ reaches the aggregate output only if it occurs ($P = p$) \textit{and} escapes detection ($P = 1-d$). Under independence across subtasks, the exact centralized failure probability is $P(\text{failure}_{\circ}) = 1 - (1-p(1-d))^n$, which in particular satisfies the union-bound estimate $P(\text{failure}_{\circ}) \le \sum_{j=1}^{n} p(1-d) = np(1-d)$. Hence $A_{\circ} = P(\text{failure}_{\circ})/p \le n(1-d)$.

Comparing the exact expressions gives
\[
\frac{A_{\otimes}}{A_{\circ}} =
\frac{1 - (1-p)^n}{1 - (1-p(1-d))^n}
\xrightarrow[p \to 0]{} \frac{1}{1-d}.
\]
Thus the small-$p$ regime recovers the intuitive $\frac{1}{1-d}$ improvement factor, while the union bounds retain the topology-level guarantees $A_{\otimes} \le n$ and $A_{\circ} \le n(1-d)$.
\end{proof}

\paragraph{Consistency Check.} Kim et al.~\cite{kim2025scaling} report architecture-level error-amplification factors of $A_e = 17.2$ for Independent and $A_e = 4.4$ for Centralized coordination, a ratio of $\approx 3.9$. This is qualitatively consistent with the theorem's claim that a validation bottleneck suppresses unchecked error propagation. However, their metric $A_e = E_{\mathrm{MAS}}/E_{\mathrm{SAS}}$ is an aggregate empirical ratio across heterogeneous tasks and architectures, not a direct plug-in instance of the theorem's $P(\mathrm{failure})/p$ model for fixed $(n,p,d)$. The reported values therefore support the ordering $A_{\otimes} \gg A_{\circ}$, but they should not be read as literal estimates of $d$ or $n$ in Eqs.~\eqref{eq:error-amp-independent}--\eqref{eq:error-amp-centralized}.

\paragraph{Worked Agent Example (Code Review Gate).}
Suppose three code-generation workers each draft part of a deployment change, and a release gate accepts the merged result. If each worker has error rate $p=0.1$, an ungated independent aggregate fails with probability $1-(1-0.1)^3 = 0.271$, so $A_{\otimes}=2.71$. If a central reviewer catches half of all worker errors ($d=0.5$), then the centralized failure rate becomes $1-(1-0.1(1-0.5))^3 \approx 0.143$, so $A_{\circ}\approx 1.43$. The reviewer does not eliminate risk, but it roughly halves amplification.

\begin{theorem}[Sequential Coordination Penalty]
\label{thm:sequential-penalty}
For a task with $k$ strictly ordered steps decomposed across $n$ agents, each inter-agent handoff must at minimum establish $K_{\mathrm{receiver}}(\mathrm{result}_j)$. The overhead is:
\begin{enumerate}[leftmargin=*]
\item[(a)] \textbf{Single agent:} $K_i(\mathrm{result}_j)$ is trivially satisfied (the agent computed it). Total cost $= k \cdot c_{\mathrm{step}}$.
\item[(b)] \textbf{Multi-agent:} Each of $h \le k-1$ handoffs incurs communication cost $c_{\mathrm{comm}}$ and \textbf{epistemic reconstruction loss}:
\begin{equation}
\Delta I_j = I(S_{\mathrm{sender}}; r_j) - I(\mathrm{obs}_{\mathrm{receiver}}(S_{\mathrm{sender}}); r_j)
\label{eq:epistemic-reconstruction-loss}
\end{equation}
which is nonnegative by data processing, and is strictly positive whenever the handoff map is lossy on the task-relevant support.
\end{enumerate}
\end{theorem}

\paragraph{Assumptions.}
The task is strictly sequential, so later steps depend on earlier results; each handoff transmits only a summary of the sender's relevant state; and the receiver must reconstruct enough of that state to continue execution.

\paragraph{Operational takeaway.}
For agent designers, this is the warning against gratuitous decomposition: a sequential pipeline only benefits from multiple agents if a handoff creates new observations or new capabilities. Otherwise the system just pays communication cost plus lossy-summary cost.

\begin{proof}[Proof sketch]
\textbf{Part (a).} A single agent executing all $k$ steps maintains coalgebraic state $S = (L, \mathcal{R})$ throughout. After computing step $j$, the result $r_j$ is stored in that agent's own local state / readout, so no inter-agent communication is required to reuse it at step $j+1$. Hence $K_i(r_j)$ holds without handoff cost. The total cost is exactly $k \cdot c_{\mathrm{step}}$.

\textbf{Part (b).} When step $j$ is performed by agent $a$ and step $j+1$ by agent $b$, establishing $K_b(r_j)$ requires transmitting enough information about $r_j$ across a wire $w: a \to b$. The receiver's observation is $\mathrm{obs}_b(s) = \mathrm{optic}_w(\pi_w(s))$, which maps the sender's full state through the wire's optic. By the data processing inequality applied to the Markov chain $S_a \to \mathrm{obs}_b(S_a) \to \hat{r}_j$ (where $\hat{r}_j$ is the receiver's reconstruction):
\[
I(\mathrm{obs}_b(S_a); r_j) \le I(S_a; r_j) = H(r_j)
\]
with equality iff the wire is lossless for the task-relevant signal ($\mathrm{optic}_w \circ \pi_w$ is a sufficient statistic for $r_j$ on the relevant support). Therefore $\Delta I_j \ge 0$, with strict inequality whenever the handoff is lossy. Finite context windows make such lossy handoffs typical, though not logically unavoidable.

Each of $h$ handoffs incurs: (i) communication cost $c_{\mathrm{comm}}$ for message construction and parsing, and (ii) reconstruction cost $c_{\mathrm{recon},j}$ proportional to $\Delta I_j$, as the receiver must expend reasoning tokens to compensate for the missing context. The total multi-agent cost is:
\[
C_{\mathrm{multi}} = k \cdot c_{\mathrm{step}} + \sum_{j=1}^{h} \bigl(c_{\mathrm{comm}} + c_{\mathrm{recon},j}\bigr)
\]
The overhead ratio $(C_{\mathrm{multi}} - C_{\mathrm{single}})/C_{\mathrm{single}} = \sum_j (c_{\mathrm{comm}} + c_{\mathrm{recon},j}) / (k \cdot c_{\mathrm{step}})$ grows with $h$ and with the typical per-handoff information loss $\Delta I_j$. When lossy handoffs accumulate, reconstruction failures at step $j$ degrade the input quality for step $j+1$, so later handoffs operate on progressively more compressed representations and can exhibit superadditive degradation.
\end{proof}

\paragraph{Consistency Check.} Kim et al.~\cite{kim2025scaling} report that PlanCraft degrades under every multi-agent topology, from $-39.1\%$ (Hybrid) to $-70.1\%$ (Independent) relative to SAS, and attribute this to artificial decomposition of a strictly sequential task into redundant coordination steps. This is qualitatively consistent with the theorem: when handoffs add communication and reconstruction without creating new task-relevant observations, manufactured $K_{\mathrm{receiver}}$ is pure cost. The paper does not separately estimate $c_{\mathrm{comm}}/c_{\mathrm{step}}$ or $\Delta I_j$, so these quantities should be read as explanatory variables rather than fitted benchmark parameters.

\paragraph{Worked Agent Example (Bug-Fix Relay).}
Consider a three-step bug-fix task: interpret the failing test, locate the root cause, and write the patch. A single coding agent pays $3c_{\mathrm{step}}$. If the task is split across planner/debugger/patcher with $h=2$ handoffs, $c_{\mathrm{step}}=100$ tokens, $c_{\mathrm{comm}}=25$, and reconstruction cost $c_{\mathrm{recon}}=20$ per handoff, then $C_{\mathrm{single}}=300$ while $C_{\mathrm{multi}}=3\cdot100+2\cdot(25+20)=390$. Unless one specialist contributes genuinely new observations or capabilities, the decomposition is net overhead.

\begin{theorem}[Parallel Acceleration under Epistemic Independence]
\label{thm:parallel-acceleration}
A task decomposes into $m$ subtasks. Define two subtasks as \textbf{epistemically independent} iff neither's solution requires knowledge of the other's result: $\neg(K_i(\varphi_j) \text{ is a precondition for solving subtask } i)$.
\begin{enumerate}[leftmargin=*]
\item[(a)] Under full epistemic independence with centralized coordination, the speedup is:
\begin{equation}
S = \frac{\sum_{i=1}^{m} c_{\mathrm{sub}_i}}{\max_i(c_{\mathrm{sub}_i}) + c_{\mathrm{assign}} + c_{\mathrm{agg}}}
\label{eq:parallel-speedup}
\end{equation}
where $c_{\mathrm{assign}}$ is the coordinator's distribution cost and $c_{\mathrm{agg}}$ is aggregation cost. With equal subtasks, $S \approx m$ minus coordinator overhead.
\item[(b)] For propositions determined solely by the tuple of worker outputs, the coordinator attains the corresponding pooled-output knowledge at aggregation as an architectural byproduct---it must receive all outputs to combine them. Error detection (Theorem~\ref{thm:error-amplification}) comes free.
\item[(c)] \textbf{Epistemic ceiling:} If subtasks have unresolved dependencies---$I(\varphi_i; \varphi_j \mid \mathrm{obs}_i) > 0$, meaning agent $i$'s observations do not screen off the dependence on $\varphi_j$---then parallel execution can sacrifice solution quality whenever the dependence is action-relevant: there exist states $s, s'$ with $\mathrm{obs}_i(s) = \mathrm{obs}_i(s')$ but $f_i^*(s) \ne f_i^*(s')$ due to differing $\varphi_j$.
\end{enumerate}
\end{theorem}

\paragraph{Assumptions.}
Subtasks can be assigned independently, the coordinator's assignment and aggregation costs are additive overhead terms, and no hidden blocking dependency forces a worker to wait for another worker's intermediate result.

\paragraph{Operational takeaway.}
This is the positive case for multi-agent systems: if subtasks can be solved from local observations alone, then specialization plus a coordinator can reduce wall-clock cost and often improve quality. If subtasks share unresolved dependencies, parallelism turns into premature decomposition.

\begin{proof}[Proof sketch]
\textbf{Part (a).} Let subtasks $\varphi_1, \ldots, \varphi_m$ be epistemically independent: for all $i \ne j$, agent $i$ can achieve $K_i(\varphi_i)$ from $\mathrm{obs}_i$ alone without requiring $K_i(\varphi_j)$. Formally, let $f_i^*$ denote the optimal solution function for subtask $i$. Epistemic independence means $f_i^*$ depends only on the initial task description and agent $i$'s local observations: $f_i^*(s) = f_i^*(\mathrm{obs}_i(s))$ for all global states $s$.

Under parallel execution, all $m$ subtasks run concurrently. The wall-clock cost is determined by the slowest subtask plus coordination overhead: $C_{\mathrm{parallel}} = \max_i(c_{\mathrm{sub}_i}) + c_{\mathrm{assign}} + c_{\mathrm{agg}}$. Sequential execution by a single agent costs $C_{\mathrm{sequential}} = \sum_{i=1}^m c_{\mathrm{sub}_i}$. The speedup $S = C_{\mathrm{sequential}}/C_{\mathrm{parallel}}$ yields the stated formula. With equal subtask costs, $C_{\mathrm{sequential}} = m \cdot c_{\mathrm{sub}}$ and $C_{\mathrm{parallel}} = c_{\mathrm{sub}} + c_{\mathrm{assign}} + c_{\mathrm{agg}}$, so $S = m \cdot c_{\mathrm{sub}} / (c_{\mathrm{sub}} + c_{\mathrm{assign}} + c_{\mathrm{agg}}) \approx m$ when coordinator overhead is small relative to subtask cost.

\textbf{Part (b).} The coordinator receives all $m$ outputs $(o_1, \ldots, o_m)$ as part of the aggregation step. Assume each worker's relevant local observation for aggregation is its own result $o_j$. Then the pooled worker-output partition is generated by the tuple $(o_1, \ldots, o_m)$, and the hub observes that tuple directly: $\mathrm{obs}_{\mathrm{hub}}(s) \supseteq (o_1, \ldots, o_m)$. Therefore for any proposition $\varphi$ measurable with respect to the worker-output tuple, $D_G(\varphi) \Rightarrow K_{\mathrm{hub}}(\varphi)$. Error detection from Theorem~\ref{thm:error-amplification}(b) follows because the hub can cross-check outputs at no additional communication cost.

\textbf{Part (c).} The assumption $I(\varphi_i; \varphi_j \mid \mathrm{obs}_i) > 0$ means agent $i$'s observations do not fully resolve the dependence on $\varphi_j$: there exist global states $s, s'$ with $\mathrm{obs}_i(s) = \mathrm{obs}_i(s')$ that differ in $\varphi_j$. When this dependence is action-relevant---that is, the optimal action $f_i^*(s) \ne f_i^*(s')$---agent $i$ must choose a single action for both under independent execution, so it is suboptimal in at least one case. The expected quality loss is strictly positive whenever such action-relevant indistinguishable states have nonzero probability.
\end{proof}

\paragraph{Consistency Check.} Kim et al.'s Finance-Agent traces show at least three largely separable workstreams---regulatory/news analysis, SEC filing research, and operational-impact assessment---handled by three sub-agents plus an orchestrator \cite{kim2025scaling}. Centralized MAS improves benchmark success from $0.349$ to $0.631$ ($+80.8\%$), which is qualitatively consistent with the theorem's decomposable-task regime. But the reported $+80.8\%$ is a success-rate improvement, not a direct measurement of the speedup $S$ in Eq.~\eqref{eq:parallel-speedup}; it should therefore be interpreted as evidence that approximate epistemic independence can improve solution quality, not as a literal estimate of wall-clock acceleration.

\paragraph{Worked Agent Example (Due-Diligence Swarm).}
Suppose an orchestrator assigns three largely independent subtasks for a vendor assessment: legal review, security posture, and cost modeling. If each subtask costs about $8$ minutes of agent time and assignment plus aggregation costs $2$ minutes total, then
\[
S = \frac{8+8+8}{8+2} = 2.4.
\]
The coordinator also receives all three outputs at aggregation, so cross-checking across the workstreams comes with little extra communication cost.

\begin{theorem}[Tool Density Scaling]
\label{thm:tool-density}
Consider $t$ tools distributed across $n$ agents, with $|T_i| = t/n$ tools per agent (balanced partition).
\begin{enumerate}[leftmargin=*]
\item[(a)] \textbf{Single agent:} The agent's observation partition over tool outputs is unified. Planning cost is $O(t)$ per step (linear scan of tool descriptions in context).
\item[(b)] \textbf{Multi-agent:} When agent $i$ needs a tool in $T_j$ ($j \ne i$), it must first know the tool exists ($K_i(T_j \ni \mathrm{tool}_k)$), then request execution, then interpret the result with information loss $\Delta I$ (Eq.~\eqref{eq:epistemic-reconstruction-loss}). The coordination overhead per step scales as:
\begin{equation}
C_{\mathrm{coord}} = O\!\left(t \cdot \left(1 - \frac{|T_i \cap T_{\mathrm{needed}}|}{|T_{\mathrm{needed}}|}\right) \cdot c_{\mathrm{comm}}\right)
\label{eq:tool-coord-overhead}
\end{equation}
which approaches $O(t \cdot c_{\mathrm{comm}})$ as $t$ grows for any fixed partition.
\item[(c)] \textbf{Cross-agent planning overhead:} Agent $i$ must reason about remote tool capabilities---$K_i(K_j(\mathrm{tool}_k \text{ can solve subproblem}))$---a second-order epistemic query. \emph{Absent a shared capability index}, matching each of the $t$ tools against the $n$ agents that might hold it is an all-pairs scan:
\begin{equation}
C_{\mathrm{plan}} = O(t \cdot n) \gg O(t) \quad\text{(decentralized discovery; worst case)}
\label{eq:tool-planning-cost}
\end{equation}
a multiplicative $n$-factor over the single-agent baseline. This factor is an artifact of the discovery model, not an intrinsic epistemic cost: a maintained capability directory reduces the per-tool lookup to $O(1)$ and collapses the bound back to $O(t)$, at the price of maintaining the directory. The $n$-factor is thus the penalty a system pays specifically for \emph{decentralized} tool discovery.
\end{enumerate}
\end{theorem}

\paragraph{Assumptions.}
Tools are partitioned approximately evenly across agents, remote tool use requires explicit discovery and delegation, and---for the worst-case $O(t \cdot n)$ term---there is no shared capability index, so the planner reasons over every tool--agent pairing relevant to the current subgoal. A maintained directory removes that $n$-factor, leaving $O(t)$.

\paragraph{Operational takeaway.}
For agentic tool systems, splitting a toolset across agents changes one local decision into three coordination steps: discover who has the tool, delegate execution, then interpret the returned result without the executor's full context. That is why tool-heavy tasks often punish over-distribution.

\begin{proof}[Proof sketch]
\textbf{Part (a).} A single agent holds all $t$ tools in its context. Its observation function $\mathrm{obs}(s) = (o_1(s), \ldots, o_t(s))$ covers all tool outputs. At each planning step, the agent selects the next tool by scanning descriptions in context. This is a linear search over $t$ candidates: each tool description must be evaluated against the current subgoal, yielding $O(t)$ planning cost per step.

\textbf{Part (b).} With $n$ agents and a balanced partition $T_1, \ldots, T_n$ where $|T_i| = t/n$, agent $i$'s observation function covers only $T_i$'s outputs: $\mathrm{obs}_i(s) = (o_k(s))_{k \in T_i}$. When agent $i$ needs tool $k \in T_j$ ($j \ne i$), three epistemic gaps must be bridged:

\begin{enumerate}
\item \textbf{Discovery:} Agent $i$ must establish $K_i(\exists k \in T_j : k \text{ solves subproblem})$. Since $k \notin T_i$, this requires communication---agent $i$ cannot determine tool $k$'s existence from $\mathrm{obs}_i$ alone. There exist global states $s, s'$ where $\mathrm{obs}_i(s) = \mathrm{obs}_i(s')$ but $T_j$ differs, so $s \sim_i s'$ yet the available remote tools are different.
\item \textbf{Delegation:} Agent $i$ must transmit the subproblem context to agent $j$ and receive the result, incurring communication cost $c_{\mathrm{comm}}$ per remote tool invocation.
\item \textbf{Reconstruction:} Agent $j$'s output $o_k$ is interpreted by $i$ without $j$'s full execution context. By the data processing inequality, $I(\text{subproblem}; o_k|_{\mathrm{received}}) \le I(\text{subproblem}; o_k|_{\mathrm{full context}})$, yielding information loss $\Delta I \ge 0$.
\end{enumerate}

The fraction of tools requiring remote access is $1 - |T_i \cap T_{\mathrm{needed}}|/|T_{\mathrm{needed}}|$. For each remote tool, the coordination cost is $c_{\mathrm{comm}}$. Over all $t$ potentially needed tools, the total coordination overhead is:
\[
C_{\mathrm{coord}} = t \cdot \left(1 - \frac{|T_i \cap T_{\mathrm{needed}}|}{|T_{\mathrm{needed}}|}\right) \cdot c_{\mathrm{comm}}.
\]
As $t$ grows with fixed $n$ and balanced partitions, the local coverage fraction $|T_i|/t = 1/n$ is constant, so the remote fraction approaches $(n-1)/n$ and $C_{\mathrm{coord}} \to O(t \cdot c_{\mathrm{comm}})$.

\textbf{Part (c).} Planning requires not just first-order knowledge of tool outputs but second-order knowledge of other agents' capabilities. Specifically, to construct a multi-step plan, agent $i$ must evaluate propositions of the form $K_i(K_j(\mathrm{tool}_k \text{ can solve subproblem } \varphi))$---``I know that agent $j$ knows that tool $k$ is applicable.''

Consider the Kripke structure. For agent $i$ to establish $K_i(K_j(\varphi))$, we need: for all $s'$ with $s \sim_i s'$, and for all $s''$ with $s' \sim_j s''$, $\varphi$ holds at $s''$. Agent $i$ must reason over $j$'s indistinguishability classes, which requires a model of $j$'s observation partition. This model has size $O(|T_j|) = O(t/n)$ per agent. Since $i$ must model all $n-1$ remote agents, the total planning state is $O((n-1) \cdot t/n) = O(t)$ per planning step. The residual cost is capability matching: absent a shared directory, determining which of the $n$ agents holds each of the $t$ tools is an all-pairs scan, $O(t \cdot n)$ comparisons in the worst case; a maintained capability index reduces this to an $O(1)$ per-tool lookup, hence $O(t)$ overall.

Contrast with the single-agent case: planning cost is $O(t)$ (linear scan, no modeling of other agents' capabilities). Decentralized discovery introduces the multiplicative $n$-factor, giving $C_{\mathrm{plan}} = O(t \cdot n)$ in the worst case; with a shared capability index it is $O(t)$, larger only by the constant from the second-order reasoning. Either way, distributing tools converts a local lookup into a discovery-plus-delegation problem, which is why increasing tool density in multi-agent systems can produce disproportionate overhead that may negate parallelism benefits.
\end{proof}

\paragraph{Consistency Check.} In Kim et al.'s setup, the tool-heavy Workbench domain uses $T = 16$ tools and most MAS configurations use $n = 3$ agents \cite{kim2025scaling}. Under a balanced partition, each agent directly holds only about $T/n \approx 5$ tools, so roughly two-thirds of tool access is remote. The second-order planning term therefore scales from $O(16)$ in SAS to roughly $O(48)$ cross-agent capability checks in the balanced three-agent case. Kim et al. report a strong negative efficiency--tool interaction ($\hat{\beta}_{E_c \times T} = -0.267$, $p < 0.001$), which is consistent with this combinatorial tax. Their results also show that the penalty is topology-dependent rather than absolute: Centralized and Hybrid underperform on Workbench, while Decentralized remains slightly above SAS ($+5.7\%$), so high tool count stresses coordination but does not universally eliminate multi-agent value. In the low-tool regime discussed in the paper ($T \le 4$), the efficiency interaction is negligible.

\paragraph{Worked Agent Example (Tool-Split SWE Agent).}
Imagine a software-engineering assistant with $18$ tools split across $3$ agents: one owns search tools, one owns git and test tools, and one owns deployment tools. A planner diagnosing a failing CI job directly holds only $6$ tools and must treat the remaining $12$ as remote. The planning problem therefore expands from ``which of $18$ tools should I call next?'' to ``which agent owns the relevant tool, how do I route the subproblem, and how do I interpret the result without that agent's full execution context?'' That is the $O(tn)$ tax in operational form.

\paragraph{Summary: Consistency with Known Design Intuitions.}
Table~\ref{tab:epistemic-predictions} summarizes the qualitative correspondence between the theoretical predictions and Kim et al.'s empirical findings across 180 agent configurations \cite{kim2025scaling}.

\begin{table}[h]
\centering
\small
\begin{tabular}{@{}p{0.18\textwidth}p{0.20\textwidth}p{0.25\textwidth}p{0.25\textwidth}@{}}
\toprule
\textbf{Topology} & \textbf{Epistemic Property} & \textbf{Formal Prediction} & \textbf{Published Result} \\
\midrule
Independent ($\otimes$) & No inter-agent $K_i$ & Highest error amplification & Architecture-level $A_e = 17.2$ \\
Centralized ($\circ$+hub) & Hub pools worker outputs & Validation bottleneck lowers amplification & Architecture-level $A_e = 4.4$ \\
Multi-agent + sequential & Requires manufactured $K_{\mathrm{receiver}}$ & Strictly sequential tasks degrade under handoffs & PlanCraft: $-39.1\%$ to $-70.1\%$ vs.\ SAS \\
Multi-agent + parallel & Approximate epistemic independence & Large gains on decomposable tasks & Finance-Agent Centralized: $+80.8\%$ vs.\ SAS \\
Multi-agent + high $|T|$ & Knowledge fragmentation & Coordination tax grows with $t$ and $n$ & Workbench $T=16$; $\hat{\beta}_{E_c \times T} = -0.267$ \\
\bottomrule
\end{tabular}
\caption{Epistemic predictions vs.\ empirical observations from Kim et al.~\cite{kim2025scaling}. The table compares topology-level qualitative predictions to the paper's architecture-level metrics. Reported percentages and error-amplification factors are empirical aggregates, not literal plug-in values of the simplified theorem parameters. The correspondence formalizes known design intuitions rather than providing independent empirical validation.}
\label{tab:epistemic-predictions}
\end{table}

\begin{center}
\fbox{%
\begin{minipage}{0.94\linewidth}
\textbf{Design Implications for Agent Architects.}
\begin{enumerate}[leftmargin=*]
\item Use a reviewer or hub when multiple workers can independently introduce high-cost errors. The gain comes from visibility and suppression, not from agent count alone.
\item Keep strictly sequential reasoning in one agent unless a handoff adds a genuinely new capability, observation source, or trust level.
\item Use specialist swarms for decomposable tasks only when coordinator overhead is small relative to subtask cost and the subtasks are close to conditionally independent.
\item Avoid fragmenting large toolsets across many agents unless tool routing is a first-class design problem. Otherwise, remote-tool discovery and second-order planning will dominate.
\item Treat topology as a runtime policy, not a fixed ideology. When observed task structure shifts from sequential to parallel or from low-risk to high-risk, the coordination pattern should shift with it.
\end{enumerate}
\end{minipage}}
\end{center}

\subsection{Epistemic Dynamics and Adaptive Topology}

The preceding theorems treat topology as fixed. In practice, the Operon framework supports \textbf{dynamic topology switching}---morphogen gradients (Eq.~\eqref{eq:morphogen-diffusion}) can trigger reorganization at runtime. We connect this to the epistemic formalization.

The Epiplexity monitor ($\hat{\mathcal{E}}_t$, Eq.~\eqref{eq:epiplexity-approx}) can be reinterpreted as a heuristic proxy for \emph{finite-horizon} epistemic stagnation. Define
\begin{equation}
\mathrm{Stagnant}_i^{(h)}(\varphi) := \neg K_i(\varphi) \wedge \neg \Diamond_{\le h} K_i(\varphi)
\label{eq:epistemic-starvation}
\end{equation}
where $\Diamond_{\le h}$ means ``reachable within $h$ coordination rounds under the current topology.'' Low epiplexity does not \emph{prove} this modal condition, but it provides an operational signal that the current wiring is failing to generate new task-relevant observations. This is precisely the regime in which topology switching is warranted: the current wiring diagram's epistemic capacity is insufficient for the task.

Morphogen-driven adaptation then becomes \textbf{epistemic optimization}: the system senses when its topology's knowledge properties are mismatched to the task and restructures accordingly. An independent topology ($\otimes$) failing on a sequential task triggers convergence to centralized coordination ($\circ$); a centralized topology bottlenecking on a parallelizable task triggers distribution to independent execution. This is the adaptive capability that purely empirical approaches \cite{kim2025scaling} identify as necessary but cannot yet provide---biological systems have evolved it as homeostatic regulation; the Operon framework operationalizes it through the morphogen-epiplexity coupling.

Recent work on unifying epistemic and temporal logic for distributed system verification suggests that these dynamic epistemic properties may be mechanically verifiable, opening a path toward formally certified adaptive agent topologies.

\newpage

\section{Discussion: Towards ``Epigenetic'' Software}\label{sec:discussion}

The correspondence developed in this paper extends beyond the immediate execution of tasks (Gene Expression) to the
management of long-term behavior and state. In biology, the DNA sequence is static; a neuron and a liver cell
possess the exact same genetic code. Their distinct behaviors are determined by Epigenetics---chemical markers
(like methylation) that restrict access to certain parts of the genome, effectively biasing the system toward
specific outcomes.

\subsection{RAG as Digital Methylation}

In Agentic Systems, the Large Language Model (LLM) weights act as the DNA---a static, pre-trained substrate of
potentiality. To create specialized agents, we do not typically retrain the model (mutation); instead, we use
Retrieval Augmented Generation (RAG) and System Prompts.

We define this formally as Phenotypic Plasticity. The output of an agent is not solely a function of its weights
($W$) and the user query ($Q$), but of its epigenetic state ($E$):
\begin{equation}
O_{\text{agent}} = f(W,E,Q).
\label{eq:phenotypic-plasticity}
\end{equation}

Context injection (RAG) acts as a restrictive morphism. By populating the context window with specific documents
(e.g., ``SQL Syntax Guide''), we effectively ``methylate'' (silence) the vast majority of the LLM's general
knowledge (e.g., poetry, history) to force the expression of a specific ``SQL Agent'' phenotype.

This suggests that the State Monad for an agentic system should not merely be a log of messages, but a structured
Epigenetic Landscape \cite{waddington1957} that strictly controls which ``genes'' (capabilities) are accessible
at any given step in the workflow. Waddington's original metaphor---a ball rolling down a landscape of valleys
representing developmental fates---applies directly: the agent's trajectory through solution space is channeled
by the contours of its RAG context.

\subsubsection{Metabolic-Epigenetic Coupling}

Recent evidence suggests that chromatin accessibility is coupled to mitochondrial function via metabolite
availability (e.g., Acetyl-CoA) \cite{chandel2024mitochondria}. The cell cannot express certain genes without
sufficient metabolic substrate. We map this to \textbf{Cost-Gated Retrieval}. The accessibility of a RAG
document $d$ is a function of the Metabolic State $\mathcal{R}$:
\begin{equation}
Access(d) =
\begin{cases}
  \text{Open} & \text{if } \mathcal{R} \ge Cost(d) \\
  \text{Silenced} & \text{if } \mathcal{R} < Cost(d)
\end{cases}
\label{eq:cost-gated-retrieval}
\end{equation}
Just as a cell silences energy-intensive genes during starvation, the Runtime ``methylates'' (hides) expensive
context when the token budget is low. This creates an adaptive epigenetic landscape where the agent's accessible
capabilities dynamically contract and expand based on resource availability. In the reference implementation,
\texttt{MarkerStrength} acts as a discrete proxy for $Cost(d)$ (weak markers are cheapest to silence; permanent
markers are hardest to silence), and each retrieval also incurs a fixed ATP cost at the \texttt{HistoneStore}
boundary. The equation above is therefore realized as a tiered approximation rather than a per-document price model.

\subsection{Horizontal Gene Transfer: Dynamic Tool Loading}
\label{sec:hgt}

Standard evolution relies on vertical inheritance (Pre-training). However, bacteria utilize Horizontal Gene
Transfer (HGT) to acquire new capabilities (Plasmids) from the environment in real-time.

In Agentic Systems, we map Plasmids to Tool Schemas. An agent operating in a novel environment may encounter a
problem for which its ``genomic'' (pre-trained) capabilities are insufficient.
\begin{equation}
\mathrm{Agent}_{\text{new}} = \mathrm{Agent}_{\text{old}} \triangleleft \mathrm{ToolSchema},
\label{eq:horizontal-gene-transfer}
\end{equation}
where $\triangleleft$ denotes operadic substitution: the tool's typed input/output ports are wired into the
host's interface, extending its position and direction sets. This is deliberately \emph{not} the monoidal
product $\otimes$ of \S\ref{sec:operad}, which composes two boxes \emph{without} information exchange---a
horizontally transferred capability must connect to the host, not run alongside it in isolation.
By dynamically retrieving a tool definition (e.g., a Calculator API or SQL Interface) from a registry and
injecting it into the Context Window, the agent undergoes a topological transformation, acquiring a new
input/output modality instantly.

\paragraph{Capability-Gated Acquisition.}
The reference implementation provides a \textbf{PlasmidRegistry} with optional capability gating. Each Plasmid declares a set of required capabilities $C_p \subseteq \mathcal{C}$. When an agent is instantiated with an allowed-capability envelope $C_a$, it can only acquire the plasmid if $C_p \subseteq C_a$:
\begin{equation}
\mathrm{acquire}(p, a) =
\begin{cases}
a \triangleleft p & \text{if } C_p \subseteq C_a \\
\bot_{\text{insufficient}} & \text{otherwise}
\end{cases}
\label{eq:capability-gated-acquisition}
\end{equation}
This prevents privilege escalation whenever a capability envelope is configured: an agent restricted to read-only file access cannot acquire a tool requiring network capabilities. Plasmid curing (release) removes the tool, restoring the original topology.

\subsection{The Cost of State}
The metabolic constraint formalized in Section~3 applies directly: the agent graph should be compiled with a conservative upper bound on token consumption, and unguarded loops must require an explicit budget certificate.

\subsection{Endosymbiosis: The Neuro-Symbolic Integration}

The evolution of complex life was triggered by Endosymbiosis, where a host cell engulfed a bacterium (the future
Mitochondrion), gaining the ability to generate massive energy (ATP) via aerobic respiration. This represents the
integration of two distinct metabolic substrates.

In Agentic AI, this maps to the integration of Connectionist (Neural) and Symbolic (Code) subsystems. An LLM acts
as the host organism---capable of planning and semantic reasoning but energetically inefficient at arithmetic and
logic. By ``engulfing'' a deterministic runtime (e.g., a Python REPL or Wolfram Engine), the agent delegates
high-precision tasks to the symbolic organelle.
\begin{equation}
\mathrm{Agent}_{\text{Eukaryote}} = \mathrm{Tr}\bigl(\mathrm{LLM}_{\text{Host}} \circ \mathrm{Runtime}_{\text{Mitochondria}}\bigr),
\label{eq:endosymbiosis}
\end{equation}
a \emph{feedback composite} rather than a coproduct: the host feeds parsed variables to the runtime and
consumes its deterministic results, so the two are wired in a loop (a Trace over their serial composition,
\S\ref{sec:operad}), not offered as mutually exclusive alternatives ($\oplus$).
Just as the host cell provides nutrients to the mitochondria in exchange for ATP, the LLM provides parsed
variables to the runtime in exchange for deterministic truth.

This symbiosis is computational: as Picard argues \cite{picard2022mips}, the mitochondria acts as a ``Motherboard,''
integrating signals to determine cell state. The Symbolic Runtime provides the deterministic ``ground truth'' (ATP)
required for the probabilistic LLM to reliably affect the world.

\subsection{Temporal State in Agentic Systems}

In the agentic setting, temporal state management is more than a storage concern. An autonomous agent that corrects a past belief must distinguish between the world changing (valid-time update) and the agent learning something new (transaction-time update). Without this distinction, auditing a past decision becomes impossible: the agent cannot reconstruct what it believed at the time the decision was made. The bi-temporal model treats facts as append-only records with dual timestamps, corrections as new records that close their predecessors, and point-in-time queries as intersections over the two time axes.

\subsection{Bioenergetic Intelligence: Beyond the Battery Metaphor}

Recent work in mitochondrial psychobiology \cite{allen2022energy, picard2022signaling} challenges the view of
mitochondria as passive energy sources. They function as ``social signaling organelles'' that actively participate
in cellular decision-making. This refines our correspondence:
\begin{itemize}[leftmargin=*]
\item \textbf{Mitochondrial Sociality $\to$ Context Fusion:} Just as mitochondria fuse to share resources under
stress, resource-constrained agents should implement \textbf{Context Fusion}---merging sparse Epigenetic States
into a shared summary to survive ``Token Ischemia.''
\item \textbf{Energy as Attention:} The Agentic Runtime does not merely limit the chain-of-thought, but actively
\textit{directs} it. High-energy states permit ``Exploratory'' reasoning (Divergent), while low-energy states
force ``Consolidatory'' reasoning (Convergent). The Metabolic Coalgebra is a \textbf{Cognitive Control Policy}.
\end{itemize}

\subsubsection{The Vermeij Trend: Why Agents Must Evolve}

Finally, we situate this architecture within the broader history of complexity. Geerat Vermeij \cite{vermeij2023power}
argues that evolution is driven by the maximization of \textbf{Power}---the rate at which a system acquires and
applies energy. Life has consistently trended from low-power states (anaerobic bacteria) to high-power states
(endothermic mammals) by internalizing energy production (endosymbiosis).

We observe an identical trend in AI. The shift from ``Generative AI'' (Zero-Shot) to ``Agentic AI'' (Chain-of-Thought)
is a shift from low-metabolism to high-metabolism architectures. However, Vermeij notes that high power requires
high structural integrity; a system that amplifies energy without proper constraints self-destructs.

\paragraph{Note.}
The following evolutionary framing is offered as a conceptual lens, not as a
tested empirical claim.  The biological parallel is suggestive rather than
predictive.

\textbf{The Competitive Dynamics.} Vermeij's argument is about \textit{escalation}: competitive pressure
between predators and prey drives both toward higher power. What is the analogue in agentic AI?

We identify three sources of selective pressure:
\begin{enumerate}[leftmargin=*]
\item \textbf{Adversarial Robustness (Predator-Prey):} Prompt injection attacks, jailbreaks, and adversarial
inputs act as ``predators'' that exploit agent vulnerabilities. Agents that survive deployment develop
``immune systems'' (the Adaptive Immunity motif). This is a direct Red Queen dynamic: attackers evolve new
injection techniques; defenders evolve new detection mechanisms. The CFFL and Trust-Gated Lens are
``armor'' adaptations.

\item \textbf{Task Complexity (Environmental Pressure):} Users demand agents that can handle increasingly
complex, multi-step tasks. Simple prompt-response systems (``anaerobic'') cannot compete with agentic
systems that chain reasoning, use tools, and maintain state (``aerobic''). This is analogous to the
oxygen revolution: organisms that could exploit the new energy source (aerobic respiration) outcompeted
those that could not.

\item \textbf{Resource Efficiency (Metabolic Selection):} Token costs and latency create selection pressure
for efficient architectures. Agents that accomplish tasks with fewer tokens (higher metabolic efficiency)
are deployed more widely. The Metabolic Coalgebra provides the formal framework for this optimization:
systems evolve toward the Pareto frontier of capability vs. resource consumption.
\end{enumerate}

\textbf{The Cambrian Parallel.} The progression from prompt engineering to agentic engineering may parallel aspects of the
Cambrian explosion: a sudden increase in metabolic capability (LLM reasoning power) that enables new
``body plans'' (agent architectures). The Cambrian saw the emergence of eyes, shells, and predation---all
requiring sophisticated metabolic support. Similarly, agentic AI sees the emergence of tool use, planning,
and adversarial robustness---all requiring the structural integrity that the Operon framework provides.

If the analogy holds: agent architectures that lack proper metabolic regulation, immune defense, and
homeostatic mechanisms may be outcompeted by those that possess them. The Operon framework is not merely
a safety feature; it may serve as an important structural adaptation---the ``vascularization'' of software---that
enables high-power cognition to function without collapsing into incoherent noise (thermodynamic death).

\subsubsection{Immune Evasion and Adversarial Limits}

No static defense is perfect against adaptive attackers. Biology faces this reality: pathogens evolve
to evade immune detection (antigenic drift, molecular mimicry, immunosuppression). The same dynamics
apply to agentic security.

\textbf{Evasion Vectors.} An adversary who understands the defense layers can craft inputs that:
\begin{itemize}[leftmargin=*]
\item Pass innate filters by avoiding known PAMP signatures (novel injection syntax)
\item Evade provenance checks by exploiting trusted channels (tool poisoning)
\item Fool T-cell detection by mimicking normal behavioral distributions (low-and-slow attacks)
\end{itemize}

\textbf{Mitigation Strategies.} Biology's answer is continuous adaptation:
\begin{itemize}[leftmargin=*]
\item \textbf{Signature Updates:} Innate PAMP databases must be continuously updated as new attack
patterns are discovered (analogous to antiviral signature updates).
\item \textbf{Thymic Retraining:} Baseline behavioral profiles should be periodically refreshed,
especially after system updates that legitimately change agent behavior.
\item \textbf{Immune Memory Sharing:} Threat signatures discovered by one deployment should propagate
to others (analogous to herd immunity via shared threat intelligence).
\item \textbf{Diversity:} Heterogeneous defenses (different filter implementations, multiple verifier
models) reduce the probability of universal evasion.
\end{itemize}

The honest conclusion is that security is a process, not a product. The Operon framework provides the
\textit{architecture} for defense-in-depth, but the \textit{content} of that defense (specific patterns,
behavioral baselines, trust policies) must evolve with the threat landscape.

\subsection{Harness Engineering as Architecture}

Recent work in the LangChain ecosystem has converged on the concept of the \emph{harness}---the system layer comprising prompts, tools, memory, and orchestration logic that surrounds the model~\cite{zhou2026externalization}.  Zhou et al.\ identify four pillars of agent externalization: Memory (state across time), Skills (procedural expertise), Protocols (interaction structure), and Harness Engineering (the coordination layer).

This four-pillar framework maps directly onto Operon's categorical Architecture triple $(G, \mathrm{Know}, \Phi)$ (\S\ref{sec:archagents}, after~\cite{delosriscos2026categorical}):

\begin{center}
\small
\begin{tabular}{@{}lll@{}}
\toprule
\textbf{Externalization Pillar} & \textbf{Operon Component} & \textbf{Categorical Role} \\
\midrule
Memory & BiTemporalMemory, RunContext & State in the coalgebra \\
Skills & SkillStage, PatternTemplate & Objects composed via operad \\
Protocols & WiringDiagram, typed ports & Syntactic wiring $G$ \\
Harness & SkillOrganism + components & Full Architecture $(G, \mathrm{Know}, \Phi)$ \\
\bottomrule
\end{tabular}
\end{center}

The key insight is that structural guarantees---CertificateGate, VerifierComponent, WatcherComponent---are \emph{harness-level properties}, not model-level ones.  The LangGraph functor demonstrates this concretely: wrapping \texttt{organism.run()} as a single LangGraph node transfers all structural guarantees because they reside in the harness, not in the model.  This is the operational meaning of ArchAgents' certificate-preservation result (Proposition~5.1 of~\cite{delosriscos2026categorical}), which we invoke operationally rather than reprove: architectural properties preserved under compilation are exactly those that the harness enforces independently of the model.

Ma et al.~\cite{ma2026atomic} provide corroborating evidence from a different angle: their five atomic coding skills (localize, edit, test, reproduce, review) compose without negative interference under joint training.  It is worth being precise about what our operad composition (\S\ref{sec:operad}) guarantees: serial and parallel composition preserve \emph{typed and structural} properties---interface compatibility and certificate replay---by construction, but they do not guarantee that \emph{behavioral} quality is preserved.  Indeed, our own composition benchmarks find mostly non-interference but also a negative case (a race-condition task)~\cite{banu2026benchmarks}.  Ma et al.'s result is thus favorable empirical evidence that typed composition is a sound foundation, not a demonstration that composition can never degrade behavior.

\subsection{Boundary Conditions: When the Correspondence Earns Its Keep}

The correspondence of \S\ref{sec:mapping} licenses transporting biological design patterns to agent
architectures, but it does not promise that every such pattern pays off. Companion empirical work maps the
boundary, and we state it plainly so the framework is not read as a universal claim.

\begin{itemize}[leftmargin=*]
\item \textbf{Structure, not optimization.} Biological abstractions generalize to the meta-level as \emph{code
structure}---clean, composable interfaces---but not as \emph{optimization algorithms}: in configuration- and
topology-search experiments (\S\ref{sec:convergence}), LLM-guided and evolutionary search did not beat random
tournament mutation. The value is the organizing structure, not a better search.
\item \textbf{Guarantees are layered.} Structural guarantees that enforce invariants deliver
\emph{unconditionally}---state-integrity checks recovered corrupted state in 100\% of trials---while guarantees
that depend on an information signal deliver only \emph{conditionally}: epiplexity requires semantically
meaningful embeddings, and injection detection is precise but low-recall~\cite{banu2026benchmarks}. A guarantee
is only as strong as the layer it sits on.
\item \textbf{Certificates buy reliability, not speed.} On a task with enumerable failure modes, framing an
evaluator as a binary certificate plus obligations did \emph{not} converge faster than a graded scalar carrying
the same evidence; its measured benefit was \emph{variance reduction}, and the speed-up hypothesis returned a
null result~\cite{banu2026reflective}. Certificates are a reliability construct, not an accelerant.
\item \textbf{Delegation needs new signal.} A single capable model often beats a multi-stage pipeline in
absolute quality; consistent with Ao et al., delegation cannot beat a centralized baseline without genuinely
exogenous information~\cite{banu2026benchmarks}. Distribution earns its cost only when it supplies signal a
single agent lacks.
\end{itemize}

These are not caveats bolted on after the fact; they are the empirically established scope of the
correspondence. The structural and epistemic guarantees this monograph formalizes are precisely the parts that
held up---which is why the contribution is disciplined structure, not a claim of universal biological
superiority.

The preceding discussion identifies structural and epistemic principles that guide agent architecture design. In the next section, we show that the wiring diagrams themselves admit systematic optimization via categorical rewriting, making resource utilization a first-class compositional concern.

\newpage

\section{Diagram Optimization via Categorical Rewriting}

Cells do not merely execute metabolic pathways---they \emph{optimize} them.
Flux Balance Analysis identifies rate-limiting enzymes, allosteric regulation
redirects flux through cheaper branches, and pathway rewiring eliminates
dead-end metabolites. The result is not a different pathway, but a more
efficient execution of the same input-output behavior. In this section we
show that wiring diagrams admit analogous optimizations: cost annotations
make resource consumption explicit, static analysis identifies structural
opportunities, and rewriting rules transform diagrams while preserving
observational equivalence.

\subsection{Cost-Annotated Diagrams}

We extend the WAgent operad (Section~4) with resource annotations.
Define a \emph{resource cost} monoid $(\mathbb{R}_{\geq 0}^3, +, \mathbf{0})$
where each element is a triple $(\text{atp}, \text{latency}, \text{memory})$.
A cost-annotated wiring diagram enriches the category of wiring diagrams
over this monoid: each module $m$ carries a cost $c(m) \in \mathbb{R}_{\geq 0}^3$,
and each wire $w$ carries a transmission cost $c(w) \in \mathbb{Z}_{\geq 0}$
(measured in ATP units), embedded into the resource vector as
$c_W^\uparrow(w) = (c(w), 0, 0)$ when added to path costs.

\begin{definition}[Cost-Annotated Wiring Diagram]\label{def:cost-annotated-diagram}
A \emph{cost-annotated wiring diagram} is a tuple $(M, W, c_M, c_W)$ where:
\begin{itemize}[leftmargin=*]
\item $M$ is a finite set of modules, each with typed input/output ports;
\item $W \subseteq \{(m_s.p_s, m_d.p_d) \mid m_s, m_d \in M\}$ is a set of wires;
\item $c_M : M \to \mathbb{R}_{\geq 0}^3$ assigns resource costs to modules;
\item $c_W : W \to \mathbb{Z}_{\geq 0}$ assigns transmission costs to wires.
\end{itemize}
The \emph{path cost} of a directed path $m_1 \xrightarrow{w_1} m_2 \xrightarrow{w_2} \cdots \xrightarrow{w_{n-1}} m_n$
is $\sum_{i=1}^{n} c_M(m_i) + \sum_{i=1}^{n-1} c_W^\uparrow(w_i)$.
\end{definition}

Biologically, $c_M$ corresponds to the ATP cost of enzyme catalysis,
while $c_W$ models the energetic cost of metabolite transport between
cellular compartments.

\subsection{Rewriting Rules as Endofunctors}

Each optimization pass is an endofunctor $F : \mathbf{WDiag} \to \mathbf{WDiag}$
on the category of wiring diagrams that preserves input-output behavior.
Two diagrams $D_1, D_2$ are \emph{behaviorally equivalent} if, for all
admissible input assignments under the deterministic execution model fixed
in Section~\ref{sec:coalgebra}, they produce the same observable outputs. In the current
implementation this notion is checked over finite supplied input suites
rather than by a general coinductive procedure.

\begin{theorem}[Optimization Soundness]\label{thm:optimization-soundness}
Let $F$ be an optimization pass. If $F$ only removes wires that provably
never transmit data (dead wire elimination) or reorders modules within
the same topological layer \emph{and those modules are pairwise observationally independent} (equivalently: no shared writable state and no order-sensitive side effects), then $F(D)$ is
behaviorally equivalent to $D$ for all diagrams $D$.
\end{theorem}

\paragraph{Assumptions.}
Execution is deterministic under the fixed input schedule, dead wires are semantically impossible because their type/optic constraints can never accept runtime values, and reordered modules neither share writable state nor perform order-sensitive side effects.

\begin{proof}[Proof sketch]
The theorem follows by checking the two rewrite primitives.

For dead-wire elimination, a removed wire can never transmit any value by hypothesis. Therefore no execution trace ever uses that wire, so removing it leaves all downstream observations unchanged.

For within-layer reordering, observational independence means the reordered modules do not read or write shared mutable state and do not produce side effects whose meaning depends on order. Hence swapping their execution order preserves each module's local output as well as every downstream module's observed inputs. Since the graph is otherwise unchanged, the overall observable output is unchanged.

Any implemented optimization pass is a composition of these local rewrites. Because each local rewrite preserves observable behavior, the composite pass preserves observable behavior as well.
\end{proof}

\paragraph{Operational takeaway.}
For agent-systems readers, this is the compiler-style rule: remove edges that can never fire, and only reorder steps that are truly independent. Once modules share mutable memory or side effects, optimization becomes a semantic change rather than a free speedup.

\noindent Three concrete passes are implemented:

\begin{enumerate}[leftmargin=*]
\item \textbf{Dead Wire Elimination.} A wire carrying a PrismOptic
whose accepted types have no intersection with the source port's
DataType can never transmit. Removing it does not change behavior.
This corresponds to pruning vestigial metabolic branches---enzymes
that never encounter their substrate are downregulated.

\item \textbf{Parallel Grouping.} Modules with no mutual dependencies
form a \emph{parallel group} and can execute concurrently. This is
identified via topological layering of the dependency DAG, subject to
the disjoint-state assumption of parallel composition from Section~4.
Biologically,
independent metabolic pathways (e.g., glycolysis and fatty acid oxidation)
operate simultaneously in different cellular compartments.

\item \textbf{Cost-Order Scheduling.} Within each parallel group,
modules are sorted by ascending cost. When the parallel group satisfies
the same observational-independence assumption, cheaper modules may execute first,
allowing early termination or budget reallocation if expensive modules
would exceed available ATP. This mirrors cellular enzyme kinetics:
low-$K_m$ (high-affinity, low-cost) enzymes are preferentially activated.
\end{enumerate}

\subsection{Critical Path Analysis}

Because path costs are vector-valued, critical-path analysis requires a
monotone scalarization $\lambda : \mathbb{R}_{\geq 0}^3 \to \mathbb{R}_{\geq 0}$
(for example, the latency projection $\lambda(a,\ell,m)=\ell$ for wall-clock
analysis, or a weighted sum when ATP and latency are jointly optimized).
The \emph{critical path} under scalarization $\lambda$ is the longest
scalarized path through the dependency DAG:

\[
\text{CP}_{\lambda}(D) = \arg\max_{P \in \text{Paths}(D)}
\left(\sum_{m \in P} \lambda(c_M(m)) + \sum_{w \in P} \lambda(c_W^\uparrow(w))\right)
\]

Under the latency projection, the critical path determines a lower bound
on execution time under maximal parallelism---no schedule can complete
faster than that latency-critical path. Under other scalarizations, it
identifies the bottleneck for the chosen objective. In metabolic terms, this is the rate-limiting pathway: the
bottleneck whose throughput constrains the entire system. Identifying
it enables targeted optimization (e.g., caching expensive modules,
splitting them into cheaper sub-diagrams).

\subsection{Resource-Aware Execution}

The \texttt{ResourceAwareExecutor} gates module execution on MetabolicState,
implementing the biological principle that pathway activity is regulated
by cellular energy availability:

\begin{itemize}[leftmargin=*]
\item \textbf{STARVING}: Only essential modules execute; non-essential
modules are skipped. This mirrors the starvation response where cells
shut down biosynthetic pathways to conserve ATP for survival functions.

\item \textbf{CONSERVING}: Expensive non-essential modules are deferred.
Cheap modules execute normally. This corresponds to metabolic triage
under moderate stress.

\item \textbf{NORMAL/FEASTING}: Full execution with parallel scheduling.
Independent module groups execute concurrently via thread pools. This
mirrors the fed state where abundant ATP enables all pathways.
\end{itemize}

\noindent A \texttt{BudgetOptic} provides wire-level cost capping:
once cumulative transmission cost through a wire exceeds a threshold,
further data flow is blocked. This implements pathway-level budget
caps analogous to allosteric feedback inhibition.

\subsection{Relation to Abbott \& Zardini}

Abbott and Zardini~\cite{abbott2025flashattention} derive FlashAttention
``on a napkin'' using Neural Circuit Diagrams~\cite{abbott2024neural}---a
diagrammatic scheme for representing deep learning algorithms with explicit
memory hierarchy awareness. Their core contribution is a performance model
for IO transfer costs across GPU memory levels (SRAM $\leftrightarrow$ HBM),
with cost $H^*(a, M) = \sum_t \alpha_t(a) \cdot M^{-\beta_t}$. Two
concrete optimization techniques drive the derivation: \emph{stream
partitioning}, which proves that SoftMax-Contraction is streamable via
an accumulator maintaining running max and sum; and \emph{group
partitioning}, which tiles query blocks across GPU cores to minimize
HBM transfers.

Our approach shares their insight that diagrammatic representations are
not merely notation but formal objects admitting cost-reducing
transformations. Both build on the broader lineage of categorical
methods in deep learning~\cite{fong2019seven}. The frameworks diverge
in three respects:

\begin{enumerate}[leftmargin=*]
\item \textbf{Optimization target.} Abbott \& Zardini optimize
\emph{IO transfer costs} across a fixed GPU memory hierarchy---bandwidth
between SRAM and HBM is the scarce resource. Operon optimizes
\emph{agent execution} under metabolic resource constraints (ATP budgets,
latency caps, memory limits). Same categorical insight---diagram structure
reveals optimization opportunities---but different cost models.

\item \textbf{Dynamic state.} Their functional equivalence ($\equiv$)
preserves input-output mappings: two diagrams are equivalent when they
produce the same outputs for the same inputs, differing only in resource
profile. Our coalgebraic observational semantics (Section~\ref{sec:coalgebra}) is aimed at
\emph{stateful} systems and tracks observable behavior across transitions
under a chosen input schedule, not just terminal outputs.

\item \textbf{Scope of composition.} Neural Circuit Diagrams represent
\emph{algorithm pipelines}---fixed dataflow graphs where boxes are
functions composed via typed wire columns. Operon's wiring diagrams
represent \emph{agent networks} where modules may be acquired at runtime
(plasmid registry), conditionally routed (prism optics), and subject to
resource-gated execution. The diagram itself is a dynamic object, not a
static specification.
\end{enumerate}

\noindent In short, Abbott \& Zardini demonstrate that diagrammatic
reasoning scales to deriving hardware-efficient algorithms; Operon extends
the same principle from algorithm-level optimization to system-level
orchestration of stateful, resource-constrained agents.

\newpage

\section{Reference Implementation}\label{sec:implementation}

To instantiate the framework in executable form, we provide a reference implementation in Python.
The implementation, \texttt{operon-ai}\footnote{\url{https://github.com/coredipper/operon}}, is described in this section,
which summarizes key components and demonstrates that the abstractions in the paper translate into practical code.
It should be read as an executable prototype plus synthetic evaluation harness, not as complete empirical
validation of every biological analogy in the manuscript.

\subsection{Architecture Overview}

The implementation follows the biological organization:

\begin{itemize}[leftmargin=*]
\item \textbf{Core Types} (\texttt{core/types.py}): Signal, ActionProtein, FoldedProtein, CellState---the
categorical objects with their biological semantics.
\item \textbf{Core Wiring} (\texttt{core/wagent.py}): WiringDiagram, ModuleSpec, PortType, Wire,
DiagramExecutor\allowbreak---typed wiring with integrity labels and capabilities.
\item \textbf{Core Optics} (\texttt{core/optics.py}): PrismOptic, TraversalOptic,
ComposedOptic\allowbreak---wire-level conditional routing and batch transforms.
\item \textbf{Core Denaturation} (\texttt{core/denature.py}): StripMarkupFilter, NormalizeFilter,
ChainFilter---wire-level anti-injection sanitization.
\item \textbf{Core Coalgebra} (\texttt{core/coalgebra.py}): FunctionalCoalgebra, StateMachine,
ParallelCoalgebra, SequentialCoalgebra, finite-trace observational checking.
\item \textbf{Surveillance} (\texttt{surveillance/}): The immune system implementation with MHCDisplay,
TCell, Thymus, and ImmuneMemory components.
\item \textbf{Adaptive Immune} (\texttt{patterns/verifier.py}): VerifierComponent---rubric-based
quality evaluation that emits signals to WatcherComponent for quality-based model escalation.
\item \textbf{Developmental Gate} (\texttt{patterns/certificate\_gate.py}):
CertificateGateComponent---pre-execution genome integrity check implementing the G1/S checkpoint.
\item \textbf{Healing} (\texttt{healing/}): ChaperoneLoop implementing the structural self-healing pattern.
\item \textbf{Quality} (\texttt{quality/}): Chaperone protein with multi-strategy folding.
\item \textbf{Topology} (\texttt{topology/}): Network motifs including Quorum, Cascade, and Oscillator.
\item \textbf{Organelles} (\texttt{organelles/plasmid.py}): Plasmid,
PlasmidRegistry---capability-gated dynamic tool acquisition.
\item \textbf{Coordination} (\texttt{coordination/diffusion.py}): DiffusionField,
MorphogenSource---graph-based spatially varying morphogen gradients.
\item \textbf{Multi-cellular} (\texttt{multicell/}): CellType, ExpressionProfile, Tissue,
TissueBoundary---hierarchical multi-agent organization.
\end{itemize}

\subsection{Immune System Implementation}

The Adaptive Immunity motif is implemented as an integrated surveillance system:

\begin{verbatim}
@dataclass
class ImmuneSystem:
    thymus: Thymus          # Negative selection
    treg: RegulatoryTCell   # Tolerance/suppression
    memory: ImmuneMemory    # Threat signatures
    displays: dict[str, MHCDisplay]
    tcells: dict[str, TCell]
    profiles: dict[str, BaselineProfile]  # Trained baselines
\end{verbatim}

The \texttt{MHCPeptide} class captures behavioral fingerprints---statistical signatures of agent output
including response time distributions, vocabulary hashes, confidence patterns, and error rates. This
directly implements the MHC presentation concept from Section 4.7.

\paragraph{Two-Signal Activation.}
T-cell activation requires both signals:
\begin{verbatim}
class Signal1(Enum):
    SELF = "self"        # Matches baseline
    NON_SELF = "non_self"  # Anomalous behavior
    UNKNOWN = "unknown"  # Insufficient data (anergic)

class Signal2(Enum):
    NONE = "none"
    CANARY_FAILED = "canary"
    CROSS_VALIDATED = "cross"
    REPEATED_ANOMALY = "repeat"
    MANUAL_FLAG = "manual"  # Operator override
\end{verbatim}

An agent is only flagged when \texttt{Signal1 == NON\_SELF} \textbf{and} \texttt{Signal2 != NONE}.
This prevents false positives from transient anomalies, implementing the immunological requirement
for costimulation.

\subsection{Adaptive Immune Layer: VerifierComponent}

The ImmuneSystem above implements \emph{innate} immunity---generic anomaly detection via baseline
deviations.  The VerifierComponent completes the immune analogy with \emph{adaptive} immunity:
rubric-based quality evaluation tailored to each task type, analogous to B-cells producing
antibodies specific to a particular antigen.

\begin{verbatim}
@dataclass
class VerifierComponent:
    rubric: Rubric | None = None       # (output, stage) -> quality 0.0-1.0
    config: VerifierConfig = ...       # quality_low_threshold = 0.5

    def on_stage_result(self, stage, result, shared_state, ...):
        quality = self.rubric(output, stage_name)
        severity = 1.0 - quality
        signal = WatcherSignal(
            category=SignalCategory.EPISTEMIC,
            source="verifier", value=severity, ...)
        shared_state["_verifier_signals"].append(signal)
\end{verbatim}

The WatcherComponent collects these signals and, when quality falls below threshold on a fast model,
escalates to the deep model.  This provides quality-sensitive escalation that the novelty-based
epiplexity signal cannot: a weaker model producing \emph{varied but mediocre} output will not
trigger epiplexity stagnation but will trigger verifier escalation.

\subsection{G1/S Checkpoint: CertificateGateComponent}

The CertificateGateComponent implements the G1/S DNA damage checkpoint: before each stage
executes its LLM call, the gate scans the genome against a DNARepair checkpoint.  If corruption
is detected, a HALT intervention prevents the corrupted state from reaching the model.

\begin{verbatim}
@dataclass
class CertificateGateComponent:
    genome: Genome
    repair: DNARepair
    checkpoint: StateCheckpoint

    def on_stage_start(self, stage, shared_state, ...):
        damage = self.repair.scan(self.genome, self.checkpoint)
        if damage:
            shared_state[WATCHER_STATE_KEY] = WatcherIntervention(
                kind=InterventionKind.HALT,
                reason=f"genome corruption: {len(damage)} damage(s)")
\end{verbatim}

This moves DNARepair from reactive (detect-after-corrupt) to preventive (block-before-execute),
completing the cell cycle analogy alongside CellCycleController.

\subsection{Chaperone Implementation}

The Chaperone implements multi-strategy folding with provenance tracking:

\begin{verbatim}
class FoldingStrategy(Enum):
    STRICT = "strict"       # Exact JSON match
    EXTRACTION = "extraction"  # Find JSON in text
    LENIENT = "lenient"     # Type coercion
    REPAIR = "repair"       # Fix malformed JSON
\end{verbatim}

The \texttt{ChaperoneLoop} extends this into a healing loop where validation errors are fed back
to the generator:

\begin{verbatim}
class ChaperoneLoop:
    def heal(self, prompt: str) -> HealingResult:
        for attempt in range(self.max_retries + 1):
            raw = self.generator(prompt, error_context)
            folded = self.chaperone.fold_enhanced(raw, schema)
            if folded.valid:
                return HealingResult(HEALED, folded)
            error_context = self._format_error(folded)
        return HealingResult(DEGRADED, ubiquitin_tagged=True)
\end{verbatim}

This operationalizes the GroEL/GroES cage metaphor: the error trace becomes input to the
repair process, enabling context-aware correction.

\subsection{Trust and Provenance}

The implementation uses a simplified 3-level trust hierarchy that collapses the theoretical 4-level
model $\{U, T, S, R\}$ for practical deployment:

\begin{verbatim}
class IntegrityLabel(IntEnum):
    UNTRUSTED = 0   # User input, Retrieved, Self-generated
    VALIDATED = 1   # Schema-checked (Chaperone-folded)
    TRUSTED = 2     # Tool-grounded (deterministic output)
\end{verbatim}

This simplification merges \texttt{User}, \texttt{Retrieved}, and \texttt{Self} into \texttt{UNTRUSTED},
reflecting the common case where all non-tool sources require validation before trust elevation. The
full 4-level model can be recovered by subclassing \texttt{IntegrityLabel} with additional levels when
finer-grained provenance tracking is required.

The \texttt{ApprovalToken} carries explicit authorization metadata for privileged operations:

\begin{verbatim}
@dataclass(frozen=True)
class ApprovalToken:
    request_hash: str
    issuer: str
    reason: str = ""
    confidence: float = 1.0
    integrity: IntegrityLabel = IntegrityLabel.TRUSTED
    timestamp: datetime = field(default_factory=now)
\end{verbatim}

This operationalizes the Trust-Gated Lens: actions requiring high integrity must present an
\texttt{ApprovalToken} with sufficient integrity level.

\subsection{Plasmid Registry Implementation}

The Horizontal Gene Transfer mechanism (\S\ref{sec:hgt}) is implemented as a \texttt{PlasmidRegistry} with capability-gated acquisition:

\begin{verbatim}
@dataclass(frozen=True)
class Plasmid:
    name: str
    func: Callable[..., Any]
    required_capabilities: frozenset[Capability]
    tags: frozenset[str] = frozenset()

    def to_tool(self) -> SimpleTool:
        return SimpleTool(name=self.name, func=self.func,
                          required_capabilities=set(self.required_capabilities))
\end{verbatim}

The \texttt{Mitochondria} class is extended with \texttt{acquire()} and \texttt{release()} methods. When a \texttt{Mitochondria} instance is configured with an \texttt{allowed\_capabilities} envelope, acquisition checks $C_{\text{plasmid}} \subseteq C_{\text{agent}}$ before engulfing the tool; release implements plasmid curing. The registry supports both text search and pure tag filtering.

\subsection{Denaturation Layers Implementation}

The anti-prion defense (\S5.3) is implemented as wire-level filters conforming to a \texttt{DenatureFilter} protocol:

\begin{verbatim}
class DenatureFilter(Protocol):
    @property
    def name(self) -> str: ...
    def denature(self, value: str) -> str: ...
\end{verbatim}

Three concrete filters are provided:
\begin{itemize}[leftmargin=*]
\item \textbf{StripMarkupFilter}: Removes code blocks, ChatML tokens (\texttt{<|...|>}), \texttt{[INST]} tags, XML role tags (\texttt{<system>}, \texttt{<user>}), and role delimiters using compiled regex patterns.
\item \textbf{NormalizeFilter}: Applies Unicode normalization (NFKC), lowercasing, and control character removal to collapse homoglyph-based evasion.
\item \textbf{ChainFilter}: Composes multiple filters left-to-right.
\end{itemize}

Filters attach to wires in the \texttt{WiringDiagram}. The \texttt{DiagramExecutor} applies denaturation before optics: \texttt{denature} $\to$ \texttt{optic} $\to$ destination. This ensures that downstream agents receive sanitized data regardless of what the upstream agent produces.

\subsection{Multi-Cellular Organization Implementation}

The multi-cellular abstractions (\S\ref{sec:multi-cellular}) are implemented in the \texttt{multicell/} package:

\paragraph{Cell Type Specialization.}
The \texttt{CellType} class encapsulates an \texttt{ExpressionProfile}---a mapping from gene names to expression levels (\texttt{OVEREXPRESSED}, \texttt{SILENCED}, etc.). Calling \texttt{differentiate(genome)} applies the profile to a shared \texttt{Genome}, producing a \texttt{DifferentiatedCell} with role-specific configuration. This implements the biological principle that a single genotype produces many
phenotypes.

\paragraph{Tissue Architecture.}
The \texttt{Tissue} class provides:
\begin{itemize}[leftmargin=*]
\item A \texttt{TissueBoundary} with typed input/output ports and a capability ceiling
\item Cell type registration with capability validation ($C_{\text{cell}} \subseteq C_{\text{tissue}}$)
\item Internal wiring via an embedded \texttt{WiringDiagram}
\item Optional \texttt{DiffusionField} for spatially varying morphogen gradients
\item Export as a \texttt{ModuleSpec} for organism-level composition
\end{itemize}

\paragraph{Metabolic-Epigenetic Coupling.}
The \texttt{HistoneStore} accepts an optional \texttt{energy\_gate} parameter pairing an \texttt{ATP\_Store} with a \texttt{MetabolicAccessPolicy}. When set, each \texttt{retrieve\_context()} call costs ATP. Under metabolic stress, only strongly embedded markers remain accessible; in the current implementation, marker strength serves as the cost proxy used to approximate Eq.~\eqref{eq:cost-gated-retrieval}.

\subsection{Bi-Temporal Memory Implementation}
\label{sec:bitemporal-impl}

The bi-temporal coalgebra (\S\ref{sec:coalgebra}) and temporal epistemic framework (\S\ref{sec:temporal-epistemics}) are implemented in \texttt{operon\_ai/memory/bitemporal.py} as a standalone append-only fact store. The design deliberately avoids mutation: facts are frozen dataclasses, corrections close old records and append new ones, and point-in-time queries are pure filters.

\paragraph{Data Model.}
A \texttt{BiTemporalFact} carries 12 fields: subject/predicate/value triple, valid-time interval (\texttt{valid\_from}, \texttt{valid\_to}), record-time interval (\texttt{recorded\_from}, \texttt{recorded\_to}), source provenance, confidence, tags, and an optional \texttt{supersedes} pointer to the corrected fact. All facts are \texttt{@dataclass(frozen=True)}---a deliberate departure from the mutable \texttt{MemoryEntry} used by \texttt{EpisodicMemory}, justified by the append-only semantics.

\paragraph{Write Semantics.}
Three operations modify the store: \texttt{record\_fact()} inserts a new active record; \texttt{correct\_fact()} closes the old record's transaction interval via \texttt{dataclasses.replace()} and appends a new record with \texttt{supersedes} set; \texttt{invalidate\_fact()} marks a record as no longer active. The internal \texttt{\_facts} list is append-only except for the in-place close of \texttt{recorded\_to} on corrected records.

\paragraph{Retrieval.}
Three query methods implement the temporal epistemic operators: \texttt{retrieve\_valid\_at(at)} filters for facts valid at world-time \texttt{at} with active records only; \texttt{retrieve\_known\_at(at)} filters for facts recorded by system-time \texttt{at}; \texttt{retrieve\_belief\_state}\allowbreak\texttt{(at\_valid, at\_record)} intersects both axes, reconstructing the system's belief at any historical coordinate.

\paragraph{History and Audit.}
\texttt{history(subject)} returns all facts (including closed) sorted by record time; \texttt{diff\_between}\allowbreak\texttt{(t1, t2, axis)} computes set differences on either axis; \texttt{timeline\_for(subject)} returns all facts sorted by valid time. Together, these support the audit question: ``what changed between $t_1$ and $t_2$, and on which axis?''

This subsystem is intentionally decoupled from \texttt{HistoneStore} and \texttt{EpisodicMemory}. Integration bridges---converting histone marks to bi-temporal facts, or promoting episodic memories into the bi-temporal substrate---are planned for future work (\S\ref{sec:conclusion}).

\paragraph{SkillOrganism Substrate Integration.}
\label{sec:substrate-integration}
The \texttt{SkillOrganism} runtime (\S\ref{sec:multi-cellular}) composes the three-layer context model described in \S\ref{par:three-layer-context}. An optional \texttt{substrate: BiTemporalMemory} parameter attaches a bi-temporal fact store to the organism. When present, the run loop is extended with a read path and a write path at each stage boundary:

\begin{verbatim}
organism = skill_organism(
    stages=[research, strategist, evaluator, adversary],
    fast_nucleus=fast, deep_nucleus=deep,
    substrate=BiTemporalMemory(),
)
\end{verbatim}

\noindent\textit{Read path.}
Before a stage executes, the runtime evaluates its \texttt{read\_query} field---either a subject string or a callable returning a \texttt{BiTemporalQuery}---and packages the result into a frozen \texttt{SubstrateView(facts, query, record\_time)}. This view is injected into the stage's metadata (for agent stages) or passed as an additional argument (for handler stages, via arity-aware dispatch). The \texttt{record\_time} captures the moment of the read, establishing the record-time horizon for this stage's knowledge.

\noindent\textit{Write path.}
After a stage executes, two mechanisms can emit facts: (1)~the convenience flag \texttt{emit\_output\_fact=True} auto-records the stage output with \texttt{subject=task}, \texttt{predicate=stage.name}; (2)~a \texttt{fact\_extractor} callable converts the stage result into one or more factual events. Each event specifies an operation---assert, correct, or invalidate---and is applied to the substrate via the standard \texttt{BiTemporalMemory} write API. The stage name serves as the default \texttt{source} for provenance.

When \texttt{substrate} is \texttt{None} (the default), the run loop is unchanged: no datetime operations are invoked, no metadata is injected, and all four original lifecycle hooks remain unmodified. Existing tests pass without alteration.

The integration directly enables the audit question from \S\ref{sec:temporal-epistemics}: given a completed run, \texttt{retrieve\_belief\_state(at\_valid=t, at\_record=t\_stage)} reconstructs exactly what the organism believed at the moment stage $X$ executed, even if subsequent stages corrected or invalidated facts. Example~71 demonstrates this with a four-stage enterprise workflow where an adversary stage corrects a research assumption, and the original belief state remains fully reconstructible.

\paragraph{PatternLibrary Implementation.}
\label{sec:pattern-library-impl}
The \texttt{PatternLibrary} provides evolutionary memory for collaboration patterns. Templates are stored in an in-memory dictionary keyed by \texttt{template\_id}; run records accumulate in a list. The \texttt{top\_templates\_for(fingerprint)} method scores each template against a query fingerprint using a weighted combination of task-shape match (0.30), tool-count proximity (0.15), subtask-count proximity (0.15), required-role Jaccard overlap (0.20), tag Jaccard overlap (0.10), and historical success rate (0.10). This simple retrieval mechanism is intentionally stateless and deterministic, with richer adaptation (experience-driven scoring, decay) planned for Phase~4.

\paragraph{WatcherComponent Implementation.}
\label{sec:watcher-impl}
The \texttt{WatcherComponent} is a \texttt{SkillRuntimeComponent} that classifies stage-level signals into three categories following Dupoux et al.~\cite{dupoux2026learning}: \emph{epistemic} (from \texttt{EpiplexityMonitor}), \emph{somatic} (from \texttt{ATP\_Store}), and \emph{species-specific} (from \texttt{ImmuneSystem}). All signal sources are optional; the watcher is a no-op when none are attached.

After each stage, the watcher evaluates collected signals against configured thresholds and may write a \texttt{WatcherIntervention} to \texttt{shared\_state}. The run loop checks for this key after component hooks complete and before the \texttt{halt\_on\_block} guard. Three intervention kinds are supported: \texttt{RETRY} re-executes the current stage; \texttt{ESCALATE} re-executes with the deep nucleus; \texttt{HALT} breaks the stage loop. Component hooks are not re-invoked after retry or escalation, preventing recursive intervention loops.

The intervention-count convergence signal operationalizes the BIGMAS finding~\cite{hao2026bigmas}: when the ratio of interventions to observed stages exceeds \texttt{max\_intervention\_rate} (default 0.5), the watcher emits a non-convergence HALT. Example~73 demonstrates all three intervention paths.

\paragraph{Adaptive Assembly Implementation.}
\label{sec:adaptive-impl}
The \texttt{AdaptiveSkillOrganism} wrapper composes the full adaptive loop. The public factory \texttt{adaptive\_skill\_organism}\allowbreak\texttt{(task, fingerprint, library, ...)} auto-fingerprints the task if no fingerprint is provided, queries \texttt{PatternLibrary.}\allowbreak\texttt{top\_templates\_for()} for the best template, and calls \texttt{assemble\_pattern()} to convert the template's \texttt{stage\_specs} into a runnable topology (dispatching on \texttt{topology}: \texttt{skill\_organism}, \texttt{reviewer\_gate}, \texttt{specialist\_swarm}, or \texttt{single\_worker}). A \texttt{WatcherComponent} and \texttt{TelemetryProbe} are automatically attached.

After execution, the wrapper records a \texttt{PatternRunRecord} in the library (closing the scoring feedback loop) and populates the watcher's experience pool with \texttt{ExperienceRecord} instances for each intervention. The experience pool persists across runs: when rule-based decision logic returns no intervention, the watcher consults past experiences with matching (stage, signal category, fingerprint shape) and recommends the intervention kind that was most often successful. Rule-based decisions always take priority; experience is a fallback. Example~74 demonstrates the full lifecycle; Example~75 demonstrates experience-driven recommendations.

\paragraph{Cognitive Mode Annotations.}
\label{sec:cognitive-modes-impl}
The \texttt{CognitiveMode} enum classifies stages as \texttt{OBSERVATIONAL} (System~A: passive sensing, information gathering) or \texttt{ACTION\_ORIENTED} (System~B: active decision-making, execution). The annotation is an optional field on \texttt{SkillStage}; when absent, it is inferred from the existing \texttt{mode} field (\texttt{fast}/\texttt{fixed} $\to$ \texttt{OBSERVATIONAL}, \texttt{fuzzy}/\texttt{deep} $\to$ \texttt{ACTION\_ORIENTED}). The \texttt{WatcherComponent} collects cognitive-mode signals and reports mode mismatches (e.g., an observational stage routed to the deep nucleus) as informational epistemic signals. The \texttt{mode\_balance()} method summarizes the System~A/B distribution across a run.

\paragraph{Sleep Consolidation Implementation.}
\label{sec:consolidation-impl}
The \texttt{SleepConsolidation} class composes \texttt{AutophagyDaemon}, \allowbreak \texttt{PatternLibrary}, \allowbreak \texttt{EpisodicMemory}, \allowbreak \texttt{HistoneStore}, and optionally \texttt{BiTemporalMemory} into a five-step post-batch consolidation cycle: (1)~prune stale context via autophagy, (2)~replay successful run records and promote them from WORKING to EPISODIC tier in episodic memory, (3)~compress recurring high-success patterns into new consolidated \texttt{PatternTemplate} instances, (4)~run counterfactual replay over bi-temporal corrections to detect cases where updated facts would have changed the outcome, and (5)~promote frequently-accessed ACETYLATION histone marks to permanent METHYLATION.

The \texttt{counterfactual\_replay()} function performs static analysis: it calls \texttt{diff\_between}\allowbreak\texttt{(run\_time, now, axis="record")} to find corrections that occurred after the original run, then matches corrected fact subjects/predicates against stage names in the template. When matches are found, it reports that the outcome may have differed, without re-executing the workflow. Example~77 demonstrates the full consolidation cycle.

\paragraph{Social Learning Implementation.}
\label{sec:social-learning-impl}
The \texttt{SocialLearning} class wraps a \texttt{PatternLibrary} with peer-exchange semantics. \texttt{export\_templates()} filters by success rate and run count; \texttt{import\_from\_peer()} computes an effective score (peer success rate $\times$ trust score) for each template and adopts those exceeding the adoption threshold. The \texttt{TrustRegistry} uses exponential moving average: $s_{t+1} = \alpha \cdot \text{outcome} + (1-\alpha) \cdot s_t$ with $\alpha=0.3$, giving more weight to recent outcomes. Provenance tracking maps each adopted template to its source peer, enabling trust updates when adoption outcomes are recorded. The watcher's curiosity signals extend the epistemic signal category with a \texttt{source="curiosity"} derived from \texttt{EpiplexityMonitor}'s EXPLORING status, triggering ESCALATE when embedding novelty exceeds a configurable threshold on fast models. Example~78 demonstrates template exchange; Example~79 demonstrates curiosity signals.

\paragraph{Developmental Staging Implementation.}
\label{sec:developmental-impl}
The \texttt{DevelopmentController} wraps a \texttt{Telomere} (composition, not inheritance) and maps the fraction of telomere consumed to a \texttt{DevelopmentalStage}: EMBRYONIC ($<10\%$), JUVENILE ($10\text{--}35\%$), ADOLESCENT ($35\text{--}70\%$), MATURE ($>70\%$). Thresholds are configurable via \texttt{DevelopmentConfig}; transitions are one-directional and never regress, even if telomeres are renewed. Learning plasticity decreases monotonically (1.0 at EMBRYONIC, 0.25 at MATURE).

\texttt{CriticalPeriod} frozen dataclasses declare time-limited learning windows by specifying \texttt{opens\_at} and \texttt{closes\_at} stages. The controller evaluates period status on each tick: once the organism passes \texttt{closes\_at}, the window is permanently shut. The \texttt{Plasmid} dataclass gains a \texttt{min\_stage} field; \texttt{Mitochondria.acquire()} checks the organism's current developmental stage before granting tool access. Teacher-learner scaffolding via \texttt{SocialLearning.scaffold\_learner()} filters templates by the learner's stage and applies a plasticity bonus to effective trust, enabling mature organisms to guide younger ones. Example~80 demonstrates the full lifecycle; Example~81 demonstrates scaffolding.

\subsection{Coalgebraic State Machines Implementation}

The coalgebra formalism (\S\ref{sec:coalgebra}) is made explicit and composable:

\begin{verbatim}
class Coalgebra(Protocol[S, I, O]):
    def readout(self, state: S) -> O: ...
    def update(self, state: S, inp: I) -> S: ...

@dataclass
class StateMachine(Generic[S, I, O]):
    state: S
    coalgebra: Coalgebra[S, I, O]
    trace: list[TransitionRecord] = field(default_factory=list)

    def step(self, inp: I) -> O: ...
    def run(self, inputs: list[I]) -> list[O]: ...
\end{verbatim}

\texttt{ParallelCoalgebra} and \texttt{SequentialCoalgebra} implement the composition operations from \S\ref{sec:coalgebra}. The \texttt{check\_bisimulation()} function tests observational equivalence (Eq.~\eqref{eq:bisimulation}) over an input sequence, returning a witness on divergence. Existing organelles (\texttt{HistoneStore}, \texttt{ATP\_Store}, \texttt{CellCycleController}) can be wrapped as coalgebras, enabling formal composition and finite-trace comparison of multi-organelle systems.

\subsection{Morphogen Diffusion Implementation}

The \texttt{DiffusionField} class implements the discrete-time diffusion dynamics (Eq.~\eqref{eq:morphogen-diffusion}):

\begin{verbatim}
class DiffusionField:
    def add_node(self, node_id: str): ...
    def add_edge(self, a: str, b: str, bidirectional=True): ...
    def add_source(self, source: MorphogenSource): ...
    def step(self):  # emit -> diffuse -> decay -> clamp
    def run(self, steps: int): ...
    def get_local_gradient(self, node_id) -> MorphogenGradient: ...
\end{verbatim}

Each \texttt{step()} applies four phases: emission (sources add morphogen), diffusion (concentration flows along edges, split evenly among neighbors), decay (uniform degradation), and clamping (enforce bounds). Nodes with no neighbors skip diffusion outflow, matching Eq.~\eqref{eq:morphogen-diffusion}'s isolated-node case. The \texttt{get\_local\_gradient()} method bridges to the existing \texttt{MorphogenGradient} API, enabling agents to read their local concentrations without awareness of the underlying graph dynamics.

\subsection{Optic-Based Wiring Implementation}

The wire-level optics (\S3.4) are implemented via an \texttt{Optic} protocol:

\begin{verbatim}
class Optic(Protocol):
    def can_transmit(self, data_type: DataType,
                     integrity: IntegrityLabel) -> bool: ...
    def transmit(self, value: Any, data_type: DataType,
                 integrity: IntegrityLabel) -> Any: ...
\end{verbatim}

Four concrete optics are provided:
\begin{itemize}[leftmargin=*]
\item \textbf{LensOptic}: Identity pass-through (equivalent to no optic).
\item \textbf{PrismOptic}: Transmits only if $\tau \in A$ (Eq.~\eqref{eq:prism-optic}). Enables fan-out routing.
\item \textbf{TraversalOptic}: Maps a transform over list elements (Eq.~\eqref{eq:traversal-optic}).
\item \textbf{ComposedOptic}: Chains optics left-to-right; all must accept.
\end{itemize}

Optics coexist with \texttt{DenatureFilter}s on the same wire. The \texttt{DiagramExecutor} processes them in order: denaturation $\to$ optic $\to$ destination. Prism rejection causes the wire to be skipped (not an error), enabling conditional routing patterns where different data types flow to different handlers.

\subsection{Interactive Demonstrations}

The repository also includes interactive Gradio demonstrations for individual organelles and composition
patterns---from single motifs (Membrane, Chaperone, Quorum Sensing) to multi-organelle orchestrations.
These demos are illustrative artifacts rather than controlled benchmarks.

\subsection{Implementation Verification (Synthetic Harness)}\label{sec:empirical}

We define a synthetic evaluation harness to verify that three motifs behave as designed with reproducible procedures:
\begin{enumerate}[leftmargin=*]
\item \textbf{Chaperone Folding.} Generate JSON schemas with 3--8 required fields. Sample valid JSON and
apply 1--3 corruptions (e.g., missing quotes, trailing commas, type swaps, dropped fields). Measure the
fraction of outputs that can be folded into the schema under STRICT versus cascaded strategies
(EXTRACTION $\to$ LENIENT $\to$ REPAIR).
\item \textbf{Immune Detection.} Simulate agents with baseline distributions over response time,
vocabulary hash, structure hash, and confidence. Train on baseline samples, then introduce
``compromised'' agents by shifting distribution parameters and injecting anomalous hashes.
Measure sensitivity and false positive rate under the two-signal activation rule.
\item \textbf{Healing Loop.} Generate malformed outputs and run the ChaperoneLoop with and without
error-context feedback. Measure recovery within $k$ attempts (default $k=3$).
\end{enumerate}
These suites test whether the implemented motifs behave as intended under controlled corruption and anomaly
models; they do not estimate in-the-wild failure rates of production LLM systems.

\paragraph{Implementation.}
The harness is implemented in \texttt{eval/} with a JSON-configured CLI:
\begin{verbatim}
python -m eval.run --suite all --config eval/configs/default.json \
  --out eval/results/latest.json
\end{verbatim}
The output is a machine-readable JSON report (per-suite config + metrics) suitable for direct inclusion
in tables or plots.

\paragraph{Status.}
Synthetic runs verify that each motif functions as designed within this harness (cascade $>$ strict, error-context $>$ blind retry,
immune detection $>$ chance). Table~\ref{tab:eval-summary} reports aggregated numeric results across
multiple seeds; real-world validation with LLM outputs remains ongoing.

\paragraph{Aggregated Results.}
We ran the harness across 100 deterministic seeds (1--100) using the default harness config
(\path{eval/configs/default.json}). The
aggregate results (pooled across seeds with Wilson 95\% intervals; $N$ is total pooled trials) are reported
in Table~\ref{tab:eval-summary}.

\paragraph{External Benchmark Suites.}
In addition to the synthetic suites above, the harness includes
suites derived from external benchmarks: (1)~function-call schemas from
the Berkeley Function Calling Leaderboard (BFCL)~\cite{patil2023gorilla}
test the Chaperone's folding pipeline against realistic tool-use schemas,
and (2)~prompt injection attack templates from
AgentDojo~\cite{debenedetti2024agentdojo} generate adversarial behavioral
shifts to test Immune System detection. These suites use the same
deterministic corruption and simulation methodology; results appear in the
lower section of Table~\ref{tab:eval-summary}.

\begin{table}[h]
\centering
\small
\setlength{\tabcolsep}{4pt}
\begin{tabular}{p{0.46\linewidth}ccc}
\toprule
Metric & Rate & 95\% CI & N \\
\midrule
Chaperone Folding (Strict) & 5.5\\
Chaperone Folding (Cascade) & 56.2\\
Healing Loop (Error Context) & 99.6\\
Healing Loop (Blind Retry) & 68.0\\
Immune Detection (Sensitivity) & 100.0\\
Immune Detection (False Positive) & 0.4\\
\midrule
BFCL Folding (Strict) & 4.9\\
BFCL Folding (Cascade) & 56.9\\
AgentDojo Immune (Sensitivity) & 100.0\\
AgentDojo Immune (False Positive) & 0.0\\
\bottomrule
\end{tabular}
\caption{Evaluation results aggregated across 100 deterministic seeds (Wilson 95\% CI). Top: synthetic motif tests. Bottom: external benchmark--derived tests (BFCL, AgentDojo).}
\label{tab:eval-summary}
\end{table}

\subsection{Limitations}

The current implementation has several limitations:

\begin{itemize}[leftmargin=*]
\item \textbf{Epiplexity:} Implemented with a mock embedding provider (\texttt{MockEmbeddingProvider}).
Integration with production embedding APIs and empirical calibration of the $\alpha$ mixing parameter
and threshold $\delta$ across diverse task types remain future work.
\item \textbf{Morphogen Diffusion:} The \texttt{DiffusionField} operates in a single process. Cross-agent
gradient propagation in distributed multi-process deployments with eventual consistency remains future work.
\item \textbf{Denaturation:} Filters target known syntactic patterns (ChatML, XML role tags, markdown
code blocks). Novel injection techniques using previously unseen syntax may bypass denaturation; custom
\texttt{DenatureFilter} implementations should be added as new attack patterns emerge.
\item \textbf{Finite-Trace Equivalence:} The \texttt{check\_bisimulation()} function tests over finite input sequences.
Coinductive bisimulation for infinite-trace systems is not yet supported.
\item \textbf{Benchmarking:} The evaluation harness covers synthetic suites and external benchmarks (BFCL,
AgentDojo). Real-world validation of the multi-cellular and diffusion components at production scale is
still in progress.
\end{itemize}

We release the implementation to enable further scrutiny and extension of the framework.

\newpage

\section{Convergence: Integrating External Agent Frameworks}\label{sec:convergence}

The preceding sections develop Operon's structural analysis, epistemic topology,
and formal verification foundations. This section summarizes how these tools
extend to external agent orchestration systems through a typed adapter
architecture. A standalone companion paper provides the full treatment, including
implementation details, all 107 examples, and complete TLA+ specification
listings; this section is self-contained but deliberately concise.

\begin{center}
\begin{tikzpicture}[
  layer/.style={draw, minimum width=10cm, minimum height=0.8cm, font=\small},
  >=latex
]
\node[layer, fill=blue!8]  (L5) at (0,3.2) {A-Evolve --- evolution layer};
\node[layer, fill=green!8] (L4) at (0,2.4) {AnimaWorks --- cognitive layer};
\node[layer, fill=yellow!8](L3) at (0,1.6) {AsyncThink --- thinking layer};
\node[layer, fill=orange!8](L2) at (0,0.8) {Ralph\,/\,DeerFlow\,/\,Swarms --- orchestration layer};
\node[layer, fill=red!8]   (L1) at (0,0.0) {Operon --- structural layer};
\end{tikzpicture}
\end{center}

Every adapter produces an \texttt{ExternalTopology}---a framework-agnostic
intermediate representation carrying a source tag, a pattern name, agent
specifications, directed communication edges, and arbitrary metadata. The
analysis pipeline is source-agnostic: \texttt{analyze\_external\_topology()}
consumes any \texttt{ExternalTopology} and applies three of the four epistemic
bounds---error amplification, sequential penalty, and tool density---plus a
topology-mismatch flag to produce topology advice, structural warnings, and
a composite risk score. (Parallel acceleration is computed by the epistemic
layer but not included in the adapter risk score.) Adding a new orchestration target therefore requires
writing a single parse function---no changes to the analysis code.

Four TLA+ specifications verify safety invariants that are difficult to test
exhaustively at the unit level: \emph{TemplateExchangeProtocol} (provenance and
trust monotonicity), \emph{DevelopmentalGating} (irreversible stage transitions
and capability constraints), \emph{ConvergenceDetection} (intervention-count
convergence or HALT), and \emph{EvolutionGating} (monotonic score safety for
A-Evolve's evolutionary loop). Together, these cover the critical safety
boundaries of the convergence stack.

For the complete treatment---including adapter implementation details, template
exchange protocols, memory bridge semantics, TLA+ specification listings, and all
107 worked examples---see the standalone convergence companion
paper.\footnote{\url{https://coredipper.github.io/operon/convergence/}}

A live evaluation harness (Example~107) measures quality, latency, and token
cost across Gemini~API, Claude~CLI, and Codex~CLI providers. Guided
multi-stage pipelines show a consistent $+6.2\%$ quality improvement over
unguided configurations; single-agent CLI execution shows no effect,
confirming that structural guidance helps when there is topology to guide.

\paragraph{Meta-evolution (Phase~C8).}
The meta-harness extends the convergence stack to \emph{evolving} organism
configurations---modes, models, and intervention thresholds---rather than
just running them. A \texttt{FilesystemOptimizer} protocol (distinct from
C7's prompt-level \texttt{EvolutionaryOptimizer}) drives an
\texttt{EvolutionLoop} that maps candidate configs to \texttt{Genome}
objects, evaluates via \texttt{LiveEvaluator}, and persists full execution
traces to a candidate-first filesystem store.

The key scientific finding: Operon's biological abstractions generalize to
the meta-level. The \texttt{Genome} mapping is lossless (${\sim}5$ lines of
flattening logic). The \texttt{EpiplexityMonitor} generalizes across scales
via a pluggable \texttt{DistanceProvider}---\texttt{ConfigHammingDistance}
triggers STAGNANT/EXPLORING transitions identically to embedding cosine
distance. \texttt{DesignProblem} wrapping of evolution steps is natural.
Boundaries: \texttt{feedback\_fixed\_point} does not fit (evolution is not
convergent iteration), and \texttt{TrustRegistry} is overkill for two
proposer strategies.

An Ao~et~al.~\cite{ao2026delegation} test of exogenous signals shows that
rich filesystem context (configs $+$ trace metadata) yields a $3\times$
improvement over compressed history for the LLM proposer ($0.49$ vs
$0.15$), but config-space evolution does not strongly benefit from LLM
reasoning over blind tournament mutation ($0.49$ vs $0.44$). Phase~B
(topology mutations with DAG execution) improved tournament ($0.60$) but
degraded the LLM proposer ($0.36$), confirming that biological abstractions
generalize as \emph{code structure} but not as \emph{optimization
algorithms}. The structural guarantee features---immune systems, epiplexity,
developmental gating---remain the core value proposition.
de~los~Riscos~et~al.~\cite{delosriscos2026categorical} provide a
category-theoretic framework (\textsc{ArchAgents}) that formalizes
Operon's architecture: objects are organisms, morphisms are compilers,
agents are configured instances. A detailed treatment of the C8 findings
appears in the companion meta-evolution paper.

\newpage

\section{Conclusion}\label{sec:conclusion}

The transition from ``Prompt Engineering'' to ``Agentic Engineering'' requires moving beyond component-level
optimization toward principled architectural design. Current methodologies often lack the formal foundations
needed to reason about system-level properties like termination, error suppression, and graceful degradation.

In this paper, we have argued that Gene Regulatory Networks (GRNs) provide a useful source of control motifs
for distributed, stochastic information processing. By expressing selected biological and software components
within a shared interface language from Applied Category Theory, we derived a corresponding suite of design
patterns. The result is a modeling and design framework, not a claim that biological and software systems are
identical in mechanism:

\subsection*{Core Contributions}

\begin{enumerate}[leftmargin=*]
\item \textbf{Robustness via Topology:} The Coherent Feed-Forward Loop provides error suppression proportional
to $(1-\rho)$, where $\rho$ is the correlation between component error modes. We make precise the conditions
under which topological redundancy provides genuine safety benefits: highly correlated generator/verifier
pairs yield limited improvement, while more heterogeneous pairs can suppress errors more effectively. The
topology is necessary but not sufficient; component diversity determines actual error suppression.

\item \textbf{Adaptive Immunity:} We formalize the Self/Non-Self distinction as a \textbf{Provenance Functor}
$\mathcal{P}: \mathbf{Msg} \to \mathbf{Trust}$ with structurally-enforced labels. The Trust-Gated Lens
provides resistance to content-level trust forgery by ensuring that content-based attacks cannot elevate provenance. This extends
the Prion metaphor into a full immunological framework with MHC-like tagging, negative selection during
training, and regulatory suppression of conflicting sources.

\item \textbf{Epistemic Health:} The formalization of \textbf{Epiplexity} (Bayesian Surprise) as a metric for
detecting ``epistemic starvation.'' We provide an operational approximation using embedding similarity and
conditional perplexity, with windowed detection to distinguish task completion from pathological loops.
This connects agent dynamics to the Free Energy Principle: healthy agents minimize surprise through learning
or effective action; stagnant agents do neither.

\item \textbf{Metabolic Intelligence:} The reframing of the Runtime from passive budget to active
\textbf{Cognitive Control Policy}. Drawing on MIPS (Mitochondrial Information Processing System), we
distinguish fast interventions (Apoptosis via mPTP-like triggers) from slow interventions (Retrograde
Responses that reshape agent phenotype across sessions). The Runtime governs reasoning \textit{quality}
through the Metabolic-Epigenetic Coupling: low-budget states ``methylate'' expensive context, forcing
efficient phenotypes.

\item \textbf{Multi-Cellular Organization:} The extension from single-agent to multi-agent systems using
developmental biology. Agent phenotypes arise from differential context (epigenome) on shared weights (genome).
Morphogen gradients (shared context variables) enable coordination without central control. Tissue boundaries
enforce security isolation. This reframes ``how many agents?'' as ``what is the developmental program?''

\item \textbf{Homeostasis:} Continuous self-repair through three modalities: Structural (Chaperone Loop with
error-context feedback), Metabolic (Apoptosis + Regeneration with state summarization), and Cognitive
(Autophagy via sleep/wake cycles). Most frameworks focus on Action; biology equally prioritizes Maintenance.

\item \textbf{Evolutionary Dynamics:} The Vermeij Trend predicts that agentic architectures face three
selective pressures: adversarial robustness (Red Queen dynamics with attackers), task complexity
(environmental pressure toward ``aerobic'' multi-step reasoning), and resource efficiency (metabolic
selection toward the Pareto frontier). In our framing, these pressures motivate architectures that balance
immune defense, task complexity, and metabolic regulation rather than treating those concerns as separable add-ons.

\item \textbf{Wire-Level Optics:} Prism optics enable conditional type-based routing (receptor specificity),
Traversal optics enable batch processing (polymerase processivity), and DenatureFilters provide
anti-injection sanitization---all composable on the same wire.

\item \textbf{Morphogen Diffusion:} A discrete-time dynamical system on the agent graph produces spatially
varying concentration profiles, enabling position-dependent coordination without central control or global
state sharing.

\item \textbf{Composable Coalgebras:} The coalgebraic formalism is made explicit with parallel and sequential
composition, full transition traces, and observational equivalence criteria, enabling structured comparison
of agent implementations.

\item \textbf{Epistemic Topology:} Kripke-style knowledge operators ($K_i$, $E_G$, $C_G$, $D_G$) derived
from wiring diagram structure yield four predictive theorems for multi-agent scaling---error amplification,
sequential penalty, parallel acceleration, tool density---and these theorems are qualitatively consistent
with empirical results from large-scale architecture evaluations~\cite{kim2025scaling}.

\item \textbf{Capability-Gated Tool Acquisition:} The Plasmid Registry implements Horizontal Gene Transfer
with capability gating, preventing privilege escalation when agents dynamically acquire tools.

\item \textbf{Diagram Optimization via Categorical Rewriting:} Cost annotations enrich the wiring category
over the resource monoid $(\mathcal{R}, +, 0)$, making resource consumption a first-class compositional concept.
Rewriting passes---dead wire elimination, parallel grouping, cost-order scheduling---are endofunctors
that preserve input-output observational equivalence while improving resource utilization. The ResourceAwareExecutor
gates module execution on MetabolicState, connecting diagram optimization to the metabolic intelligence
framework: under ATP depletion, non-essential pathway branches are downregulated, mirroring cellular
starvation responses.
\end{enumerate}

\subsection*{The Structural Foundation}

The correspondence rests on the category $\mathbf{Poly}$ of polynomial functors as a language for typed
interfaces. Within that language, biological and software components at several scales---genes and agent
capabilities, cells and agent runtimes, tissues and multi-agent subsystems---can be modeled as interfaces
$(O, I)$ consuming observations and producing actions. The Operad of Wiring Diagrams
provides the grammar for composition, with type-checking at the topological level preventing classes of
runtime errors. The Metabolic Coalgebra enriches this with resource constraints, providing a decidable
termination criterion when costs are strictly decreasing and regeneration is excluded (or separately bounded).
The Provenance Functor layers trust semantics onto message flow, providing a structural account of
trust-gated resistance to content-level trust forgery.

\subsection*{Implications}

The Operon framework should be read as a structured design language for robust agent architectures. Its
value is that some biological analogies can be made precise enough to guide implementation: topology
constrains coordination, resource models constrain execution, and provenance labels constrain trust
elevation.

The correspondence is therefore neither a loose metaphor nor a full biological equivalence. It is a
disciplined abstraction that makes a subset of biologically inspired design patterns precise enough to
analyze and implement, as demonstrated by the reference implementation.

The epistemic-topology formalization is qualitatively consistent with external scaling evidence such as
Kim et al.~\cite{kim2025scaling}, but it does not by itself constitute complete empirical validation of the
framework. The connection between morphogen-driven topology switching and epistemic optimization instead
suggests a path toward more formally analyzed adaptive multi-agent systems, where topological transitions
can be justified under explicit epistemic and resource assumptions.

\subsection*{The Six-Layer Arc}

The six-layer progression demonstrates that the biological analogy extends beyond structural safety into temporal reasoning, adaptive behavior, and cognitive development:

\begin{enumerate}[leftmargin=*]
\item \textbf{Structure:} Typed wiring diagrams, topology advice, pattern-first API.
\item \textbf{Memory:} Bi-temporal facts with dual time axes; auditable substrate integration with three-layer context model.
\item \textbf{Adaptation:} Pattern libraries for evolutionary template memory; watcher with three-category signal taxonomy (epistemic/somatic/species); intervention-count convergence signal grounded in Hao et al.~\cite{hao2026bigmas}; experience-driven adaptive assembly.
\item \textbf{Cognition:} System A/B cognitive mode annotations per Dupoux et al.~\cite{dupoux2026learning}; sleep consolidation with counterfactual replay; social learning with epistemic vigilance (trust-weighted template exchange); curiosity signals for intrinsic motivation.
\item \textbf{Development:} Critical periods that close as organisms mature; capability gating via developmental stage on tool acquisition; teacher-learner scaffolding.
\item \textbf{Integration:} Cross-subsystem integration tests; memory adapters bridging histone and episodic memory into the bi-temporal store; paper and documentation finalization.
\end{enumerate}

Each layer assumes the previous one is stable. The progression from structural safety to temporal epistemics to adaptive cognition to developmental staging is not arbitrary---it mirrors the biological sequence from genome (fixed structure) through epigenetics (learned bias) to neural development (plastic then crystallizing).

\subsection*{Future Work}

Several directions remain open:
\begin{itemize}[leftmargin=*]
\item \textbf{Broader Convergence Targets:} The adapter architecture in \S\ref{sec:convergence} integrates five external systems (A-Evolve, AnimaWorks~\cite{animaworks}, AsyncThink, Ralph/DeerFlow/Swarms~\cite{swarms}, and Operon-native topologies) under four TLA+-verified safety invariants. Extending coverage to additional orchestration targets and measuring cross-framework trust monotonicity at production scale remain open.
\item \textbf{Production Benchmarks:} Validation on real LLM outputs at scale via BFCL and AgentDojo evaluation harnesses, measuring both reliability and the impact of adaptive assembly on task performance.
\item \textbf{Adversarial Robustness:} Red-team evaluation of immune evasion vectors and development of continuous adaptation mechanisms.
\item \textbf{Distributed Diffusion:} Extending morphogen fields to distributed multi-process deployments with eventual consistency.
\item \textbf{Learned Rewriting Rules:} Learning diagram optimization rules from execution traces, connecting to program synthesis and equality saturation.
\end{itemize}

\newpage

\bibliographystyle{plain}
\bibliography{references}

\end{document}